\documentclass{article}

\usepackage[preprint,nonatbib]{neurips_2024}
\usepackage{hyperref}
\usepackage{tikz}
\usepackage{chicago}
\usepackage{color}
\usepackage{centernot}
\usepackage{amsmath}
\usepackage{amsthm}
\newtheorem{theorem}{Theorem}[section]
\newtheorem{lemma}[theorem]{Lemma}

\newtheorem{example}[theorem]{Example}

\newtheorem{definition}[theorem]{Definition}
\newcommand{\bbox}{\vrule height7pt width4pt depth1pt}
\newcommand{\prf}{\noindent{\bf Proof:} }
\newcommand{\eprf}{\bbox\vspace{0.1in}}
\newcommand{\PN}{\textsc {PN}}
\newcommand{\PS}{\textsc{PS}}
\newcommand{\PNS}{\textsc {PNS}}
\newcommand{\U}{{\cal U}}
\newcommand{\V}{{\cal V}}
\newcommand{\R}{{\cal R}}
\newcommand{\F}{{\cal F}}
\renewcommand{\S}{{\cal S}}
\renewcommand{\P}{{\cal P}}
\newcommand{\Lcaus}{{\cal L}}
\newcommand{\red}{\color{black}}
\newcommand{\blue}{\color{black}}
\newcommand{\green}{}
\newcommand{\purple}{\color{black}}
\newcommand{\pink}{\color{black}}
\newcommand{\union}{\cup}
\newcommand{\sat}{\models}
\newcommand{\dimp}{\Leftrightarrow}

\newcommand{\new}[1]{{\color{red}{#1}}}

\newcommand{\commentout}[1]{}
\newcommand{\shortv}[1]{#1}
\newcommand{\fullv}{\commentout}

\renewcommand{\phi}{\varphi}

\usepackage[american]{babel}
\usepackage{mathtools} % amsmath with fixes and additions
\usepackage{booktabs} % commands to create good-looking tables
\usepackage{tikz} % nice language for creating drawings and diagrams

\title{Intervention and Conditioning in Causal Bayesian Networks}

\author{Sainyam Galhotra\\
    Computer Science Dept. \\
    Cornell University\\
   \texttt{sg@cs.cornell.edu}
    \And
    Joseph Y. Halpern\\
    Computer Science Dept. \\
    Cornell University\\
     \texttt{halpern@cs.cornell.edu}
    }

\begin{document}
\maketitle

\begin{abstract}
%sg1: added abstract
Causal models are crucial for understanding complex systems and
identifying causal relationships among variables. Even though causal
models are extremely popular, conditional probability calculation of
formulas involving interventions pose significant challenges.
In case of Causal Bayesian Networks (CBNs), Pearl assumes autonomy 
of mechanisms that determine interventions to calculate a range of
probabilities.
%joe5
%
%In this paper, we show that by making simple yet realistic independence
We show that by making simple yet
%joe13
often
realistic independence
assumptions, it  is possible 
to uniquely estimate the probability of an interventional formula (including
%joe5
%the popular notions of probability of sufficiency and necessity).
the well-studied  notions of probability of sufficiency and necessity). 
%joe13
We discuss when these assumptions are appropriate.
Importantly,
%joe13
in many cases of interest, when the assumptions are appropriate,
these probability estimates can be evaluated using
observational data, which carries immense significance in scenarios
where conducting experiments is impractical or unfeasible.
\end{abstract}

{\blue
  \section{Introduction}
  
Causal models play a pivotal role in elucidating the causal relationships
among variables. These models facilitate a principled approach to
understanding how various factors interact and influence each other in
complex systems. For instance, in epidemiology, causal models are
instrumental in deciphering the relationships between lifestyle choices
%joe5
%and health outcomes~\cite{greenland1999causal}. In economics,
and health outcomes~\cite{greenland1999causal}; and in economics,
they help in analyzing the impact
of policy changes on market dynamics~\cite{hicks1980causality}.
%sg1:added citations
These examples underscore the
%joe3
%versatility and utility of causal models to provide a formal representation
versatility and utility of causal models for providing a formal representation
of  system variables.
%of underlying system variables.
%sg1:shaved a line

Interventions and conditioning are the most fundamental procedures in the
application of causal models, useful to examine and analyze causal mechanisms.
One of the most recent applications of an intervention is to explain the outcome of a
%joe3: I think we need referenceas for all the claims in the rest of
%this paragraph and the next section
complex ML system~\cite{galhotra2021explaining},
For example, in AI-driven healthcare diagnostics,
it's crucial to discern whether a particular intervention (like a change in treatment protocol) 
will sufficiently alter patient outcomes~\cite{greenland1999relation}. 
%sg1:shortened and added citation
%Similarly, in social science research,
%conditioning and intervening on specific variables allows for the assessment of direct and
%indirect effects of social policies.

Despite their utility, calculating the probabilities related to interventions and
conditioning in tandem presents significant challenges.
%joe3*: slowing down here.  New paragraph
Indeed, it is not even clear what the semantics of queries involving
%joe4
%counterfactuals is.  More precisely, work in the AI literature has
%focuse on two types of models: \emph{functional} causal models} and
counterfactuals is.  Work in the AI literature has
focused on two types of models: \emph{functional} causal models} and
%joe4: line shaving 
%\emph{causal Bayesian networks} \cite{pearl:2k}.  Both types of models
\emph{causal Bayesian networks} \cite{pearl:2k}.  \shortv{Both}
\fullv{Models of both types}
are typically described using directed acyclic graphs, where each node is
associated with a 
variable.  In a causal model, with each variable $Y$ associated with a
non-root node, there is a deterministic (structural) equation, that
gives the value 
%joe4
%of $Y$ as a fuenction of the value of its parents; there is also a
%probability on the root nodes.  In a CBN, like in a
of $Y$ as a function of the values of its parents; there is also a
probability on the values of root nodes.  In a CBN, like in a
Bayesian network, each variable $Y$ is associated with a \emph{conditional
probability table (cpt)}, that for each setting of the parents of $Y$, gives
the probability of $Y$ conditional on that setting.  In a functional
causal model, it is actually straightforward to determine the
conditional probability of formulas involving interventions.  In a
CBN, this is far from true.  Indeed, recent work of Beckers
\citeyear{Beckers23} has shown that an approach given by Pearl
\citeyear{pearl:2k} to calculate these probabilities in a CBN is incorrect.  
%13*
\footnote{Pearl \citeyear{pearl:2k}[Theorem 7.1.7] provides a  (correct)
  three-step procedure 
for calculating counterfactual probabilities in a causal model.
But then on p. 220, Pearl says that the same procedure works for
CBNs.  Specifically, he says ``counterfactual
probabilities $p(Y_x = y \mid e)$ can still be evaluated using the three
steps (abduction, action, and prediction) of Theorem 7.1.7. In the
abduction phase, we condition the prior probability $p(u)$ of the root
nodes on the evidence available, $e$, and so obtain $p(umid e)$. In the
action phase, we delete the arrows entering variables in set $X$ and
instantiate their values to $X = x$. Finally, in the prediction phase,
we compute the probability of $Y = y$ resulting from the updated
manipulated network.''   As Beckers shows, this is incorrect.
Here's a trivial counterexample  Suppose that we have a simple causal
model with one exogenous variable $U$, which is the parent of an
endogenous variable $Y$, which in turn in is the parent of an
endogenos variable $X$.  All variables are binary.  $U=1$ with probability 1.
$Y=U$, and if $Y=1$, then $X=0$ with probability $1/2$ and $X=1$ with
probability $1/2$.  Now consider $p(X=1\mid X=1)$.  Applying Pearl's
procedure, the probability of $U=1$ continues to be $1$ (no amount of
conditioning will change that).  Since there are no interventions,
$Y=1$ with probability 1, and $X=1$ with probability $1/2$.  That is,
$p(X=1 \mid X=1) = 1/2$ according to Pearl's procedure.  But this is
clearly incorrect.}
%This complexity often
%arises from the intricate interplay between various variables in a
%causal model.
%joe4
%Pearl does calculates probabilities in a CBN by implicitly reducing
Pearl also calculates probabilities in a CBN by implicitly reducing
the CBN to a family of functional causal models (see, e.g.,
\cite[Theorem 9.2.10]{pearl:2k}), but he does not give an explicit
reduction, nor does he give a formal definition of the probability of a
formula in  a CBN.  Here, we do both.  Using this approach leads to
formulas having a range of probabilities in a CBN, whereas in a
functional causal model, their probability is unique.  

%joe3: added next sentence
But we take an additional significant step.
Pearl assumes that mechanisms  
that determine how interventions work (which are
%joe3
given by
the cpts in the case of
CBNs and the structural equations in the case of causal models)
%joe3: I prefer the new quote
%are \emph{autonomous}, i.e., ``it is conceivable to change one such
%without changing the others'' \cite[p. 22]{pearl:2k}.  These
are \emph{autonomous}: as Pearl puts it, ``external changes affecting
one equation do not imply changes to the others''
\cite[p. 28]{pearl:2k}.
%joe3
%These assumptions are made 
%to analyze the effects of an intervention and estimate the probability
%of counterfactual questions.
We model this autonomy formally by taking the equations to be
%independent of each other, in an approprate space.  As showm recently
%sg1:typo
independent of each other, in an appropriate space.  As shown recently
by Richardson and Halpern \citeyear{HR23}, taking the equations that
characterize different variables to be independent is a necessary and
sufficient condition for reproducing all the (conditional) independencies
in the underlying Bayesian network, as determined by
\emph{d-separation} \cite{Pearl}.  Thus, this independence seems like
a natural and critical assumption to get CBNs and causal models to
work as we would expect.

%sg4
%Here we assume that, not only are the equation that define different
Here we assume that, not only are the equations that define different
variables independent, but also the 
equations that give the values of a variable for different settings of
its parents.  We never need to consider the values of a variable for
different settings of its parents in a standard Bayesian network, but
this is necessary to determine the probability of a formula involving
interventions, such as $X=0 \land Y=0 \land [X \gets 1](Y=1)$ ($X$ and
$Y$ have value 0, but if $X$ is set to 1, $Y$ gets value 1).
%joe13*
%Taking these latter equations to be independent seems to us completely
%in the spirit of taking the equations for different variables to be
%independent.  Moreover, it is easily motivated: if the values of a
%sgn2:typo
%Taking these latter equations to be indepedent is not always
Taking these latter equations to be independent is not always
appropriate;\footnote{We thank Elias Bareinboim and Scott Muller for
stressing this point.}  For example, there may be a latent exogenous
variable that affects the value of $Y$ for different settings of $Y$'s
parents. But if the parents of $Y$ (including exogenous variables) are
all observable, and screen $Y$ off from the 
effects of all other variables, then the independence assumption seems
%sgn2:typo
%approprirate.  
appropriate.
%probabilities of formulas.

%joe3*: we need to stress the indpendence assumptions, which I've just
%done.  Without that, the claim below is false.
%In this work, we show that the probability of any combination of
%interventions along with 
%conditional events can be estimated from observational
%data. Specifically, we assume  
%(probabilistic) independence between different interventions. 

%joe4
%Making this assumption has significant benefits.  For one thing, it
Making these independence assumptions has significant benefits.  For
one thing, it 
allows us to uniquely identify the probability of queries in a CBN;
rather than getting a range of values, we get a unique value.
%joe4
%Moreover, as we show, for many formulas of interest (include the
Moreover, for many formulas of interest (including the
\emph{probability of necessity} and \emph{probability of sufficiency}
\cite{pearl:2k}, we can compute the probability by considering only
%joe13: acceptd
%conditional probabilities \new{involving a subset of endogenous and
%exogenous variables,} that do not involve interventions.  This
conditional probabilities involving only a subset of endogenous and 
exogenous variables, which do not involve interventions.  This
%joe5
%means estimated from observational data,
means that these probabilities can be estimated from observational data,
without requiring involving controlled experiments.  This can
have huge implications in settings where such experimental data is not
%joe13: perhaps "observed" is better?
%available \new{but the exogenous variables can be measured}.
available but the exogenous variables can be observed.

%joe3: If there's room
%sg1:copied from the other version
The rest of this paper is organized as follows. Section~\ref{sec:causalmodel}
reviews the formalism of causal
models. Section~\ref{sec:semantics} gives semantics to formulas in Causal Bayesian
Networks (CBNs)  and Section~\ref{sec:convert}  shows that any CBN can be
converted to a compatible causal model that satisfies
%joe5
%independence assumptions.
the independence assumptions that we are interested in.
%sgn2: moving sec 5 to appendix
%joe14
%We show the simplification of counterfactual probabilities of
%necessity and sufficiency in the appendix.
We show how counterfactual probabilities of 
necessity and sufficiency can be simplified and calculated in 
Section~\ref{sec:computing}.
%the appendix.
%Section~\ref{sec:computing} analyzes the counterfactual probabilities of
%necessity and sufficiency.

 \section{Causal Models and CBNs}\label{sec:causalmodel}
 %sg1:added label

  In a (\emph{functional) causal model} (also called a \emph{structural
  equations model}), the world is assumed  to be described in terms of 
variables and their values.  
Some variables may have a causal influence on others. This
influence is modeled by a set of {\em structural equations}.
It is conceptually useful to split the variables into two
sets: the {\em exogenous\/} variables, whose values are
determined by 
factors outside the model, and the
{\em endogenous\/} variables, whose values are ultimately determined by
%joe13: again, "observed"?
%the exogenous variables.  \new{In some settings, exogenous variables
%can be measured but never be intervened, as these  
the exogenous variables.  In some settings, exogenous variables
can be observed; but they can never be intervened upon, as (by
assumption) their values
are determined by factors outside the model.
Note that
%joe13
%exogenous variables may contain certain latent factors which are
%can not be measured and could be unknown in some cases too.
%sgn2:typo
%xogenous variables may involve  latent factors that are
exogenous variables may involve  latent factors that are
not observable, and may even be unknown.
For example, in an agricultural setting, we could have
%joe13
%endogenous variables that describe crop-produce, amount of fertilizers 
%used, water consumption, etc and exogenous variables
%that describe weather conditions (cannot be modified but can be measured)
endogenous variables that describe crop produce, amount of fertilizers 
used, water consumption, and so on, and exogenous variables
that describe weather conditions (which cannot be modified, but can be
observed) 
%joe1
%and some latent factors like activity level of pollinators (cannot be
%measured).}
and some latent factors, like the activity level of pollinators (which
cannot be observed or measured).
%sgn1: Commented as I changed the example
%For example, in a voting scenario, we could have endogenous variables
%that describe what the voters actually do (i.e., which candidate they
%vote for), exogenous variables 
%that describe the factors
%that determine how the voters vote, and a
%variable describing the outcome (who wins).  
The structural equations
%joe13
%describe how the outcome is determined
describe how the values of endogenous variables are determined
%joe5
%joe13
%\fullv{
%(majority rules; a candidate
%wins if $A$ and at least two of $B$, $C$, $D$, and $E$ vote for him;
%etc.).}
%\shortv{(e.g., majority rules).}
(e.g., how the water consumption depends on the weather conditions and
the amount of fertilizer used).

Formally, a \emph{causal model} $M$
is a pair $(\S,\F)$, where $\S$ is a \emph{signature}, which explicitly
lists the endogenous and exogenous variables  and characterizes
their possible values, and $\F$ defines a set of \emph{modifiable
structural equations}, relating the values of the variables.  
A signature $\S$ is a tuple $(\U,\V,\R)$, where $\U$ is a set of
exogenous variables, $\V$ is a set 
of endogenous variables, and $\R$ associates with every variable $Y \in 
\U \cup \V$ a nonempty set $\R(Y)$ of possible values for 
$Y$ (that is, the set of values over which $Y$ {\em ranges}).  
%joe4
%For simplicity, I assume here that $\V$ is finite, as is $\R(Y)$ for
For simplicity, we assume that $\V$ is finite, as is $\R(Y)$ for
every endogenous variable $Y \in \V$.
$\F$ associates with each endogenous variable $X \in \V$ a
function denoted $F_X$ such that $F_X: (\times_{U \in \U} \R(U))
\times (\times_{Y \in \V - \{X\}} \R(Y)) \rightarrow \R(X)$.
This mathematical notation just makes precise the fact that 
$F_X$ determines the value of $X$,
given the values of all the other variables in $\U \cup \V$.
%joe4: cut for UAI
\fullv{If there is one exogenous variable $U$ and three endogenous
variables, $X$, $Y$, and $Z$, then $F_X$ defines the values of $X$ in
terms of the values of $Y$, $Z$, and $U$.  For example, we might have 
$F_X(u,y,z) = u+y$, which is usually written as
$X = U+Y$.   Thus, if $Y = 3$ and $U = 2$, then
$X=5$, regardless of how $Z$ is set.%
\footnote{The fact that $X$ is assigned  $U+Y$ (i.e., the value
of $X$ is the sum of the values of $U$ and $Y$) does not imply
that $Y$ is assigned $X-U$; that is, $F_Y(U,X,Z) = X-U$ does not
necessarily hold.}}    

The structural equations define what happens in the presence of external
interventions. 
Setting the value of some variable $X$ to $x$ in a causal
model $M = (\S,\F)$ results in a new causal model, denoted $M_{X
\gets x}$, which is identical to $M$, except that the
equation for $X$ in $\F$ is replaced by $X = x$.

Following most of the literature, we restrict attention here to what
are called {\em recursive\/} (or {\em acyclic\/}) models.
%joe4: line shaving
%This is the special case where there is some
In such models, there is a
total ordering $\prec$ of the endogenous variables
%joe4: line shaving
%(the ones in $\V$)
\fullv{(the ones in $\V$)} 
such that if $X \prec Y$, then
%joe13
%\new{$X$ is not causally influenced by $Y$,}
$X$ is not causally influenced by $Y$,
%sgn1: changed
%$X$ is independent of $Y$, 
that is, $F_X(\ldots, y, \ldots) = F_X(\ldots, y', \ldots)$ for all $y, y' \in
\R(Y)$.  Intuitively, if a theory is recursive, there is no
feedback.  If $X \prec Y$, then the value of $X$ may affect the value of
$Y$, but the value of $Y$ cannot affect the value of $X$.
It should be clear that if $M$ is an acyclic  causal model,
then given a \emph{context}, that is, a setting $\vec{u}$ for the
exogenous variables in $\U$, there is a unique solution for all the
equations.  We simply solve for the variables in the order given by
$\prec$. The value of the variables that come first in the order, that
is, the variables $X$ such that there is no variable $Y$ such that $
Y\prec X$, depend only on the exogenous variables, so their value is
immediately determined by the values of the exogenous variables.  
The values of the variables later in the order can be determined once we have
determined the values of all the variables earlier in the order.

A recursive causal model can be described by a dag (directed acyclic
graph) whose nodes are labeled by variables, and there is an edge from
$X$ to $Y$ if $X \prec Y$.  We can assume without loss of generality
that the equation for $Y$ involves only the parents of $Y$ in the dag.
The roots of the dag are labeled by exogenous variables or endogenous
variables with no parents; all the remaining nodes are labeled by
%joe4
%endogenous variables.
endogenous variables.%
\footnote{Note that the equation for an endogenous variable $X$ with no
parents must be a constant function; e.g., $F_X = 3$.  In the model
$M_{X \gets x}$ that results from $M$ after intervening on $X$, $X$ is
an endogenous variable with no parents.}

A \emph{probabilistic} (functional) causal model is a pair $(M,\Pr)$
consisting of a causal model $M$ and a probability $\Pr$ on the
contexts of $M$.  In the rest of this paper, when we refer to a
``causal model'', we mean a probabilistic functional causal model,
unless we explicitly say otherwise.

A \emph{causal Bayesian network (CBN)} is a tuple $M = (\S,\P)$
described by a signature $\S$, just like a causal model, and a
  collection $\P$ of \emph{conditional 
  probability tables (cpts)}, one for each (endogenous and
  exogenous) variable.%
  %joe13
  \footnote{Some authors (e.g., Pearl \citeyear{pearl:2k} seem to
  assume that CBNs do not include exogenous variables. We find it
  useful to allow them.}
For this paper, we focus on recursive CBNs that can be characterized
by a dag, where there is a bijection between the nodes and the 
(exogenous and endogenous) variables.  The cpt for a variable $X$
quantifies the effects of the parents of $X$
on $X$.  For example, if the parents of $X$ are $Y$ and $Z$ and all
variables are binary, then the
cpt for $X$ would have entries for all $j, k \in \{0,1\}^2$, where the
entry for $(j,k)$ describes$\{Pr(X=0 \mid Y=j, Z=k)$.
(There is no need to have an explicit entry for
$P(X=1 \mid Y=j \cap Z=k),$ since this is just $1 - P(X=0 \mid
Y=j \cap Z=k)$.)
There is also a cpt for the roots of the dag; it is just an
unconditional probability, since a root has no parents.

Just as for causal models, we can also perform interventions in a CBN:
intervening to set the value of some variable $X$ to $x$ in a CBN
$M$ results in a new CBN, denoted $M_{X
\gets x}$, which is identical to $M$, except that now $X$ has no
parents; the cpt for $X$ just gives $X$ value $x$ with probability 1.

Note that we typically use the letter $M$ to refer to both
non-probabilistic causal models {\green and} CBNs, while 
%joe4
%we use $\Pr$ to refer to both the probability in a
%probabilistic causal model and the probability determined by the cpts
%in a CBN, {\pink while we use $P$ to refer to the probability in a cpt}.  We
we use $\Pr$ to refer to the probability on contexts in a
probabilistic causal model.  We use $P$ to refer to the probability in
a cpt.
%joe4
%We hope that, despite the abuse of notation, the meaning 
%will be clear from context.
%sg4
%It is also worth nothing that a causal model can be viewed as a CBN;
It is also worth noting that a causal model can be viewed as a CBN;
the equation $Y= F(\vec{x})$ can be identified with the entry
$P(Y=F(\vec{x})) \mid \vec{X} 
= \vec{x}) = 1$ in a cpt.

\section{Giving semantics to formulas in CBNs}\label{sec:semantics}

\subsection{The problem}\label{sec:problem} Consider the following (standard) language for
reasoning about causality: Given a signature $\S = (\U,\V,\R)$, a
\emph{primitive event} is a formula of the form $X = x$, for  $X \in
\V$ and $x \in \R(X)$.   A {\em causal formula (over $\S$)\/} is one
of the form $[Y_1 \gets y_1, \ldots, Y_k \gets y_k] \phi$,
where
\fullv{
\begin{itemize}
\item}
$\phi$ is a Boolean
combination of primitive events,
\fullv{\item} $Y_1, \ldots, Y_k$ are distinct variables in $\V$, and
\fullv{\item} $y_i \in \R(Y_i)$.
\fullv{\end{itemize}}
Such a formula is
abbreviated
as $[\vec{Y} \gets \vec{y}]\phi$.
The special
case where $k=0$
is abbreviated as
$\phi$.
Intuitively,
$[Y_1 \gets y_1, \ldots, Y_k \gets y_k] \phi$ says that
%$\phi(\vec{u})$ holds in the counterfactual world that would arise if
$\phi$ would hold if
$Y_i$ were set to $y_i$, for $i = 1,\ldots,k$. 
{$\Lcaus(\S)$ is the language consisting of Boolean
    combinations of causal formulas}.  We typically take the
   signature $\S$
to be fixed, and just write $\Lcaus$.  It will be convenient to
  consider a slightly richer language, that we denote
  $\Lcaus^+(\S)$.  It extends $\Lcaus(\S)$ by allowing primitive
events $U=u$, where $U \in \U$, and also allowing interventions on
exogenous variables.%
\footnote{ It is conceptually somewhat inconsistent to allow
interventions on exogenous variables, since their value is assumed to
be determined by factors outside the model, but it is technically
convenient for some of our results.}

{A pair $(M,\vec{u})$ consisting of a (non-probabilistic)
  causal model $M$ and a 
  context $\vec{u}$ is called a \emph{(causal) setting}.} {\blue  
  A formula $\phi \in \Lcaus^+$ is either true or false in a
  setting.
We write $(M,\vec{u}) \sat \psi$  if
the causal formula $\psi$ is true in
the setting $(M,\vec{u})$.
The $\sat$ relation is defined inductively.
$(M,\vec{u}) \sat X = x$ if
the variable $X$ has value $x$
in the
unique (since we are dealing with acyclic models) solution
to
the equations in
$M$ in context $\vec{u}$
(that is, the
unique vector
of values for the exogenous variables that simultaneously satisfies all
equations 
in $M$ 
with the variables in $\U$ set to $\vec{u}$).
The truth of conjunctions and negations is defined in the standard way.
Finally, 
$(M,\vec{u}) \sat [\vec{Y} \gets \vec{y}]\phi$ if 
$(M_{\vec{Y} \gets \vec{y}},\vec{u})_{\vec{Y} \gets \vec{y}}
  \sat \phi$, where $(M_{\vec{Y} \gets \vec{y}}$ is identical to
  $M$ except that the equation for each endogenous variable $Y \in 
  \vec{Y}$ is replaced by $Y = y^*$, where $y^* \in \R(Y)$ is the value in
  $\vec{y}$ corresponding to $Y$, and $\vec{u}_{\vec{Y} \gets
    \vec{y}}$ is identical to $\vec{u}$, except that for each
  exogenous variable $U \in \vec{Y}$, the component of $\vec{u}$
  corresponding to $U$ is replaced by $u^*$,
  where $u^* \in \R(U)$ is the value in $\vec{y}$ corresponding to
  $U$.  (We remark that in a CBN, intervening to set an exogenous
  variable $U$ to $u^*$ is just like any other intervention; we change
  the cpt for $U$ so that $u^*$ gets probability 1.)}
  
%I write $M \sat \phi$ if $(M,\vec{u}) \sat \phi$ for all contexts $\vec{u}$.
In a probabilistic causal model  $(M,\Pr)$, we can assign a probability
to formulas in $\Lcaus$ by taking the probability of a formula $\phi$
in $M$, denoted $\Pr(\phi)$,
to be $\Pr(\{\vec{u}: (M,\vec{u}) \sat \phi\})$.
Thus, the probability of $\phi$ in $M$ is simply the
probability of the set of contexts in which $\phi$ is true; we can
view each formula as corresponding to an event.

When we move to CBNs, things are not so straightforward.  First, while
we still have a probability on contexts, each context determines a
probability on \emph{states}, assignments of values
%joe4
%to variables.  A state clearly determines a truth value of formulas
to variables.  A state clearly determines a truth value for formulas
that do not involve interventions; call such formulas \emph{simple
formulas}.  Thus, we can compute the truth of a simple formula $\phi$
in a context, and 
then using the probability of contexts, determine the probability of $\phi$
in a CBN $M$.  But what about a causal formula such as
$\psi = [\vec{Y} \gets \vec{y}]\phi$?  Given a context $\vec{u}$, we can
determine the model $M' = M_{\vec{Y} \gets \vec{y}}$.  In
$(M',\vec{u})$, $\phi$ is an event whose probability we can
compute, as discussed above.  We can (and will) take this probability to be
the probability of the formula $\psi$ in $(M,\vec{u})$.  But note that
$\psi$ does not correspond to an event in $M$, although we assign it a
probability.

%joe4
%The situation gets worse for a formula like $\psi \land \psi'$, where
%at least one of $\psi$ and $\psi'$ is a causal formula that involves
%an intervention.
The situation gets worse if we add another conjunct $\psi'$ and
consider the formula $\psi \land \psi'$.
While we can use the procedure above to compute the probability of
$\psi$ and $\psi'$ individually in $(M,\vec{u})$, what is the
probability of the conjunction?  {\red Because such formulas do
not correspond to events in $M$, this is not obvious.  {\blue We give
  one approach for defining the probability of a formula in a CBN by 
making one key assumption, which can be
%sg1:shaved a line
viewed as a generalization of Pearl's assumption.}
% made by Pearl.}
%We start by sketching the high-level ideas, then formalize them in
%Section~\ref{sec:technical}.}

{\blue Pearl assumes that mechanisms  
that determine how interventions work (which are the cpts in the case of
CBNs and the structural equations in the case of causal models)
are \emph{autonomous}; he takes
that to mean ``it is conceivable to change one such relationship
without changing the others'' \cite[p. 22]{pearl:2k}.  We go further
and assume, roughly speaking, that they are (probabilistically)
independent.  In a causal 
model, the mechanism for a given variable  (specifically, the outcome
after the intervention)
%sg1: shaved a line
%that results when the variable is intervened on)
is an event, so we
can talk about mechanisms being independent.  While it is not an event
in a CBN, we nevertheless use the assumption that mechanisms are
independent to guide how we determine the probability of formulas in
$\Lcaus$ in a CBN.}

\commentout{
We actually do not want that mechanisms are unconditionally
independent; that
would lead to the probability of a formula such as $\phi \land [X \gets
  1]\psi$ in $M$ to 
being the product of the 
probability of $\phi$ in $M$ and the probability of $[X \gets 1]\psi'$ in $M$.
This would not even give us the right answer in functional causal
models.  If we identify probability 1 with truth in functional models,
then we get the appropriate answer by assuming independence
conditional on $\vec{u}$: $\phi \land [X \gets 1]\psi$ is true (i.e., has
probability 1) in $(M,\vec{u})$ iff both $\phi$ and $[X \gets 1]\psi$ are true
in $(M,\vec{u})$.
}

\blue{
  %\subsection{The suggested solution}
  {\green \subsection{Independence of cpts and complete combinations of
      conditional events}\label{sec:independence}}
{\green To describe our approach, we must first make clear what we mean by}
mechanisms (cpts) being independent.
%conditional on the context.
This has two components: the outcomes of cpts for different variables
%joe4
%are independent, and for a cpt for a single variable Y, the outcomes for
are independent, and for the cpt for a single variable Y, the outcomes for
different settings of the parents of $Y$ are independent.
Indeed, all these outcomes are mutually independent.
{\green We believe that these independence assumptions are quite
%joe4
%  reasonable and, indeed, capture the spirit of Bayesian networks.  In
  reasonable and, capture the spirit of Bayesian networks.  In
%joe13
  %  fact, Halpern and Richardson \citeyear{HR23} showed that the
    fact, in Halpern and Richardson \citeyear{HR23}, it is shown that the
  assumption that cpts involving different variables are independent
  is equivalent to the (conditional) independence assumptions made in
Bayesian networks (see Section~\ref{sec:discussion} for further
discussion).%
%joe13: we should point to an anonymous post
%sgn2:typo
%\footnote{An excerpt of the other submssion giving a formal proof of
\footnote{An excerpt of the other submission giving a formal proof of
%sgn2:added link
%this claim can be found here.}
this claim can be found
 \href{https://drive.google.com/file/d/157tkPaj9SqyDE_HvrEV88q-rrqXZ4igu/view?usp=sharing}{here}.}

In more detail,
suppose that we have a variable $Y_1$ in a CBN $M$ with parents $X_1
\ldots, X_m$.  We want to consider events of the form
 $Y_1 = y_1 \mid
(X_1 = x_1, \ldots, X_m = x_m)$, which we read ``$Y_1 = y_1$ given that
$X_1 = x_1$, \ldots, and $X_m = x_m$''.  Such events have a
probability, given by the cpts for $Y_1$.
%joe4
%{\pink We call such a formula a \emph{conditional event for CBN $M$}.
We call such an event a \emph{conditional event for CBN $M$}.
(Explicitly mentioning the CBN $M$ is necessary, since on the right-hand
  side of the conditional with left-hand side $Y$, we have all the
  parents of $Y$; what the parents are depends on $M$.)}  
Roughly speaking, we identify {
such a conditional event with the formula $[X_1 \leftarrow j_1,
  \ldots, X_m \leftarrow j_m](Y_1 = 1)$. 
This identification already hints at why we 
we care about
conditional events (and their independence). Suppose for simplicity that
$m=1$.  To determine the probability of a formula such as $X_1 = 0 
\land Y_1 =0 \land [X_1 \gets 1](Y_1 = 1)$ we need to apply both the entry
in the cpt for $Y_1=0 \mid X_1 = 0$ and the entry for $Y_1=1 \mid 
X=1$.  They each give a probability; the probability of the formula 
$X_1 = 0 \land Y_1 =0 \land [X_1 \gets 1](Y_1 = 1)$ is the
probability that the conditional events $Y_1=0 \mid X_1 = 0$ and $Y_1=1 \mid 
X=1$ hold simultaneously.  Our independence assumption implies that
this probability is the product of the probability that each of them
holds individually (which is given by the cpt for $Y_1$).

This is an instance of independence within a cpt; we want the
conditional events in a cpt for a variable $Y$ for different
settings of the parents of $Y$ to  be independent.
%conditional on the context.
(Of course,
conditional events for the same setting of the parents, 
such as $Y_1 = 0 \mid X_1 = 1$ and $Y_1 = 1 \mid X_1 = 1$, are not
independent.)  Independence for cpts of different variables is most easily
explained by example: Suppose that $Y_2$ has parents $X_1$ and $X_3$.
Then we want the events $Y_1 = 0 \mid X_1 = 0$ and 
$Y_2 = 1 \mid (X_1 = 0, X_3 =1)$ to be independent.
%conditional on the context.
This independence assumption will be needed to compute the
probability of formulas such as $[X_1 \gets 0](Y_1 = 0) \land
[X_1 = 0, X_3 =1](Y_2 = 1)$.
%and $[X_1 =  0\land Y_1 = 0 \land    [X_1 = 0, X_3 =1)(Y_2 = 1)$.
      %in a context.
As we said, we in fact want to view all
the relevant conditional  events as \emph{mutually} independent.%
\footnote{This implicitly assumes that all exogenous variables are
independent.  We can easily drop this assumption by assuming that
rather than having a separate cpt for each exogenous variable, we just
have a single cpt for contexts.  Nothing in the rest of the discussion
would change if we did this.}

%  conditional on each context.

Although we use the term ``conditional event'', these are not events in a CBN.
%although we will use this intuition shortly to provide a method for
%determining the probability of formulas in a CBN.
On the other hand,
in a causal model, there are corresponding notions that really do
correspond to events.  For example, the conditional event 
$Y_1 = 0 \mid X_1 = 1$ corresponds to the set of contexts
%joe4
%where, when we intervene to set $X_1 = 1$, $Y_1 = 0$. {\purple As we
%  hinted above,   it also corresponds to the set of contexts
where the   formula $[X_1 
%    \leftarrow 1](Y_1 = 1)$ is true; we return to this point shortly.}
    \leftarrow 1](Y_1 = 1)$ is true.
  %In our construction for getting a causal model from a CBN, we will
Starting with a CBN $M$, we will be interested in causal models
for which the
probability $P(Y_1 = 0 \mid X_1 = 1)$, as given by the cpt for
$Y_1$ in $M$, is equal to the probability of the corresponding event in the
causal model.

Going back to CBNs, define a \emph{complete combination of
conditional events  (ccce) for $M$} to be a
conjunction consisting of the
choice of one conditional event for $M$
for each endogenous variable $X$ and each setting of the parents of
%joe3L*: changed fcccce to fccce globally
%$X$.  {\pink A {\em fixed context}   ccce (fcccce)} involves fewer
$X$.  A {\em fixed-context}   ccce (fccce) involves fewer 
conjuncts; we have only conditional events 
{ where for all the exogenous parents $U$ of a variable $X$,
  the value of $U$ is the same as its value in the conjunct
  determining the value of $U$} (the examples should make
clear what this means).} 

    \smallskip
    
    \begin{example}\label{xam:M*}
%joe4
%      {\green Suppose} that we consider the CBN $M^*$ such that
Consider the CBN $M^*$ with the following dag:
%joe4: added figure (which you drew for the later example)
%\begin{center}
  \begin{tikzpicture}
  \tikzstyle{var} = [circle, draw, thick]

  % Define nodes with the style
  \node[var] (U) at (0,0) {U};
  \node[var] (X) at (2,0) {X};
  \node[var] (Y) at (4,0) {Y};

 \draw[->, thick] (U) -- (X);
  \draw[->, thick] (X) -- (Y);
%  \node[label=left:] (F012) at (4.5,-0.25) {,};
  
%joe4*: Sainyam, can you add a comma at the end of the figure
%sg1: @Joe It is not looking nice. can you see if this looks ok.
%joe5: I think it's fine.  Thanks for taking care of it.
  % (I removed center and added comma before where)
%Other option is to uncomment the comma node above
  \end{tikzpicture}
 %     \end{center}
%$U \rightarrow X \rightarrow Y$, all variables are binary,
, where all variables are binary,
and the cpts give the following probabilities: $P(U=0) = a$,
$P(X=0 \mid U=0) = b$, $P(X=0 \mid U=1) = c$, $P(Y=0 \mid
X=0) = d$, and $P(Y=0 \mid X=1) = e$.
%that shown below, 
\commentout{
\begin{verbatim}
    U
   / \
  X   Y
   \ /   
    Z
\end{verbatim}

\noindent
In that case, a
%complete combination of conditional events
ccce
involves a
choice of:
\begin{itemize}
\item one of $X=0\mid U=0$ and $X=1\mid U=0$,
\item one of $Y=0\mid U=0$ and $Y=1\mid U=0$,
\item one of $X=0\mid U=0$ and $X=1\mid U=1$,
\item one of $Y=0\mid U=0$ and $Y=1\mid U=1$,
\item one of $Z=1\mid X=0, Y=0$,
\item one of $Z=0\mid X=0, Y=1$ and $Z=1\mid X=0, Y=1$, 
\item one of $Z=0\mid X=1, Y=0$ and $Z=1\mid X=1, Y=0$,  and
\item one of $Z=0\mid X=1, Y=1$ and $Z=1\mid X=1, Y=1$.
\end{itemize}
In this case, there are 16 ccces.
It is easy to see that the number of
%complete combinations of conditional events
ccce can be
doubly exponential (in the number
of variables), each one involving exponentially many choices.
}
Then a
%complete combination of conditional events
ccce
 consists of 5 conjuncts:
\begin{itemize}
\item one of $U=0$ and $U=1$;
\item one of $X=0\mid U=0$ and $X=1\mid U=0$;
\item one of $X=0\mid U=1$ and $X=1\mid U=1$;  
\item one of $Y=0\mid X=0$ and $Y=1\mid X=0$; and
 \item one of $Y=0\mid X=1$ and $Y=1\mid X=1$.
\end{itemize}
An fccce consist of only 4 conjuncts; it has only one of the second
and third conjuncts of a ccce.  In particular, if $U=0$ is a conjunct of the
fccce, then we
have neither  $X=0\mid U=1$ nor $X=1\mid U=1$  as a conjunct;
similarly, if $U=1$ is a conjunct, then 
we have neither $X=0\mid U=0$ nor $X=1\mid U=0$  as a conjunct.
(This is
what we meant above by saying that each exogenous parent $U$ of $X$ must
have the same value as in conjunct that determines $U$'s value.)
\end{example}

It is not hard to show that, {in this case, there are 32 ccces
  and 16 fccces.  Moreover, each fccce is equivalent to a
  disjunction of ccces (not just in this example, but in general).  
The number of
%complete combinations of conditional events
ccces and fccces}
can be as high as doubly exponential
(in the number of variables), each one involving exponentially many
%joe4
%choices (since if a variable $Y$ has $n$ parents, each of them binary,
choices.  For example, if a variable $Y$ has $n$ parents, each of them binary,
there are $2^n$ possible settings of the parents of $Y$, and we must
choose one value of $Y$ for each of these $2^n$ settings,
%joe1
%so there are already $2^{2^n}$ choices.
{\green already giving us $2^{2^n}$ choices}.
%joe4
%It is easy to see that there is also a double-exponential upper boun.
It is easy to see that there is also a double-exponential upper bound.
%sg1:typo

If we think of a conditional event of the form $Z=1\mid X=0, Y=0$ as saying
%sg4
%saying ``if $X$ were (set to) 0 and $Y$ were (set to) 0, then $Z$
 ``if $X$ were (set to) 0 and $Y$ were (set to) 0, then $Z$
would be 1'', then given a
%complete combination of conditional events
ccce
and a formula $\phi \in \Lcaus$ and context $\vec{u}$, we can
determine if $\phi$ is true or false.
%(As we shall see, many of the conditional events 
%in the complete combination are irrelevant to the determination.)
{ We formalize this shortly.  We can then take the probability of
  $\phi$ to be the sum of   the probabilities of the
  %complete   combinations
  ccces that 
  make $\phi$ true.   The probability of a
  %conditional   event
  ccce
  is determined by the corresponding entry of the cpt.  Thus, if 
we further assume independence, we can determine the probability of
each
%complete combination,
ccce,
and hence the probability of any formula $\phi$.}  { We now give some
informal  examples of how this works, and then formalize the procedure
in Section~\ref{sec:CBNsemantics}.}
%For example to determine $\phi = X=0 \land Y=0 \land [X \gets
%  1](Y=1)$ is true in context $U=0$ only for the two complete combinations
%that include $X=0 \mid U=0$, $Y=0 \mid X=0$, and $Y=1 \mid X=1$.
%We take the probability of a conditional event to be determined by the
%corresponding entry of the cpt and assume independence, so we can
%determine the probability of each complete combination.

\begin{example}\label{xam:M*continued} 
In the CBN $M^*$ described in Example~\ref{xam:M*}, 
there are two
%complete combinations
{ fccces
where $\phi = X=0 \land Y=0 \land [X \gets
  1](Y=1)$ is true: (a) $U=0 \land (X=0 \mid U=0)
\land (Y=0 \mid X=0) \land (Y=1 \mid X=1)$;
%(and either $X=0 \mid U=1$ or $X=1 \mid U=1$; 
and  (b) $U=1 \land (X=0
\mid U=1) \land (Y=0 \mid X=0) \land (Y=1 \mid X=1)$.
%(in this case, the value of
%$X$ when $U=0$ is irrelevant).
Each of these two fccces is the disjunction of two ccces, which
extend the fccce by adding a fifth conjunct.  For example, for the first fccce,
we can add either the conjunct $X=0 \mid U=1$ or the conjunct $X=1 \mid U=1$. 
The total probability of these two
%complete descriptions
fccces
is $abd(1-e) + (1-a)cd(1-e)$; this is the probability
of $\phi$ in $M^*$.}
\end{example}
%joe4: moved after example
%We denote by $\Pr_M(\phi)$ the probability of a formula $\phi$ in a ...
%CBN or causal model $M$.
%joe3: added next sentence
%(We provide a formal definition of $\Pr_M(\phi)$ for a CBN $M$ at the
%end of Section~\ref{sec:semantics}.)
%(The superscript $C$ is used to distinguish the
%probability of formulas in CBNs from the probability in causal models,
%for which we use $\Pr^F$.)

%joe4: line shaving
%We give one more example to show how this calculation works.
We give one more example of this calculation.

\begin{example}\label{xam:Mdag}
Consider the model CBN $M^\dag$, which differs from $M^*$ in that
now $U$ is also a parent of $Y$; the dag is shown below.  $M^*$ and
$M^\dag$ have the same 
cpts for $U$ and $X$; the cpt of $Y$ in $M^\dag$ is $P(Y=0 \mid
U=0, X=0) = f_1$, $P(Y=0 \mid U=0, X=1) = f_2$,
 $P(Y=0 \mid U=1, X=0) = f_3$, $P(Y=0 \mid U=1, X=1) = f_4$.
%sg1: changed the figure
%joe5: looks good; thanks.
   
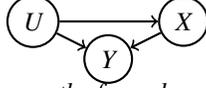
\begin{center}
  \vspace{-4mm}
\begin{tikzpicture}
  \tikzstyle{var} = [circle, draw, thick]
  % Define nodes with the style
  \node[var] (U) at (0,0) {U};
  \node[var] (X) at (2,0) {X};
    \node[var] (Y) at (1,-0.5) {Y};

  \draw[->, thick] (U) -- (X);
    \draw[->, thick] (U) -- (Y);
      \draw[->, thick] (X) -- (Y);
\end{tikzpicture}
\vspace{-5mm}
\end{center}
\commentout{
\begin{verbatim}
    U
   / \
  |   X
   \ /   
    Y
\end{verbatim}
}
{ Now there are 128 ccces, but only 16 fccces; 
%sixteen complete combinations
the formula $\phi = X=0 \land Y=0 \land [X \gets
  1](Y=1)$ is true in only two of these fccces: (a) $U=0 \land
(X=0 \mid U=0) \land (Y=0 \mid (U=0, X=0)) \land 
(Y=1 \mid (U=0, X=1))$; and
%(the value of $X$ when $U=1$ is irrelevant,
%as is the value of $Y$ when $U=1$) and
(b) $U=1 \land )X=0
\mid U=1) \land (Y=0 \mid (U=1, X=0)) \land 
(Y=1 \mid (U=1, X=1))$.  It is easy to check that
$\Pr_{M^\dag}(\phi) = abf_1(1-f_2) + (1-a)cf_3(1-f_4)$.
%joe4: cut; it's not true in general
%Thus,
%adding a parent to $Y$ has a significant effect on the number of
%ccces, but none on the number of fccces nor on the number of fccces
%that make $\phi$ true.} 
The calculation of the probability of $\phi$ is
essentially the same in $M^*$ and $M^\dag$.
%joe13*: I had trouble following the next two sentences.  I just cut
%them.  I think we discuss these issues more accurately in other
%places in the paper.
%Note that the independence of cpts requires conditioning on
%the parents (endogenous and exogenous variables) of $Y$. Choosing all
%external factors that cause dependence across different settings of  $Y$
%as $U$ would always ensure the independence to hold.
%joe4
%On the other hand, if $Y$ had
%an additional endogenous parent, say $X'$, this would greatly
%increase the number {\pink of fccces}.
%complete combinations,
}
\end{example}

We denote by $\Pr_M(\phi)$ the probability of a formula $\phi$ in a 
CBN or causal model $M$.
(We provide a formal definition of $\Pr_M(\phi)$ for a CBN $M$ at the
end of Section~\ref{sec:semantics}.)

\subsection{Giving semantics to formulas in CBNs}\label{sec:CBNsemantics}

{ We already hinted in Examples~\ref{xam:M*continued}  and~\ref{xam:Mdag}
how we give semantics to formulas in CBNs.  We now formalize this.

The first step is to show that a
ccce (resp., fccce) determines the
  truth of a formula in $\Lcaus^+(\S)$ (resp., $\Lcaus(\S)$) in a
  %sg4
  %causal model.    To make this prexcise,
  causal model.    To make this precise,
  we need a few definitions. 
  We take the \emph{type} of a CBN $M = (\S, \P)$, where $\S=(\U,
  \V,\R)$ to consist of its 
  signature $\S$ and, for each endogenous variable, a list of its
  parents (which is essentially given by the dag associated with $M$,
  without the cpts).  A causal model $M' = (\S',\F')$ has the same type as 
$M$ if $\S' = (\U \union \U', \V, \R')$, where 
  $\U'$ is arbitrary, $\R'|_{\U \union \V} = \R$, and $\F'$ is
  such that each endogenous variable $X$ depends on the same
  variables in $\U \union \V$ according to $\F'$ as it does according
  to the type of $M$ (but may also depend on any subset of $\U'$).  

%  Given a (fc)ccce $\alpha$ {\pink for $M$}, a context $\vec{u} \in \R(\U)$
%is \emph{compatible 
%with $\alpha$} if $\vec{u}$ agrees with $\alpha$ on all events of the
%form $U=u$.  Note that there is a unique context $\vec{u}$ compatible
%with a (fc)ccce $\alpha$ (although a context may be compatible wth
%many (fc)ccces). 
%joe3*: made this a formal definition, so we can refer to it later
%  {\pink For each conditional event $Y=y \mid (X_1 = x_1, \ldots, X_m =
%  x_m)$ in a ccce, let the \emph{corresponding formula} be $[X_1 \gets
\begin{definition}\label{dfn:corresponding}
  For the conditional event $Y=y \mid (X_1 = x_1, \ldots, X_m =
  x_m)$, let the \emph{corresponding formula} be $[X_1 \gets
    x_1, \ldots, X_m \gets x_m](Y=y)$.  (Note that the corresponding
  formula may be in $\Lcaus^+ - \Lcaus$, since some of the $X_i$s
  may be exogneous.)
%joe3
%  We define the formula $\phi_\alpha \in \Lcaus^+(\S)$, the formula
Let $\phi_\alpha \in \Lcaus^+(\S)$, the formula
  corresponding to the ccce $\alpha$,
%joe3
%  as the conjunction of the formulas corresponding to the conditional
be the conjunction of the formulas corresponding to the conditional
  events in $\alpha$.  We can similarly define the formula corresponding
  %sg1:coresponding
  to an fccce.
\end{definition}
}

\begin{example}  In the model $M^\dag$ of Example~\ref{xam:Mdag},
if $\alpha$ is the fccce $U=0 \land (X=0
\mid U=0) \land (Y=0 \mid (U=0, X=0)) \land (Y=1 \mid (U=0, X=1))$,
%then $u_\alpha = 0$ (so we consider the context $U=0$)  and
then $\phi_\alpha$ is $U=0 \land [U\gets 0]X=0 \land [U\gets 0, X \gets
  0](Y=0) \land [U \gets 0, X \gets 1](Y=1)$.
%  Similarly, if $\beta$ is the other ccce,
%$U=1$, $X=0
%\mid U=1$, $Y=0 \mid (U=0, X=0)$, 
%and $Y=1 \mid (U=0, X=1)$, then $u_\beta = 1$ and $\phi_\beta =
%\phi_\alpha$.}
\end{example}

 Say that a formula $\psi$ is \emph{valid with respect to a CBN $M$} if
  %joe4
%    $(M',\vec{u}) \sat \psi$ for all settings $(M',\vec{u})$ where $M'$ is
  $(M',\vec{u}) \sat \psi$ for all causal settings $(M',\vec{u})$, where $M'$ is
  a causal model with the same type as $M$.
%The following theorem makes precise the sense the sense in which a
%sg2:typo
The following theorem makes precise the sense in which a
ccce determines whether or not an arbitrary formula is true.

 \begin{theorem}\label{ccce} Given a CBN $M = (\S,\P)$ and a
  ccce (resp., fccce) $\alpha$, 
  then  for all formulas {$\psi \in \Lcaus^+(\S)$ (resp., $\psi
    \in \Lcaus(\S)$})
either $\phi_\alpha
\Rightarrow \psi$ is valid with respect to $M$ or $\phi_\alpha
\Rightarrow \neg \psi$ is valid with respect to $M$.
  \end{theorem}

%joe4*: This should go in the appendix (although I corrected a few
%typos in the proof)
\fullv{
{\purple \prf  We show that if two causal models
  $M_1$ and $M_2$ have the same type as $M$ and 
    $\vec{u}_1$ and $\vec{u}_2$ are contexts
%    that are both compatible with {\pink the ccce (resp., fccce) $\alpha$ 
  such that   $(M_1,\vec{u}_1) \sat \phi_\alpha$ and $(M_2,\vec{u}_2) \sat
  \phi_\alpha$, then for all formulas {\pink $\psi \in  \Lcaus^+(\S)$
(resp., $\psi \in \Lcaus(\S)$)}, we have that 
\begin{equation}\label{eq1}
\mbox{$(M_1,\vec{u}_1) \sat \psi$ iff $(M_2,\vec{u}_2) \sat
  \psi$.}
\end{equation}
  The claimed result follows immediately.

{\pink We give the proof in the case that $\alpha$ is a ccce and $\psi \in
  \Lcaus^+(\S)$.  The modifications needed to deal with the case that
%joe4
%$\alpha$ is an fccce and $\psi in \Lcaus(\S)$ are straightforward and
$\alpha$ is an fccce and $\psi \in \Lcaus(\S)$ are straightforward and
  left to the reader.}
  Since $M$ is acyclic, we can order the {\pink exogenous and endogenous} variables
topologically.  Let $X_1, \ldots, X_m$ be such an ordering.  We first
prove 
by induction on $j$ that, for all interventions $\vec{Y} \gets
\vec{y}$ (including the empty intervention) and $x_j \in \R(X_j)$, 
$(M_1,\vec{u}_1) \sat [\vec{Y} \gets \vec{y}](X_j = x_j)$ iff
$(M_2,\vec{u}_2) \sat [\vec{Y} \gets \vec{y}](X_j = x_j)$.

%joe4: aded paragraph break
For the
base case, {\pink $X_1$ must be exogenous, and hence have no parents.}
If $X_1$ is not one of the variables in $\vec{Y}$,
then we must have  $(M_1,\vec{u}_1) \sat [\vec{Y} \gets
%joe4
  %  \vec{y}](X_1 = x_1)$ iff $(M_1,\vec{u}_1) \sat (X_1 = x_1)$; since
    \vec{y}](X_1 = x_1)$ iff $(M_1,\vec{u}_1) \sat (X_1 = x_1)$, and
similarly for $M_2$; since
no variable in $\vec{Y}$ is a parent of $X_1$, intervening on $\vec{Y}$
%joe4
%has no effect on $X_1$.  Since $(M,\vec{u}_1) \sat \phi_\alpha$ and
%$(M,\vec{u}_2) \sat \phi_\alpha$,
%they agree on the values of variables in $\U$.  Thus,
has no effect on $X_1$.  Since $(M_1,\vec{u}_1) \sat \phi_\alpha$ and
$(M_2,\vec{u}_2) \sat \phi_\alpha$,
$M_1$ and $M_2$ agree on the values of variables in $\U$.  Thus,
$(M_1,\vec{u}_1) \sat (X_1 = x_1)$ iff
$(M_2,\vec{u}_2) \sat (X_1 = x_1)$.
%joe4
%Then, as for $(M_1, \vec{u}_1)$,
%we have $(M_2,\vec{u}_2) \sat (X_1 = x_1)$ iff
%$(M_2,\vec{u}_2) \sat [\vec{Y} \gets   \vec{y}](X_1 = x_1)$.
It follows that $(M_2,\vec{u}_2) \sat [\vec{Y} \gets   \vec{y}](X_1 = x_1)$.
$(M_2,\vec{u}_2) \sat [\vec{Y} \gets   \vec{y}](X_1 = x_1)$, as desired.

On the other hand, if $X_1$ is one of the variables in $\vec{Y}$
(which can happen only if the formula is in $\Lcaus^+(\S)$),
let $x^*$ be the value in $\vec{y}$ corresponding to $X_1$.  In that
case, the formula $[\vec{Y} \gets   \vec{y}](X_1 = x^*)$ is valid with
respect to
$M$. It follows that $(M_1,\vec{u}_1) \sat [\vec{Y} \gets
  \vec{y}](X_1 = x_1)$ iff $x_1 = x^*$, and similarly for $(M_2 ,
\vec{u}_2)$.  The desired result follows.  This completes the proof for
the base case. 

Now suppose that we have proved the
result for $j < m$.  Let $Z_1, \ldots, Z_k$ be the
%non-exogenous
parents of $X_{j+1}$ in $M$.  Since $X_1, \ldots, X_m$ is a
topological sort, we must have  $\{Z_1, \ldots, Z_k\} \subseteq \{X_1,
\ldots, X_j\}$.  Let $z_1, \ldots, z_k$ be values in $\R(Z_1), \ldots,
\R(Z_k)$, respectively, such that
$(M_1,\vec{u}_1) \sat [\vec{Y} \gets  \vec{y}](Z_h = z_h)$, for $h= 1,
\ldots, k$.  By the induction hypothesis,
$(M_2,\vec{u}_2) \sat [\vec{Y} \gets  \vec{y}](Z_h =
z_h)$, for $h= 1, \ldots, k$.
Moreover, it is easy to see that
$([\vec{Y} \gets \vec{y}]\phi \land [\vec{Y} \gets \vec{y}]\phi'))
\dimp  [\vec{Y} \gets \vec{y}](\phi \land \phi')$ is valid with
respect to $M$.
Thus, $(M_1,\vec{u}_1) \sat [\vec{Y} \gets  \vec{y}](Z_1 = z_1 \land
\ldots Z_k = z_k)$ and similarly for $(M_2,\vec{u}_2)$.  Moreover,
since $Z_1, \ldots, Z_k$ are the
%non-exogenous
parents of $X_{j+1}$,
it follows that  $(M_1,\vec{u}_1) \sat [\vec{Y} \gets
  \vec{y}](X_{j+1} = x_{j+_1})$ iff $[Z_1 = z_1 \land
\ldots Z_k = z_k)](X_{j+1} = x_{j+_1})$ is a conjunct of $\phi_\alpha$.
Since $(M_1,\vec{u}_1) \sat \phi_\alpha$ and $(M_2,\vec{u}_2) \sat
\phi_\alpha$, the desired result follows, completing the induction proof.

The argument that $(M_1,\vec{u}_1) \sat [\vec{Y} \gets \vec{y}]\psi$ iff
$(M_2,\vec{u}_2) \sat [\vec{Y} \gets \vec{y}]\psi$ for arbitrary
(simple) formulas $\psi$ now follows from the fact that (as we already
observed)
$([\vec{Y} \gets \vec{y}]\phi \land [\vec{Y} \gets \vec{y}]\phi'))
\dimp  [\vec{Y} \gets \vec{y}](\phi \land \phi')$ is valid with
respect to $M$, as are
$([\vec{Y} \gets \vec{y}]\phi \lor [\vec{Y} \gets \vec{y}]\phi'))
\dimp  [\vec{Y} \gets \vec{y}](\phi \lor \phi')$
and 
$[\vec{Y} \gets \vec{y}]\neg\phi \dimp
\neg [\vec{Y} \gets \vec{y}]\phi$.  

Finally, we can deal with Boolean combinations of causal formulas by a
straightforward induction.  This completes the argument that
(\ref{eq1}) holds for all formulas in $\psi \in \Lcaus^+(\S)$.
\eprf}
}
%joe4: \end{fullv}

%joe4
%\shortv{\prf  We show that if two causal models
 % $M_1$ and $M_2$ have the same type as $M$ and 
  %  $\vec{u}_1$ and $\vec{u}_2$ are contexts
%%    that are both compatible with {\pink the ccce (resp., fccce) $\alpha$ 
 % such that   $(M_1,\vec{u}_1) \sat \phi_\alpha$ and $(M_2,\vec{u}_2) \sat
 % \phi_\alpha$, then for all formulas $\psi \in  \Lcaus^+(\S)$
%(resp., $\psi \in \Lcaus(\S)$), we have that 
 % \begin{equation}\label{eq1}
%\mbox{$(M_1,\vec{u}_1) \sat \psi$ iff $(M_2,\vec{u}_2) \sat
 % \psi$.}
%\end{equation}
%  The claimed result follows immediately.  The details of the proof
%  can be found in the appendix.
%\eprf
%}

%joe5
{ \prf  We show that if two causal models
%{ \prf  As we said, we show that if two causal models
  $M_1$ and $M_2$ have the same type as $M$ and 
    $\vec{u}_1$ and $\vec{u}_2$ are contexts
%    that are both compatible with {\pink the ccce (resp., fccce) $\alpha$ 
  such that   $(M_1,\vec{u}_1) \sat \phi_\alpha$ and $(M_2,\vec{u}_2) \sat
  \phi_\alpha$, then for all formulas  $\psi \in  \Lcaus^+(\S)$
(resp., $\psi \in \Lcaus(\S)$), we have that 
%joe5
    \begin{equation}\label{eq1}
\mbox{$(M_1,\vec{u}_1) \sat \psi$ iff $(M_2,\vec{u}_2) \sat
  \psi$.}
  \end{equation}
  %(\ref{eq1}) (as defined in the main text) holds.
    The claimed result follows immediately.

{ We give the proof in the case that $\alpha$ is a ccce and $\psi \in
  \Lcaus^+(\S)$.  The modifications needed to deal with the case that
%joe4
%$\alpha$ is an fccce and $\psi in \Lcaus(\S)$ are straightforward and
$\alpha$ is an fccce and $\psi \in \Lcaus(\S)$ are straightforward and
  left to the reader.}
  Since $M$ is acyclic, we can order the  exogenous and endogenous variables
topologically.  Let $X_1, \ldots, X_m$ be such an ordering.  We first prove
by induction on $j$ that, for all interventions $\vec{Y} \gets
\vec{y}$ (including the empty intervention) and $x_j \in \R(X_j)$, 
$(M_1,\vec{u}_1) \sat [\vec{Y} \gets \vec{y}](X_j = x_j)$ iff
$(M_2,\vec{u}_2) \sat [\vec{Y} \gets \vec{y}](X_j = x_j)$.

%joe4: aded paragraph break
For the
base case, $X_1$ must be exogenous, and hence have no parents.
If $X_1$ is not one of the variables in $\vec{Y}$,
then we must have  $(M_1,\vec{u}_1) \sat [\vec{Y} \gets
%joe4
  %  \vec{y}](X_1 = x_1)$ iff $(M_1,\vec{u}_1) \sat (X_1 = x_1)$; since
    \vec{y}](X_1 = x_1)$ iff $(M_1,\vec{u}_1) \sat (X_1 = x_1)$, and
similarly for $M_2$; since
no variable in $\vec{Y}$ is a parent of $X_1$, intervening on $\vec{Y}$
%joe4
%has no effect on $X_1$.  Since $(M,\vec{u}_1) \sat \phi_\alpha$ and
%$(M,\vec{u}_2) \sat \phi_\alpha$,
%they agree on the values of variables in $\U$.  Thus,
has no effect on $X_1$.  Since $(M_1,\vec{u}_1) \sat \phi_\alpha$ and
$(M_2,\vec{u}_2) \sat \phi_\alpha$,
$M_1$ and $M_2$ agree on the values of variables in $\U$.  Thus,
$(M_1,\vec{u}_1) \sat (X_1 = x_1)$ iff
$(M_2,\vec{u}_2) \sat (X_1 = x_1)$.
%joe4
%Then, as for $(M_1, \vec{u}_1)$,
%we have $(M_2,\vec{u}_2) \sat (X_1 = x_1)$ iff
%$(M_2,\vec{u}_2) \sat [\vec{Y} \gets   \vec{y}](X_1 = x_1)$.
It follows that $(M_2,\vec{u}_2) \sat [\vec{Y} \gets   \vec{y}](X_1 = x_1)$.
$(M_2,\vec{u}_2) \sat [\vec{Y} \gets   \vec{y}](X_1 = x_1)$, as desired.

On the other hand, if $X_1$ is one of the variables in $\vec{Y}$
(which can happen only if the formula is in $\Lcaus^+(\S)$),
let $x^*$ be the value in $\vec{y}$ corresponding to $X_1$.  In that
case, the formula $[\vec{Y} \gets   \vec{y}](X_1 = x^*)$ is valid with
respect to
$M$. It follows that $(M_1,\vec{u}_1) \sat [\vec{Y} \gets
  \vec{y}](X_1 = x_1)$ iff $x_1 = x^*$, and similarly for $(M_2 ,
\vec{u}_2)$.  The desired result follows.  This completes the proof for
the base case. 

Now suppose that we have proved the
result for $j < m$.  Let $Z_1, \ldots, Z_k$ be the
%non-exogenous
parents of $X_{j+1}$ in $M$.  Since $X_1, \ldots, X_m$ is a
topological sort, we must have  $\{Z_1, \ldots, Z_k\} \subseteq \{X_1,
\ldots, X_j\}$.  Let $z_1, \ldots, z_k$ be values in $\R(Z_1), \ldots,
\R(Z_k)$, respectively, such that
$(M_1,\vec{u}_1) \sat [\vec{Y} \gets  \vec{y}](Z_h = z_h)$, for $h= 1,
\ldots, k$.  By the induction hypothesis,
$(M_2,\vec{u}_2) \sat [\vec{Y} \gets  \vec{y}](Z_h =
z_h)$, for $h= 1, \ldots, k$.
Moreover, it is easy to see that
$([\vec{Y} \gets \vec{y}]\phi \land [\vec{Y} \gets \vec{y}]\phi'))
\dimp  [\vec{Y} \gets \vec{y}](\phi \land \phi')$ is valid with
respect to $M$.
Thus, $(M_1,\vec{u}_1) \sat [\vec{Y} \gets  \vec{y}](Z_1 = z_1 \land
\ldots Z_k = z_k)$ and similarly for $(M_2,\vec{u}_2)$.  Moreover,
since $Z_1, \ldots, Z_k$ are the
%non-exogenous
parents of $X_{j+1}$,
it follows that  $(M_1,\vec{u}_1) \sat [\vec{Y} \gets
  \vec{y}](X_{j+1} = x_{j+_1})$ iff $[Z_1 = z_1 \land
\ldots Z_k = z_k)](X_{j+1} = x_{j+_1})$ is a conjunct of $\phi_\alpha$.
Since $(M_1,\vec{u}_1) \sat \phi_\alpha$ and $(M_2,\vec{u}_2) \sat
\phi_\alpha$, the desired result follows, completing the induction proof.

The argument that $(M_1,\vec{u}_1) \sat [\vec{Y} \gets \vec{y}]\psi$ iff
$(M_2,\vec{u}_2) \sat [\vec{Y} \gets \vec{y}]\psi$ for arbitrary
(simple) formulas $\psi$ now follows from the fact that (as we already
observed)
$([\vec{Y} \gets \vec{y}]\phi \land [\vec{Y} \gets \vec{y}]\phi'))
\dimp  [\vec{Y} \gets \vec{y}](\phi \land \phi')$ is valid with
respect to $M$, as are
$([\vec{Y} \gets \vec{y}]\phi \lor [\vec{Y} \gets \vec{y}]\phi'))
\dimp  [\vec{Y} \gets \vec{y}](\phi \lor \phi')$
and 
$[\vec{Y} \gets \vec{y}]\neg\phi \dimp
\neg [\vec{Y} \gets \vec{y}]\phi$.  

Finally, we can deal with Boolean combinations of causal formulas by a
straightforward induction.  This completes the argument that
(\ref{eq1}) holds for all formulas in $\psi \in \Lcaus^+(\S)$.
\eprf}

%joe4: \end{shortv}

%joe3
%{\pink We now take the probability of a formula $\phi \in \Lcaus^+(\S)$ in a
{ Based on this result, we can take the probability of a formula
  $\phi \in \Lcaus^+(\S)$ in a 
CBN $M$ to be the probability of the ccces that imply it.
To make this precise, given  a CBN $M$, say that a probabilistic
causal model $(M',\Pr)$ is 
\emph{compatible} with $M$ if $M'$ has the same type as $M$, and the
probability $\Pr$ is such that all the cpts in $M$ get the right
probability in $M$.  More precisely, for each endogenous variable $Y$
in $M$, if $X_1, \ldots, X_k$ are the parents of $Y$ in $M$, then
for each entry $P(Y=y \mid X_1 = x_1, \ldots, X_k = x_k) = a$ in the
cpt for $Y$, $\Pr$ is
such that the corresponding formula $[X_1 \gets x_1, \ldots, X_k \gets
  x_k](Y=y)$ gets probability $a$.  $(M',\Pr)$ is \emph{i-compatible}
with $M$ (the \emph{i} stands for \emph{independence}) if it is
compatible with $M$ and, in addition,  $\Pr$ is such that
the events described by the formulas corresponding to entries for cpts
for different variable
(i.e. the set of contexts in $M$ that make these formulas true) are
independent, as are the events described by the formulas corresponding
to different entries for the cpt for a given variable.  Thus, for
example, if $(x_1', \ldots, x_k') \ne (x_1, \ldots, x_k)$, then we
want the events described by $[X_1 \gets x_1, \ldots, X_k \gets
  x_k](Y=y)$ and $[X_1 \gets x_1', \ldots, X_k \gets x_k'](Y=y)$ to
be independent (these are different entries of the cpt for $Y$); and
if $Y' \ne Y$ and has parents $X_1', \ldots, X_m'$ in $M$, then we
want the events described by $[X_1 \gets x_1, \ldots, X_k \gets
  x_k](Y=y)$ and $[X_1' \gets x_1', \ldots, X_m \gets
  x_m'](Y'=y')$ to be independent (these are entries of cpts for
different variables).

\begin{theorem}\label{thm:same} Given a CBN $M$ and a formula $\phi
  \in \Lcaus^+(\S)$, 
  the probability of $\phi$ is the same in all causal models $M'$
  i-compatible with $M$.\label{thm:icompatible}
  \end{theorem}

\prf
It follows from Theorem~\ref{ccce} that the probability of $\phi$ is
the sum of the probabilities of the formulas $\phi_\alpha$ for the
ccces $\alpha$ such that $\phi_\alpha \Rightarrow \phi$ is valid.  It is
immediate that these formulas have the same probability in all causal
models i-compatible with $M$. \eprf

%joe3
%We take the probability of $\phi$ in a causal model i-compatible with
%$M$ to be the probability of $M$ in $\phi$.
Formally, we take $\Pr_M(\phi)$, the probability of $\phi$ in the CBN
$M$,  to be $\Pr_{M'}(\phi)$ for a causal model  $M'$   i-compatible with 
$M$.  By Theorem~\ref{thm:same}, it does not matter which causal model
%joe4
%$M'$ i-=compatible with $M$ we consider.
$M'$ i-compatible with $M$ we consider.
Note for future reference
that if we had considered only causal models compatible with $M$,
dropping the independence assumption, we would have gotten a range of
probabilities.  }

\subsection{Discussion}\label{sec:discussion}
{\blue
  Four points are worth making: First, note that this way of assigning
  probabilities in a CBN $M$ 
  always results in the probability of a formula $\phi \in \Lcaus^+$
  being a sum of 
  products of entries in the cpt.  Thus, we can in principle compute
%joe4
  %  the probabilities of events (and conditional events) involving
    the probabilities of (conditional) events involving
  interventions from observations of statistical frequencies (at
%joe13
  %  least, as long as settings of the parents of a variable in the
%  relevant entries of the cpt have positive probability.
  least, as long as all settings of the parents of a variable in the
  relevant entries of the cpt have positive probability).
%joe4
%  we return to   this issue in Section~\ref{sec:conditioning}

  Second, the number of
  %  complete combinations of conditional events
  ccces
  may make the
computation of the probability of a formula in a CBN seem unacceptably
%sg4
%high.  As the examples above show, in practice, it is not so bad.  For
high.  As the examples above shows, in practice, it is not so bad.  For
example,  we typically do not actually have to deal with
%complete combinations.
ccces.
{For one thing, it follows from Theorem~\ref{ccce} that to
  compute the probability of $\phi \in \Lcaus$, it suffices to
  consider fccces.  Moreover,
when computing $\Pr_M(\phi)$ where
$\phi$ involves an intervention of the form $X \gets x$, we can ignore 
the entries in the cpts involving $X$, and for variables for which $X$
is a parent, we consider only entries in the cpts where $X=x$.  
We  can also take advantage of the structure of the formula whose
probability we are interested in computing to further simplify the
computation, although the details are beyond the scope of this paper.
}
  
{\red 
%joe4
  %  Third, as mentioned above, an intervention formula does not
    Third, as mentioned above, a formula involving interventions does not
correspond
%joe1
%to an event in an obvious way in a CBN.  It is worth looking more closely at
%joe5: line shaving
%{\green in an obvious way to an event} in a CBN.  It is worth looking
%more closely at why it does correspond to an event in a (functional)
in an obvious way to an event in a CBN, but it 
does correspond to an event in a (functional)
causal model.  The 
key point is that in a causal model, a context not only determines a
state; it determines a state for every intervention.
%joe4
\fullv{
(Indeed, this is
what Peters and Halpern \citeyear{PH21} took to be the key feature of
a context, and how they generalized contexts in what they called
\emph{GSEMs}---Generalized Structural Equation Models.%
{\blue \footnote{Actually, Peters and Halpern allow a context to determine a
\emph{set} of states, not necessarily a single state, since they want
%joe4
%to generalize causal models that may not be recursive.}}
to generalize causal models that may not be recursive.}})
}
We can view
%joe9: we haven't defined "intervention formula" (the definition was
%there at one point, but got commented out)
%an intervention formula as an event in a space whose elements are
a formula involving interventions as an event in a space whose elements are
functions from interventions to worlds.  Since a context can be viewed
%joe4
%this way, we can view an intervention formula (and, indeed, arbitrary
%formulas in $\Lcaus^+$) as events in such a space.  This makes
%this way, we can view a formula involving intereventnions
%sg1:typo
this way, we can view a formula involving interventions
as an event in such a space.  This makes
conditioning on arbitrary formulas in $\Lcaus^+$ (with positive
probability) in causal models well defined.
By way of contrast, in a CBN, we can view a context as a function from
interventions to distributions over worlds.
%To make our semantics
%for formulas in $\Lcaus^+$ in CBNs more formal, in the next section,
%we relate the semantics of CBNs and causal models in a more formal way.
Finally, it is worth asking how reasonable is the assumption that cpts
are independent, that is, considering i-compatible causal models rather
than just compatible causal models, which is what seems to have been
done elsewhere in the literature (see, e.g., \cite{BP94,TP00}).}

%joe13
As we said, Halpern and Richardson \citeyear{HR23} have shown that the
%As we said, another NeurIPS submission shows that the
assumption that cpts involving different variables are independent is
equivalent to the (conditional) independence assumptions made in
Bayesian networks.  More
precisely, given a CBN $M$, let $M'$ be the non-probabilistic causal
%joe13
%model constructed above.  Then Halpern and Richardson show that if the
model constructed above.  Then it is shown that if the
probability $\Pr'$ 
makes interventions on different variables independent (i.e., if
$\Pr'(\vec{U}, f_1, \dots, f_m) = \Pr(\vec{u}) \times \Pr_{Y_1}(f_1)
\times \cdots \times \Pr_{Y_m}(f_m)$, as in our construction), then 
all the conditional independencies implied by
d-separation  hold in $(M,\Pr')$ (see \cite{Pearl} for the formal definition of
d-separation and further discussion).   Conversely,
if all the dependencies implied by d-separation hold in $(M,\Pr')$, 
then $\Pr'$ must make interventions on different variables
independent.

This result says nothing about making interventions for different settings of
the parents of a single variable independent.  This is relevant only
if we are interested in computing the probability of
formulas
%joe4
%involving interventions,
such as $X=0 \land Y=0 \land [X \gets
  1](Y=1)$, for which we need to consider (simultaneously) the cpt for
%joe4
%$Y$ when $X=0$ and when $X=1$.  But we would argue that independence is
%joe13
%$Y$ when $X=0$ and when $X=1$.  Independence seems
%reasonable in this case too, since by considering settings of all the
%parents of a variable $Y$, we are effectively screening off $Y$ from
$Y$ when $X=0$ and when $X=1$.  As discussed earlier, independence
is reasonable in this case if we can observe all the parents
of a variable $Y$, and thus screen off $Y$ from 
the effects of all other variables (and other settings of the
%joe4
%parents).  Pearl \cite[Section 9.2.2]{pearl:2k} proves some results under the
%assumption that a variable $Y$ is {\emph exogenous  relative to a
parents).
%joe13*
We cannot always assume this, but in many realistic circumstances, we can.

%joe13*: cut. As Elias pointed out, the situation here is more subtle.
\commentout{
Pearl \citeyear[Section 9.2.2]{pearl:2k} gives %some
bounds under the 
assumption that a variable $Y$ is \emph{exogenous  relative to a
  parent $X$}. This means that the cpt for $Y$ given
different settings of $X$ are independent, conditional on $Y$.  Pearl
rightly takes this to be a nontrivial assumption, which we would not
always expect to hold.
%joe4: shortened
%By way of contrast, using
%Pearl's terminology, we are assuming here that $Y$ is (unconditionally)
%exogenous relative to its parents.%
%\footnote{To be clear, by this we mean not that $Y$ is exogenous
%relative to each parent considered in isolation, but that the cpt for
%different 
%settings of \emph{all} the parents of $Y$ simultaneously are independent.}
%This seems to us a more reasonable assumption.}
We believe that our independence assumption is much more reasonable
\new{whenever exogenous parents of $Y$ can be measured. } 
}

%we show how to convert a
%CBN $M$ to a causal model $M'$ in a canonical way; the resulting
%causal agrees with $M$ on the semantics of all formulas in $\Lcaus^+$.
%The construction, which is in the spirit of that given by Balke and
%Pearl \citeyear{BP94}, also shows how the independence assumptions we made
%can be removed.  

%So, in going from a CBN to a  causal
%model, we ``derandomize'' the contexts.  This viewpoint will be useful
%when we discuss conditioning in CBNs in the next section.

\commentout{
But for a CBN $M$, taking the probability 
of $\psi \land \psi'$ in $(M,\vec{u})$ to be
the product of the probability of $\psi$ in $(M,\vec{u})$ and the
probability of $\psi'$ in $(M,\vec{u})$ still does not seem quite right.
Consider a simple CBN with one exogenous variable $U$ and three endogenous
variables $X$, $Y$, and $Z$, all binary, where $U$ is the parent of $X$
and $Y$, which are the parents of $Z$, as shown in the following figure:

\begin{verbatim}
    U
   / \
  X   Y
   \ /   
    Z
\end{verbatim}

Again, we want to compute the probability of the formula $\phi =
Y=0 \land [Y \gets 1](Z=0)$.

probability of $\phi$ in the context $U=0$:
$\Pr(Y=0 \mid U = 0) = \Pr(X= 0 \mid U=0) = 2/3$,
$\Pr(Z=0 \mid X= 0, Y= 1) = 2/3$, $\Pr(Z=0 \mid X= 1, Y=1) = 1/3$.
%$\Pr(Y=0 \mid U = 0) = \Pr(X= 0 \mid U=0) = 1/2$,
%$\Pr(Z=0 \mid X= i, Y= j) = 1/2$ for all $i,j \in\{0,1\}$.
It is now easy to see that in the context $U=0$, $Y=0$ has probability
$2/3$, and $[Y \gets 1](Z=0)$ has probability $2/3 \times 2/3 + 1/3
\times 1/3 = 5/9$.  We do not want to take the probability of $\phi$
to be $2/3 \times 5/9$.  
%As we argue below, it turns out to be equivalent to condition on the
%relevant \enph{partial state}, where only
%some of the values of variables are set, and the calculation is
%easier, so we do this first.  In this
%case, we care about the  partial state $v_0$, where $X=0$; and
%$v_1$, where $X=1$.  Conditional on $U=0$, $v_0$ has probability
%$2/3$ and $v_1$ 
%has probability $1/3$.  In $v_0$, if we intervene to set $Y=1$, then
%the probability of $[Y \gets 1](Z=0)$ is $2/3$; in $v_1$ it is $1/3$.
%Applying independence separately in $v_0$ and $v_1$, we take the
%probability of $\phi$ to be $2/3 x 2/3 + 1/3 \times 1/3 = 5/9$.
%The key point is that
The value of $Z$ depends on the value of $X$, which is determined by
the state.   So, we want to assume that $Y=0$ and $[Y \gets 1](Z=0)$
are independent, conditional on the state.

There are four states with $X=0$ (which differ in the values for
$Y$ and $Z$) and four states with $X=1$.  $Y=0$ in two of the
states with $X=0$, and they have total probability $2/3 \times 2/3$, and $[Y
  \gets 1](X=0)$ has probability $2/3$ in each of these.  
The way we compute the latter probability in a state $w$ in $M$ is to set
$Y=1$ to get the model $M_{Y \gets 1}$, and then 
apply the cpt for $Z$ to get the probability that $Z=0$ in the
resulting states in $M_{Y \gets 1}$.  
Since we are considering states where $X=0$, we fix $X=0$ when
computing the probability that 
$Z=0$.  Thus, the total probability of $[Y \gets 1](X=0)$ in states
where $X=0$ is $2/3 \times 2/3 \times 2/3 = 8/27$.
A similar computation in states where $X=1$ shows that the probability
of $[Y \gets 1](X=0)$ in these states is $1/3 \times 2/3 \times 1/3 =
2/27$.  Thus, we take the probability of $[Y \gets 1](X=0)$ in
$\vec{u}$ to be $10/27$.

We view it as more reasonable to condition on the state here
because we take the mechanisms that define interventions to work
independently once we have fixed all other factors that determine how
the intervention works; in this case, the value of $Z$ affects how the
intervention works.  

{\red This is the procedure we use to compute the probability both for the
conjunction of a simple formula and a causal formula involving
an intervention and for 
the conjunction of two formulas involving interventions in a state $w$
as long as 
(a) none of the formulas being intervened on in the intervention
formula(s) is true   in $w$ and (b) all the interventions involved are
distinct (up to reordering).  To understand the first point, 
suppose that we are interested in an intervention
formula such as $[Y_1 \gets y_1, Y_2 \gets y_2]\phi$, where $Y_1 =
y_1$ in $w$ but $Y_2 \ne y_2$ in $w$.    In that case, there is no
need to intervene on $Y_1$; it already has the right value in $w$.  We
thus work with the formula $[Y_2 \gets y_2]\phi$ instead.  We
implicitly are taking the formulas $[Y_1 \gets y_1, Y_2 \gets
  y_2]\phi$ and $[Y_2 \gets y_2]\phi$ to be equivalent in $w$.

With regard to the second point, while our assumption of independence
seems reasonable if the interventions involved are distinct, it does
not seem reasonable if the interventions are identical.
%We evaluate the probability of a conjunction of interventions
%similarly, applying independence conditional on the state.  
%For example, suppose we add two more binary variables $W$ and $W'$ to the
%model, where $U$ is the parent of $W$ and $X$ and $W$ are the parents
%of $W'$, and now we want to evaluate the probability of $\phi' = Y=0
%\land [Y \gets 1](X=0) \land [W \gets 0](W'=1)$. For each state $w$ in
%context $\vec{u}$ satisfying $Y=0$, we evaluate the probability of 
%$[Y \gets 1](X=0)$ and the probability of $[W \gets 0](W'=1)$
%(separately), multiply the latter two probabilities by the probability
%of $w$ conditional on $U=0$, and sum over all states, to get the
%probability of $\phi'$ in context $U=0$.  We can similarly evaluate
%the probability of $\phi'$ in context $U=1$.  The overall probability
%of $\phi'$ in $M$ is the sum, taken over all contexts, of the
%probability of the context tines the probability $\phi$' conditional
%on that context.
%
%although there is an additional subtlety that is discussed below.
%While the procedure above applies if the interventions involved are
%distinct,
%
Since the same intervention is involved in a formula such as 
$[\vec{Y} \gets \vec{y}]\phi \land [\vec{Y} \gets
\vec{y}]\phi'$; are identical, the outcomes might be correlated.  We 
identify this formula with $[\vec{Y} \gets \vec{y}](\phi \land
\phi')$; this allows us to deal with correlation.  Note that
%Nor does the procedure apply if the intervention holds in
%the state $w'$ being considered; if $\vec{Y} = \vec{y}$ in
%$w'$, then the probability of $[\vec{Y} \gets \vec{y}]\phi \land [\vec{Y} \gets
%\vec{y}]\phi'$ conditional on $w'$ is either 0 or 1, depending on
%whether $\phi'$ holds in $w'$.  Intuitively, if $\vec{Y} = \vec{y}$ in
%$w'$, we do not need to intervene; we already know the outcome in $w'$. 
the order that interventions are performed in does not effect the
outcome (whichever order the interventions are performed in, we
consider the model where the equations for the variables being
intervened on are replaced by the value that these variables get in the
intervention), so we identify formulas such as $[Y_1 \gets y_1, Y_2 \gets
  y_2]\phi$ and $[Y_2 \gets y_2, Y_1 \gets y_1]\phi$; for simplicity,
we assume that the variables in an intervention are always listed in
some canonical order.}  

As this discussion should make clear, in giving semantics to these
formulas, we made a number of choices.   While we think these choices
are reasonable, they are not obviously ``right'' (whatever ``right''
means); {\red the one perhaps most open to dispute is the assumption of
independence}.  We now provide the technical details.

\subsection{Technical details}\label{sec:technical}

In order to assign a probability to a
formula in $\Lcaus$, the first step is to convert it to an equivalent
formula in a canonical form.  This canonical form looks like a DNF
formula.  Specifically, it has the form $\phi_1 \lor \cdots \lor
\phi_k$, where the $\phi_j$s are mutually exclusive and 
each $\phi_j$ is a conjunction of the form $\phi_{j0}
\land \cdots \land \phi_{jm}$, where $\phi_{j0}$ is a simple formula
and for $1 \le k \le m$, $\phi_{jk}$ is an intervention formula of the
form $[\vec{Y} \gets \vec{y}]\psi$, and the interventions are all
distinct.  This conversion just involves standard propositional
reasoning and two further assumptions.  The
first was already mentioned above: that $[\vec{Y} \gets \vec{y}]\phi
\land [\vec{Y} \gets \vec{y}]\phi'$ is equivalent to 
$[\vec{Y} \gets \vec{y}](\phi \land \phi')$.  The second is that
$\neg [\vec{Y} \gets \vec{y}]\phi$ is equivalent to
$[\vec{Y} \gets \vec{y}]\neg \phi$. These two formulas are easily seen
to have the same probability, according to the semantics sketched
above.  The assumption is that we can replace one by the other in all
contexts, which  seems (at least to us) quite reasonable.

Ignore for now the requirement that the disjuncts be mutually exclusive, 
that all interventions be distinct, and that there be no leading
negations in intervention formulas.  Using standard propositional
reasoning, we can transform a formula $\phi$ to an equivalent formula
in DNF, where the literals are either simple formulas or
intervention formulas.  Of course, the disjuncts may not be mutually
exclusive.  Again, using straightforward propositional reasoning, we
can further the formula to a DNF where the disjuncts are mutually
exclusive.  Rather than writing out the tedious details, we give an
example.  Consider a formula of the form $(\phi_1 \land \phi_2) \lor
(\phi_3 \land \phi_4)$.  This is propositionally equivalent to
$$\begin{array}{l}
(\phi_1 \land \phi_2 \land \phi_3 \land \phi_4) \lor
(\phi_1 \land \phi_2 \land \neg \phi_3 \land \phi_4) \lor
(\phi_1 \land \phi_2 \land \phi_3 \land \neg \phi_4) \lor
(\phi_1 \land \phi_2 \land \neg \phi_3 \land \neg \phi_4)\\ \lor
(\neg \phi_1 \land \phi_2 \land \phi_3 \land \phi_4) \lor
(\phi_1 \land \neg \phi_2 \land \phi_3 \land \phi_4) \lor
  (\neg \phi_1 \land \neg \phi_2 \land \phi_3 \land \phi_4);
  \end{array}$$
moreover, this formula is in the right form.
We can now  apply the two equivalences
mentioned above to remove leading negations from intervention
formulas and to ensure that, in each disjunct, all interventions are
distinct.  These transformations maintain the fact that the disjuncts
are mutually exclusive.

The rest of the argument is straightforward.  Since the disjuncts in
$\phi$ are mutually exclusive,
the probability of $\phi$ is the  sum of the probabilities of the
disjuncts.  To compute the probability of a disjunct in a state $w$,
we proceed as outlined above.  
We {\red first remove all interventions in intervention
  formulas that are true in $w$.  This may result in interventions
  that were distinct no longer being distinct; in that case, we
  apply the equivalence above to the appropriate intervention
  formulas, to maintain the assumption that all interventions are
  distinct.}  We then compute the probability of each conjunct {\red 
  in $w$, and then treat
the intervention formulas as independent, as discussed above.    This
gives us the probability of each disjunct in $\phi$ in state $w$.
Since the disjuncts are mutually exclusive, the probability of $\phi$
in $w$ is the sum of the probability of the disjuncts of $\phi$ in $w$.  We
then multiply the probability of 
the state $w$ conditional on the context $\vec{u}$ by the probability
of $\phi$ in (i.e., conditional  
on) $w$, and sum over all states $w$, to get the probability of 
$\phi$  in $\vec{u}$.   Finally, we multiply the probability of
$\phi$ in $\vec{u}$ by the probability of $\vec{u}$ and sum over all
contexts $\vec{u}$ to get the probability of $\phi$ in $M$.}
We leave the details to the reader.
}

%{\green \subsection{The reasonableness of the assumption that cpts
%    are independent}\label{sec:reasonable} 

\section{Converting a CBN to a {\green (Probabilistic)} Causal Model}\label{sec:convert}

{Our semantics for formulas in CBNs reduced to considering their
  semantics in i-compatible causal models.  It would be useful to show
  %joe5
  explicitly
  that such i-compatible causal models exist and how to construct
%joe4
  %  them.  That is the goal of this section.
%joe5
  %  them explicitly.  That is the goal of this section.
  them.  That is the goal of this section.
  %joe5: cut
%    Essentially, an i-compatible causal model can  be viewd as 
%``pulling the probability out  of the edges'' and moving it to the
%exogenous variables.
Balke and Pearl \citeyear{BP94} sketched how this could be done.
%joe4: cut; this is a different issue
%Starting with a CBN $M$, they 
%actually considered a family of  compatible causal models,
%since they did not make the independence assumption that we have
%made.  They then took the probability of a formula in a CBN to be the
%range of probabilities determined by these compatible causal models.
%{\green With our independence assumption (i.e., by considering
%  i-compatible models), the probability of a formula in a CBN is
%\emph{identifiable}, that is, it is determined.  }
%joe4
%joe5: typo
%We larggely follow the ideas of their construction.
We largely follow and formalize their construction.

%\subsection{The Construction}

%joe1
%To make our construction precise, we must first need the following definition.
%To make our construction precise, we need the following definition.
%\begin{definition} Two models $M$ and $M'$ (each of which can be
%  either a probabilistic  causal model or a CBN) that have the same set $\V$
%  of endogenous variables are \emph{equivalent} if they assign the
%  same probability to all formulas in $\Lcaus(\S)$.
%\end{definition}

%We start by considering the first approach.  This can essentially be
%viewed as a formalization of Balke and Pearl's \citeyear{BP94}
%approach for evaluating probabilistic queries.  We give one instance of
%this approach.
Starting with a CBN $M$, we want to construct an i-compatible 
probabilistic causal model $(M',\Pr')$,} { where $M'$
  has the same type as $M$.
%{\green (It will be useful for our
%later discussion to make the probability $\Pr'$ explicit in the notation.)}
%The idea is to replace each cpt in $M$ by a new
%exogenous variable to create an equivalent CBN.  After doing this for
%all cpts, we will have a functional causal model.
\commentout{
{\red  We do this
  inductively, assuming that {\blue after $m$ steps in the induction,
    we have constructed a model where we have replaced the cpts for
    $m$ variables by equations, and for each of these variables $Y$ we
    have added an exogenous variable $U_Y$. We replace the cpts for
    variables     in topological order, so we do not replace the cpt 
   of a variable before  having replaced the cpts of all its parents.
%We further assume inductively that if we have replaced the
%  cpts of variables $Y_1, \ldots, Y_m$ by equations, then we have
%  added exogenous variables
  %  $U_{Y_1}, \ldots, U_{Y_m}$.
  Moreover, if $\U_0$ is the set of
  exogenous variables in the CBN $M_0$ that we started with, then we
  assume for the model $M_m$ that we have after step $m$, where we
  have replaced the cpts for the variables $Y_1, 
  \ldots, Y_m$, the  probability $\Pr_m$ defined on contexts in $M_m$ in
  such that $\U_0, U_{Y_1}, 
  \ldots, U_{Y_m}$ are mutually independent,
  %conditional on the   variables in $\U$,
  and if $\Pr_0$ is the probability on contexts in
  $M_0$, then $\Pr_m(\U) = \Pr_0(\U)$.}

  {\red For the inductive step, to construct $M_{m+1}$ given $M_m$, let}
}
}
{\red  To do this, for each endogenous variable $Y$ in $M$ with
  parents $X_1, \ldots, X_n$,
%$Y$ be one of the endogenous variables in $M_m$ whose cpt we have not
%  replaced (if there are any---if not, we are done and $M_n$ is the
%  desired causal model) such that $Y$'s parents $X_1, \ldots, X_n$ are
%  among the variables $Y_1, \ldots,
%Y_m$ whose cpts have already been replaced by equations.
%(Since we replace cpts by equations for variables in topological order, there
%will always be 
%such a variable $Y$ unless all cpts have been replaced.)}
%
we add a new exogenous variable $U_Y$;
$\R(\U_Y)$ consists of all functions from $\R(X_1) \times \cdots
\times \R(X_n)$ to $\R(Y)$. Balke and Pearl \citeyear{BP94} call such
an exogenous variable a \emph{response function}.  (Response
functions, in turn, are closely related to the \emph{potential response
%joe4
%variables} introduced by Rubin \citeyear{Rubin74}.  We take $U_Y$ to be
variables} introduced by Rubin \citeyear{Rubin74}.)  We take $U_Y$ to be
a parent of $Y$ (in addition to $X_1, \ldots, X_n$).  We replace the
cpt for $Y$ be the following equation for $Y$;
$F_Y(x_1, \ldots, x_n, f) = f(x_1, \ldots, x_n)$, where $f$ is the
value of $U_Y$.
%(If we want to view the resulting model as a CBN, we
%view this equation as defining a cpt for which $\Pr(f(x_1, \ldots, x_n) \mid X
%= x_1, \ldots, X=x_n, U=f) = 1$.)
Since $f$ is a function from
$\R(X_1) \times \cdots \times \R(X_n)$ to $\R(Y)$, this indeed gives a
value of $Y$, as desired.
{\blue Let $Y_1, \ldots, Y_m$ be the endogenous variables in $M$.
  We define the probability $\Pr'$ on
$\R(\U) \times \R(U_{Y_1}) \times \cdots \times \R(U_{Y_m})$ by taking
$\Pr'(\vec{u},f_1, \ldots
  f_m) = \Pr(\vec{u}) \times \Pi_{i = 1, \ldots, m} \Pr_{Y_i}(f_i)$, where $\Pr_{Y_i}$
  reproduces the probability of the cpt for $Y_i$.  Specifically, 
for an endogenous variable $Y$ with parents $X_1, \ldots, X_n$,
  %If the set of exogenous variables in $M$ is $\U$,
%We do not want to take $\U$
%to be independent of $U_Y$,
%but, as we said, we do want to take $U_Y$ to be independent of
%$U_{Y_1}, \ldots, U_{Y_m}$ conditional on $\U$.  (The induction
%assumption already tells us that $U_{Y_1}, \ldots, U_{Y_m}$ are
%independent conditional on $\U$.)  
%If $\Pr(\vec{u}) = 0$, then 
%$\Pr'(\vec{u}, f_1, \ldots, f_m,f ) = 0$; otherwise,
%$\Pr'(\vec{u}, f_1, \ldots, f_m,f ) = 
%\Pr'(f_1, \ldots, f_m,f \mid \vec{u}) \Pr(\vec{u})$, where
%$\Pr'(f_1, \ldots, f_m,f \mid \vec{u}) =
%\Pr(f_1 \mid \vec{u}) \times \cdots \times \Pr(f_m \mid \vec{u})
$\Pr_{Y}(f) =  \Pi_{\vec{x} \in \R(X_1) \times \cdots \times
  \R(X_n)} \Pr(Y= f(x_1, \ldots, x_n) \mid X_1 = x_1, \ldots, X_n =
  x_n)$.}
  %(Note that the last factor is independent of $\vec{u}$ if no
%  exogenous variable is a parent of $Y$.)} 
  %Essentially,
  This makes interventions for different
  %  interventions on
  settings of $X_1, \ldots, X_n$ independent,
%  conditional on the value of $\U$,
  which is essentially what we assumed in the previous
section when defining the probability of formulas in $\Lcaus$ in
$M_0$, in addition to making interventions on different variables
independent and independent of the context in $M$. 
In any case, it is easy to see that this gives a well-defined probability
on $\R(\U) \times \R(U_{Y_1} \times \R(Y_m)$, the contexts in $M'$. 
Moreover, $M'$ is clearly a causal model { with the same type
  as $M$} { that is i-compatible with $M$.}

\commentout{
As we now show, $M$ and $(M',\Pr')$ are equivalent.
%The fact that we treat
%the values of $f$ for different choices of $\vec{x}$ as independent is
%what allows us to reproduce the independence assumptions we made in
%defining the probability of conjunctions of intervention formulas in
%the previous section.

\begin{theorem}\label{thm:main} $M$ and $(M',\Pr')$ are equivalent;
that is, $\Pr_{M}^C(\phi) = \Pr_{(M',\Pr')}^F(\phi)$ for all
formulas $\phi \in \Lcaus$.
\end{theorem}

\prf It is straightforward to show that there is a bijection between
%complete combinations of conditional events
ccces
in $M$ and contexts $(\vec{u},
f_1, \ldots, f_m)$ in $M'$.  The conditional events involving
exogenous variables determine $\vec{u}$, and the conditional events
involving $Y_i$ determine $f_i$ (and vice versa).  Moreover, our
assumptions of independence of the components of a
%complete combination
ccce
and the independence implicit in the definition of $\Pr'$
ensure that the probability of a
%complete combination
ccce
in $M$ is equal
to the probability of the corresponding context in $(M',\Pr')$.  The
equivalence of $M$ and $(M',\Pr')$ follows immediately from this observation.
\eprf
}
}

\commentout{
  {\red First we view both $M$ and $M'$ as CBNs.
  To see that they
  assign the same probability to simple formulas, first observe that
  the probability of a context $\vec{u} \in \R(\U)$ in $M$ is the same
  as the marginal probability of $\vec{u}$ in $M'$ (and both are the
  same as the probability of $\vec{u}$ in $M_0$).  We next want to show
  that for each state $w$, the probability of $w$ conditional on a
  context $\vec{u}$ is the same in $M$ and $M'$. To do this, it
  suffices to show that the marginal probability that we have defined on
  $\R(\U) \times \R(U_Y)$ in $M'$ simulates the cpt for $Y$ in $M$ (since
  all other cpts are the same in $M$ and $M'$).  This follows from the
  observation that, for all fixed $\vec{u}, x_1, \ldots, x_n, y$, we have that 
  $\Pr(Y=y \mid X_1 = x_1, \ldots, X=x_n) = \sum_f 
  \Pr'(\{(\vec{u},f): f(x_1,\ldots, x_n) = y\})$.  (Note that fixing
  $\vec{u}$ here means that if any of $X_1, \ldots, X_n$ are
  exogenous variables, then their values are determined by $\vec{u}$;
  e.g., if $X_j = U_3$, then $x_j = u_3.)$  Since all states have the
  same probability given $\vec{u}$, it follows that all simple
  formulas have the same probability in $M$ and $M'$.

  A similar argument shows that all states in $M$ and $M'$ have the same
  probability given an intervention $\vec{Z} \gets \vec{z}$.  If none
  of the variables in $Z$ is $Y$, this is almost immediate, since $M$
  and $M'$ work the same way for all variables other than $Y$.  
  In the case of $Y$, this again follows from the fact
  the marginal probability that we have defined on 
  $\R(\U) \times \R(U_Y)$ in $M'$ simulates the cpt for $Y$ in $M$.

The probability of a formula $\phi$ in a CBN is the sum of the probabilities
of the
%complete combinations of conditional events
ccce
that make $\phi$ true.

  The fact that $M$ and $M'$ give the same probability to all formulas now
  follows from the fact that we use the same procedure to compute the
  probability of a formula $\phi$ in both cases (since we are viewing
  $M'$ as a CBN) and the procedure depends only in the probability of
  simple formulas and single intervention formulas.

For the last step, it suffices to show that if $M'$ is in fact a
causal model (so that we have replaced all cpts using the procedure
defined above, by adding additional exogenous variables), then $M'$
gives the same probability to a formula $\phi$ when viewed as a CBN
and as a causal model.  Since we have identified an equation
$Y=F(\vec{x})$ with a cpt with the entry $\Pr(Y=y) \mid \vec{X}
= \vec{x}) = 1$, it is immediate that $M'$ gives the same probability
to all states, whether viewed as a CBN or a causal model. This remains
true even after an intervention.  The fact that we treat intervention
formulas as independent (conditional on a context in $M_0$) in the causal
model follows from the fact that all the exogenous $U_{Y_i}$ are
mutually independent conditional on $\U$, and that the definition of
$\Pr'$ guarantees that for the variable $Y$ that was added last,
different interventions on $Y$ are independent of each other (and of
interventions on other variables).  [[SHOULD I SAY MORE HERE.]]

{\red
  Of course, if we repeat the process above for all the endogenous
variables in a CBN $M$, we end up with a functional causal model that,
by Theorem~\ref{thm:main}, is
equivalent to $M$ (since each step in the procedure maintains
equivalence), and the probability is the same at the final step
whether we view $M'$ as a CBN or a causal model.}
}

\begin{example}
Consider the CBN $M^*$ of the form $U
\rightarrow X \rightarrow Y$, where all variables are binary.
%Suppose that $\Pr(U=0) = 1$, $\Pr(X = 0 \mid U=0) = 1/3$,
%$\Pr(X=0 \mid U=0) = \Pr(X=0 \mid U=0) = 2/3$,  and $\Pr(Y=0 \mid
Suppose that $\Pr(U=0) = 1$, $\Pr(X = 0 \mid U=0) =
\Pr(X=0 \mid U=1) = \Pr(Y=0 \mid
X=0) = \Pr(Y=0 \mid X=1) = 1/2$.  In the construction,
there are four values of $U_X$, which we denote
  $f_{00}$, $f_{01}$, $f_{10}$, and $f_{11}$, respectively, where
  $f_{ij}: \R(U) \rightarrow \R(X)$, $f_{ij}(0) = i$, and $f_{ij}(1) =
  j$.  Similarly, 
  there are four values of $U_Y$, which we denote
  $g_{00}$, $g_{01}$, $g_{10}$, and $g_{11}$, where  $g_{ij}: \R(X)
  \rightarrow \R(Y)$, and is defined analogously to $f_{ij}$.
In the resulting functional model, %call it $M^{\dagger}$,  
%we have $\Pr'(U=0, U_X = f_{ij}, U_Y = g_{ij}) =
  %  (\frac{2}{3})^{2-i-j}(\frac{1}{3})^{i+j}\frac{1}{16}$.
%\frac{1}{16}$.
each of the 16 contexts with $U=0$ has probability $1/16$.
Of these 16 contexts,
  the 8 contexts of the form $(U=0, U_X = f_{0i}, U_Y=g_{1j}))$ or
  $(U=0, U_X = f_{1i}, U_Y=g_{j1})$ result in a world where $Y=1$; the
  remaining 8 result in a world where $Y=0$.  Thus, if we condition on
  $Y=1$, we
  get that  get that the 8 that result in $Y=1$ each has
  probability $1/8$. 
%  U_Y = {g_{i1})$ each have probability $1/6$, while the four contexts
%    of the form $(U=0, U_X = f_{1i},   U_Y = {g_{j1})$ each have
  %      probability $1/12$.
  We can now compute, for example, that
      after conditioning we have $\Pr([X\gets 1](Y=0) \mid Y=1) =
      1/4$: conditional on $Y=1$, the event $[X\gets 1](Y=0)$ holds
      only in the two contexts of the form $(U=0, U_X = f_{0i}, U_Y = g_{10})$.

If we apply the semantics of Section~\ref{sec:semantics} directly to
$M^*$, it is easy to see that $\Pr(Y=1) = 1/2$.  While
$\Pr([X \gets 1](Y=0))  = 1/2$, it is not the case that $\Pr(Y=1
\land [X \gets 1](Y=0)) = 1/4$.  To compute $\Pr(Y=1
\land [X \gets 1](Y=0))$, we consider the world where $U=0 \land X=0
\land Y=1$ and the world where $U=0 \land X=1 \land Y=1$.  Each of
these worlds has unconditional probability $1/4$.  In the former, if we
intervene and set $X=1$, so
the probability of $Y=0 \land [X \gets 1](Y=0)$ (conditional on this
world) is $1/2$.  But in the world where $U=0 \land X=1 \land Y=1$ holds,
since $X=1$ already holds, there is no need to intervene.  Since $Y=0$
does not hold in this world, the probability of 
$Y=1 \land [X \gets 1](Y=0)$Q conditional on this
world is 0.  Thus, the probabliity of $(Y=0 \land [X \gets
  1](Y=0)$ conditional on $U=0$ is $1/8$, and the probability of 
$([X \gets
  1](Y=0)$ conditional on $Y=0$ (and $U=0$) is $1/4$, matching the
probability given by the causal model (as we would expect).

\end{example}
    }

%As we observed, our construction for converting a CBN to a causal
%model builds in the independence assumption that we made in
{ We can easily modify this construction to get a family of
%joe4
  %  causal models compatible with $M$.}
    causal models compatible with $M$, by loosening the requirements
%joe13
    %    on $\Pr'$.  While we do want its marginal on $\U$ to agree with
  on $\Pr'$.  While we do want the marginal of $\Pr'$ on $\U$ to agree with
    the marginal of $\Pr$ on $\U$, and we want it to reproduce the
    probability of the cpt for each variable $Y_i$ (as defined above),
    there are no further independence requirements.
%{\blue determining the probability of
%  %  complete combination of conditional events} in CBNs.
%    ccces} in CBNs.
%Suppose we wanted to drop this
%assumption.  We still want the causal model to
%{\purple have the same type as $M$ and to}
%capture the cpts
%correctly.  Thus, we want to define $\Pr'$ so that the marginal
%probability $\Pr'(Y= f(x_1, \ldots, x_n) 
%\mid X_1 = x_1, \ldots, X_n = x_n) = \Pr(Y= f(x_1, \ldots, x_n) \mid
%X_1 = x_1, \ldots, X_n = x_n)$, where the latter probability comes from
%the cpt for $Y$ in $M$.  Beyond that, there should be no constraints
%on $\Pr'$.  With this change,
If we do that,
we get the bounds computed by
Balke and Pearl \citeyear{BP94}.
The following example illustrates the impact of dropping the
  independence assumptions.
  
  %sg4
  %\begin{example} Consider the CBN $M^*$ from Example~\ref{xam:M*} again
  \begin{example} Consider the CBN $M^*$ from Example~\ref{xam:M*} again.
%joe4: line shaving
  %  Suppose that $\Pr(U=0) = a = 1$, $\Pr(X=0 \mid U=0) = b = 1/2$, and 
%  $\Pr(Y=0 \mid X= 0) = d = 1/2$.  The independence assumption
  Using the notation from that example, suppose that $a=1$ and $b=d=1/2$.
Independence 
  guarantees that the set of
  %  complete combinations
ccces
%joe4
%that include
that includes
%joe4
%$U=0$, $X=0 \mid U=0$, and $Y=0 \mid X=0$ has probability $1/4$.
$U=0$, $X=0 \mid U=0$, and $Y=0 \mid X=0$ has probability $abd = 1/4$.
But now consider a causal  model $(M^{**}, \Pr^{**})$ {
%joe4
  %  compatible with $M^*$}, where the
    compatible with $M^*$} where the
  contexts are the same as in our construction, but 
the probability $\Pr^{**}$ does not build in the independence
assumptions of our construction.   Recall that contexts in $M^{**}$
%joe4
%have the form $(u,f_X,f_Y)$.  Since we want $(M^{**}, \Pr^{**}$ to
%be compatible with $M^*$ we must have 
have the form $(u,f_X,f_Y)$.  Since we want $(M^{**}, \Pr^{**})$ to
be compatible with $M^*$, we must have 
$\Pr^{**}(\{(u,f_X,f_U,f_Y): u=0\}) = 1$, $\Pr^{**}(\{(u,f_X,f_Y):
f_X(0) = 0\}) = 1/2$, and  $\Pr^{**}(\{(u,f_X,f_Y):
f_Y(0) = 0\}) = 1/2$, so that $\Pr^{**}$ agrees with the three cpts.
But this still leaves a lot of flexibility.  For example,
we might have $\Pr^{**}(\{(u,f_X,f_Y): f_X(0) = f_Y(0)  = 0\} = 
Pr^{**}(\{(u,f_X,f_Y): f_X(1) = f_Y(1)  = 1\} = 1/2$ (so that
$\Pr^{**}(\{(u,f_X,f_Y): f_X(0) = 0, f_Y(1)  = 1\}) = 
\Pr^{**}(\{(u,f_X,f_Y): f_X(0) = 1, f_Y(1)  = 0\}) = 0$).
%joe4
%As computed in Example~\ref{xam:M*continued}, we have $\Pr_{M^*}(X=0
As shown in Example~\ref{xam:M*continued}, $\Pr_{M^*}(X=0
\land Y = 0 \land [X \gets 1](Y =1)) = 1/4$.  However, it is easy to
check that $\Pr_{M^{**}}(X=0
  \land Y = 0 \land [X \gets 1](Y =1)) = 1/2$.  (Tian and Pearl
%joe4
%  \citeyear{TP00} give general bounds on the range of probabilities
  \citeyear{TP00} give bounds on the range of probabilities
  for this formula, which is called the \emph{probability of
%joe4
  %  necessity}; see also \cite[Section 9.2]{pearl:2k}.)
    necessity}; see also Section~\ref{sec:computing} and
  \cite[Section 9.2]{pearl:2k}.) 
\end{example}
  
  {
%\section{Proof of Theorem~\ref{ccce}}

{
\section{Computing counterfactual probabilities}\label{sec:computing}

%joe3: shortened a bit
%In this section, we analyze popular notions of
%counterfactual probabilities, capturing whether an event $x$ is a
%necessary or sufficient cause of an outcome.  \cite{BP94} introduced
%the formal notions of counterfactual probabilities in 1994.
%These notions are extremely popular in several domains such as epidemiology, 
%finance, and economics, to explain the effects of a variable on the outcome.
%Recently, counterfactual explanations have evolved as one of the most 
%popular methods to explain the output of machine learning models.
In this section, we analyze \emph{counterfactual probabilities},
%joe4: shortened this material
%introduced by Balke and Pearl \citeyear{BP94}, with a particular
%emphasis on the \emph{probability of necessity} and the
%\emph{probability of sufficiency} (defined formally below).
introduced by Balke and Pearl \citeyear{BP94}.
%joe4*: We still need these references
Counterfactual probabilities have been widely used in several domains,
including psychology~\cite{hoerl2011understanding}, 
epidemiology~\cite{greenland1999epidemiology}, and 
political science~\cite{grynaviski2013contrasts}, to explain the effects
on the outcome.  More recently, they have proved useful in machine
learning to explain the output of ML models~\cite{pmlr-v177-beckers22a}.

%joe3*: shortened and tried to highlight the main points
%As a recp, the total number of fccce's
%can be doubly exponential in the number of variables in the causal graph.
%For example, consider a variable $X$ with $t$ parents $X_1,\ldots, X_t$.
%This means that there are $2^t$ possible settings (value combinations) 
%of the parents of $X$. For each such setting of the parents of $X$, we have 
%$X=0$ or $1$. Therefore, we have a total of $2^{2^t}$ different fccce's
%for this combination. We show
%that we can estimate the probabilities of necessity and sufficiency 
%from observational data by estimating at most $2^n$ conditional probabilities.
%joe4*: As near as I can tell, we use \F(S) only to define \phi_S,
%which in turn is used only in the calculation of \psi_c (what was
%equation (3) in the paper, but has now been cut with the intention of
%moving it ot the appendix.  This feels like overkill.  I think that
%we can get away without defining \mathcal{F}(S), and just define
%\phi_S   (without defining \mathcal{S}) when we use it.
\commentout{
Recall that the total number of fccce's
can be doubly exponential in the number of variables in the causal
graph (because for a binary variable $X$ with $t$ parents,
for each of the $2^t$ setting of the parents of $X$, the
conditional event  where either $X=0$ or $X=1$ must be included in
each fccce).  We show that, given our independence assumptions, to
compute many counterfactual probabilities of 
interest, we need to compute the conditional probabilities of far
fewer (typically exponentially many) events, and these conditional
probabilities do not involve counterfactuals, so can be computed from
observational data.  Thus, our results and assumptions have
significant practical implications.

%joe3*: again, shortened.  I worry slightly that the S here conflicts
%with the \S that we use for signature, although they use different
%fonts.  Should we use a different letter?  (I don't feel strongly
%about this.)
%An fccce contains one conjunct for each endogenous
% variable $X$ and each setting of the parents of $X$.
%Let $S$ denote a subset of such conjuncts
  %joe2
% %and $\mathcal{F}(S)$ denote the complete set of fccce's which contain
%% the conjuncts $S$.
% and $\mathcal{F}(S)$ denote the complete set of fccce's that contain
% all the conjuncts in $S$.
%Note that setting $S=\emptyset$ is equivalent to all possible fccce's.
%joe3* 
For a $S$ of conditional events, let
 $\mathcal{F}(S)$ denote the the set of fccces that include
all the events in $S$.
%: I found it hard to parse the next sentence.  S is a set of
Note that $\mathcal{F}(\emptyset)$ consists of all fccces.

\begin{example}
%joe2
  %  Consider the following causal graph and
%joe3*: here and below we need to start with a cbn, not a causal model
%  Consider the a causal model $M$ with the following causal graph: 
  Consider the a CBN $M$ characterized by the following dag:
  %joe2: you took s to be a set of conjuncts.  We could identify the set
%with the conjunction of formulas in the set. That would be fine with
%me, but if we want t odo that, we should say so explicitly.  I added
%parens in any case.
  %  $S=  U=0 \land X=0\mid U=0$.
  %joe4: put in centering environment
  \begin{center}
  \begin{tikzpicture}
  \tikzstyle{var} = [circle, draw, thick]

  % Define nodes with the style
  \node[var] (U) at (0,0) {U};
  \node[var] (X) at (2,0) {X};
  \node[var] (Y) at (4,0) {Y};

 \draw[->, thick] (U) -- (X);
  \draw[->, thick] (X) -- (Y);
\end{tikzpicture}
  %joe4
  \end{center}

%joe2
%In this case the total number of fccce's that contains $S$ is
Let    $S=  \{U=0, (X=0\mid U=0)\}$.
Then $|\mathcal{F}(S)| = 4$, since to get an fccce that
contains $S$, we must add to $S$
\begin{itemize}
\item one of $Y=0\mid X=0$ and $Y=1\mid X=0$; and
 \item one of $Y=0\mid X=1$ and $Y=1\mid X=1$.
\end{itemize}
\end{example}

%joe2
%First, we show that the total probability of fccce's that contain a
%specific set of conjunctions
%joe3
%We first show that the total probability of the set of fccces that contain a
%specific set of conjuncts is just the probability of (the formula 
%corresponding to) that set.
Generalizing Definition~\ref{dfn:corresponding}, given a set $S$ of
%joe4
%conditional events, the $\phi_S$ be the conjunction of the formulas
conditional events, let $\phi_S$ be the conjunction of the formulas
corresponding to the events in $S$.  (Thus, if $S$ consists of the
single ccce $\alpha$, then $\phi_S = \phi_\alpha$.)

}
%joe4: \end{commentout}
%joe3*: where do we use this result?  If we don't explicitly use it
%somewhere, I suggest that we cut it, which I've done for now.  Feel
%free to reinsert it, or perhaps move it ot the appendix
\commentout{
\begin{lemma}
  %joe2: using earlier notation
  %$Pr[\mathcal{F}(S)] =\sum_{\alpha\in \mathcal{F}(S)} Pr[\alpha] =
%joe3
%  $\Pr[\mathcal{F}(S)] =\sum_{\alpha\in \mathcal{F}(S)} \Pr_M(\alpha) =
  Given a CBN $M$ and a set $S$ of ccces for $M$,
$$\Pr_M(\mathcal{F}(S)) =\sum_{\alpha\in \mathcal{F}(S)} \Pr_M(\phi_\alpha) =
  %joe2*: \phi_S would be more consistent with the earlier notation here
%than %S_\lpaha 
%  Pr[S_\alpha]$, where $S_\alpha$ denotes the formula corresponding
%  the conjuncts in $S$. 
%joe3
  %  \Pr_M(\phi_S)$,
    \Pr_M(\phi_S).$$
%joe3: said it above
%where $\phi_S$ is the conjunction of the elements of $S$ after
%  converting them into their corresponding formula (replacing
%  conditional event   by an intervention).
\label{lem:sumccce}
\end{lemma}
%TODO:{proof}
%joe3
%The proof follows from the proof of Theorem~\ref{ccce}.
\begin{proof}
  It is immediate from Theorem~\ref{ccce} that distinct ccces
are mutually exclusive in any causal model compatible with $M$.  The
result follows. \end{proof} 
%joe2*: I don't understand the next sentence.  What's \phi?
%joe3
%As a corollary, the total probability of $\mathcal{F}(\emptyset) = 1$.
As a corollary, it follows that  $\Pr_M(\mathcal{F}(\emptyset)) = 1$.
}
%joe3
%
%\subsection{Counterfactual Estimation using Observational Data}

%joe4:
Two types of counterfactual formulas that have proved particularly
useful are the \emph{probability of necessity}  and the
\emph{probability of sufficiency}; we focus on them in this section.
%joe3
%As discussed by Pearl~\cite{pearl:2k}, counterfactual analysis is
As discussed by Pearl~\citeyear{pearl:2k}, counterfactual analysis
%joe4
%is commonly used
is particularly useful
%joe3
%to understand the impact of a decision on the outcome. Most common form
%of counterfactual questions ask: `What are the chances that the outcome
%would not have been favorable if A were not true?' This standard
%counterfactual question captures the \emph{necessary} cause of an
%output. Similarly, other 
%notions include \emph{sufficient cause} (the outcome would become favorable
%if $A$ were changed) and \emph{necessary and sufficient cause} of an outcome.
%Below, we define the three notions of attribution more formally.
%joe4
%to understand the impact of a decision on the outcome. For example, we
when it comes to understanding the impact of a decision on the
outcome. For example, we 
might be interested in the probability that an  outcome $O$
%joe4
%would not have been favorable if $A$ were not true?' This
%captures the the extent to which $A$ is a \emph{necessary} cause of $O$.
would not have been favorable if $A$ were not true. This 
captures the extent to which $A$ is a \emph{necessary} cause of $O$.
Similarly, we might be interested in whether $A$ is \emph{sufficient}
for $O$: that is if $A$ were true, would $O$ necessarily be true?
We now review the formal definitions of these notions; see
\cite{pearl:2k} for more discussion.
\begin{definition}
%joe3
  %  Let $X$ and $Y$ be two binary variables in a causal model $M$.
    Let $X$ and $Y$ be binary variables in a causal model or CBN $M$. 
\begin{enumerate}%[wide]
%joe3*: I added X and Y explicitly.  I think this is useful, even if
%Pearl doesn't do it.
  %\item Probability of necessity: $\PN_M = \Pr_M[ [X\leftarrow 0] (Y=0)
  \item \emph{Probability of necessity of $X$ for $Y$}: 
  $\PN_M^{X,Y} =
    \Pr_M( [X\leftarrow 0] (Y=0) 
  | X=1\land Y=1 )$. 
%joe3
%  \item Probability of sufficiency: $\PS_M = \Pr_M[ [X\leftarrow 1]
%    (Y=1) | X=0\land Y=0 ]$. 
\item \emph{Probability of sufficiency of $X$ for $Y$}: $\PS_M^{X,Y} = \Pr_M( [X\leftarrow 1]
    (Y=1) \mid  X=0\land Y=0 )$.
%joe3
  %\item Probability of necessity and sufficiency: $\PNS_M = \Pr_M[
  \item \emph{Probability of necessity and sufficiency of $X$ for $Y$}:
  $\PNS_M^{X,Y} = \Pr_M(  
  [X\leftarrow 1] (Y=1) \land [X\leftarrow 0](Y=0) ).$ 
\end{enumerate}
\end{definition}
%joe3: we can just get a range
%joe5
%Pearl \citeyear{pearl:2k} gives examples showing that the probability
%of necessity and the probability of sufficiency in a CBN cannot be
Pearl \citeyear{pearl:2k} gives examples showing that neither the probability
of necessity nor the probability of sufficiency in a CBN can be
identified; we can just determine a range for these probabilities.
%joe4
%But with our (arguably extremly natural) independence assumptions,
%joe5
%But with our (arguably natural) independence assumptions,
But with our (arguably natural and realistic) independence assumptions,
they can be identified, justifying our notation.
%joe3*: added here
Moreover, these
probabilities can be computed using only conditional probabilities of
(singly) exponentially many simple formulas (not involving
interventions).
%joe4
%As we said earlier, the fact that these formulas do not involve 
%joe5: line shaving
%The fact that these formulas do not involve  interventions means that
Since these formulas do not involve  interventions,
they can be estimated from observational data, 
without requiring involving controlled experiments.  
%joe4: moved from commented out material above
Thus, our results and assumptions have significant practical implications.

%Connect with the intro example here to explain where these are used. GIVE example of income based on education.
%Exposure to covid and Covid Positive and Age and Immunity 

%joe4: As I said, I think we can skip this
%joe2
%We first show an example where counterfactual probabilities can be
%joe3
%We first give an example where the probability of sufficiency can be
%represented in terms of
\commentout{
As a warmup to our more general results on computing the probability of
necessity and sufficiency, we consider a simple example where the
probability of sufficiency can be computed using
%joe2
%conditional probabilities which can be estimated from
conditional probabilities that can be  estimated from
observational data.
%joe3
%Additionally, we show that  
% the number of conditional probabilities is not doubly exponential in
% the number of variables.
The example also suggests that we do this using (singly) exponential
conditional probabilities.

\begin{example}\label{xam:PS}
%joe2
  %  Consider a causal graph as follow:
%joe3
%  Consider a causal model $M$ with the following causal graph, where
  Consider a CBN $M$ with the following dag, where
  all variables are binary:

  %joe3
  \begin{center}
\begin{tikzpicture}
  \tikzstyle{var} = [circle, draw, thick]

  % Define nodes with the style
  \node[var] (U) at (0,0) {U};
  \node[var] (Z) at (-1,-1) {Z};
  \node[var] (X) at (1,-1) {X};
    \node[var] (Y) at (0,-2) {Y};

 \draw[->, thick] (U) -- (Z);
  \draw[->, thick] (U) -- (X);
    \draw[->, thick] (Z) -- (Y);
      \draw[->, thick] (X) -- (Y);
\end{tikzpicture}
%joe3
\end{center}

%joe2
%Consider the formula $\phi = X=0 \land Y=0 \land [X \gets
The formula $\phi = X=0 \land Y=0 \land [X \gets
  1](Y=1)$ is true in only four fccces: (a) $U=0 \land
(X=0 \mid U=0) \land (Z=z \mid U=0) \land (Y=0 \mid (U=0, X=0, Z=z)) \land 
(Y=1 \mid (U=0, X=1, Z=z))$; and
%(the value of $X$ when $U=1$ is irrelevant,
%as is the value of $Y$ when $U=1$) and
(b) $U=1 \land )X=0
\mid U=1)  \land (Z=z \mid U=1)\land (Y=0 \mid (U=1, X=0, Z=z)) \land 
%joe2
%(Y=1 \mid (U=1, X=1, Z=z))$ for both values of $z$.
(Y=1 \mid (U=1, X=1, Z=z))$, where $z\in \{0,1\}$.
%joe3
%Simplifying this further, we get
A straightforward computation (see the appendix for details) shows
that for all causal models compatible with $M$, 
%joe2
%$\sum_{z\in\{0,1\}} Pr[X=0 \land Z=z]nPr[Y=0|X=0, Z=z] Pr[X=1, Z=z]$
that for all causal models $M'$ compatible with $M'$
%joe3* I don't think that the last term here is quite right.  Please check
%{\scriptsize $${\Pr}_M(\phi) = \sum_{z\in\{0,1\}} {\Pr}_M(X=0 \land
%  Z=z){\Pr}_M(Y=0 \mid (X=0 \land Z=z)) {\Pr}_M(X=1 \land Z=z).$$}
$$\begin{array}{ll}
  \Pr_{M}(\phi)  
=  \sum_{z\in\{0,1\}} & \hspace{-.12in} \Pr_{M'}(X=0 \land Y=0 \land 
  Z=z)\\ & \land [X \gets 1, Z \gets x](Y=1)).
  \end{array}
$$
Moreover, if $M'$ is i-compatible with $M$, then 
$$\begin{array}{lll}
  \Pr_{M}(\phi) &\hspace{-.1in}  = \sum_{z\in\{0,1\}} {\Pr}_{M'}(X=0 \land Y=0 \land 
  Z=z)\\ &\hspace{.75in} \Pr_{M'}([X \gets 1, Z \gets x](Y=1))\\
 &\hspace{-.1in}  = \sum_{z\in\{0,1\}} {\Pr}_{M'}(X=0 \land Y=0 \land 
  Z=z)\\ &\hspace{.75in} \Pr_{M}(Y=1 \mid (X=1 \land Z=x)).
\end{array}
$$
\end{example}

%joe3: unnecessary; we essentially said this.
%This example shows that the probability estimates do not have to
%contain all conditional probabilities 
%from the fccce's.
%joe2
%Let us generalize this example.
%joe3
%The following theorem generalizes the example to analyze the
%probabilities of `sufficiency', `necessity' and `necessity and
%sufficiency'
}
%joe4: \end{commentout}

%joe4: redundant
%We now show that 
%$\PN_M^{X,Y}$, $\PS_M^{X,Y}$, and $\PNS_M^{X,Y}$ can all be computed
%using only conditional probabilities of at jost (singly) exponentially
%many formulas that are all simple (i.e. do not involve interventiobns,
%Yethus can be determined using only observational data).

%joe3*: moved out of theorem, and added subscript X to Pa, since
%normally Pa[Y]) consists of all parents of Y.
Let $Pa^X(Y)$ consist of all the parents of $Y$ other than $X$.
For a set ${\mathcal Z}$ of variables,
let $\mathcal{T}_{{\mathcal Z}}$ consist of all possible settings of the
variables in ${\mathcal Z}$.

\begin{theorem}\label{thm:pns}
%joe2
  %  Consider the expression,
%joe3
% If   $M$ is a causal model where $Y$ is a child of $X$, then
 If   $M$ is a CBN where $Y$ is a child of $X$, then
  %joe3*: combined PN, PS, and PNS into one theorem
    %    { \scriptsize $$ \PS_M ={\sum_{c^j_{Pa(Y)} \in 
    \begin{itemize}
    \item[(a)] $    \PN_M^{X,Y}  = \sum_{c^j_{Pa^X(Y)} \in
        \mathcal{T}_{Pa^X(Y)}} \mbox{\hspace{.2in}} {\Pr}_M(
      Pa^X(Y)= c^j_{Pa^X(Y)}\mid Y=1 \land X=1 )
      \\  \mbox{\hspace{1.5in}} \Pr_M
  (Y=0 \mid X=0 \land Pa^X(Y)= c^j_{Pa^X(Y)});$
    \item[(b)] $
      %\begin{array}{ll}
   \PS_M^{X,Y}   = \sum_{c^j_{Pa^X(Y)} \in  
   \mathcal{T}_{Pa^X(Y)}} \mbox{\hspace{.2in}} {\Pr}_M( Pa^X(Y)=
   c^j_{Pa^X(Y)}\mid 
   Y=0 \land X=0 )\\  \mbox{\hspace{1.5in}} {\Pr}_M 
   (Y=1 \mid X=1 \land Pa^X(Y)= c^j_{Pa^X(Y)});$
 \item[(c)] $\PNS_M^{X,Y}  = \PS_M^{X,Y} \cdot \Pr_M(X=0 \land Y=0) + \PN_M^{X,Y} \cdot \Pr_M(X=1 \land Y=1).$
   %\end{array}$
  %joe2
%{\scriptsize
  %  $$Pr[\psi] = \sum_{c^j_{Pa(Y)}\in \mathcal{T}_{Pa(Y)}} Pr[X=0 \land
%        \mathcal{T}_{Pa(Y)}} Pr[X=0 \land 
%    Pa[Y]= c^j_{Pa(Y)}] Pr[ Y=0 |X=0 \land Pa[Y]= c^j_{Pa(Y)}]  \Pr_M
%}where $Pa(Y)$ contains the parents of $Y$ excluding $X$ and
%$\mathcal{T}_{Pa(Y)}$ contains all possible settings of the parents of
   %$Y$.
 %   $$\begin{array}{ll} &\Pr_M(\psi)\\ =  &\sum_{c^j_{Pa(Y)}\in
  %      \mathcal{T}_{Pa(Y)}} \Pr_M(X=0 \land 
  %  Pa[Y]= c^j_{Pa(Y)}) \Pr_M( Y=0 \mid X=0 \land Pa[Y]= c^j_{Pa(Y)}) \\ &\Pr_M
  %(Y=1 \mid X=1 \land Pa[Y]= c^j_{Pa(Y)}), \end{array}$$ 
%joe3: said it above
%   where $Pa(Y)$ consists of all the parents of $Y$ other than $X$ and
%$\mathcal{T}_{Pa(Y)}$ consists of all possible settings of the
%variables in $Pa(Y)$. 
\end{itemize}
 \end{theorem}
We prove the calculation for the probability of sufficiency, $  \PS_M^{X,Y} $.
%joe5
%The same proof extends for the probability of necessity, $\PN_M^{X,Y}$, and
%of necessity, $\PN_M^{X,Y}$, and
%$\PNS_M^{X,Y}$ can be represented in terms of $\PS_M^{X,Y}$ and $\PN_M^{X,Y}$
Essentially the same argument can be used to compute the probability
of necessity, $\PN_M^{X,Y}$.  Finally,
for part (c), we use the representation of 
$\PNS_M^{X,Y}$ in terms of $\PS_M^{X,Y}$ and $\PN_M^{X,Y}$
%joe5
%using~\cite[Lemma 9.2.6]{pearl:2k}.
given in~\cite[Lemma 9.2.6]{pearl:2k}.

\begin{proof}[Proof of Theorem~\ref{thm:pns} (b)]
%joe2  
  %Let $\mathcal{T} = \{{c}^1,\ldots, {c}^{|\mathcal{T}|}\}$ denote all
%possible settings of variables $\mathcal{C} = \mathcal{U}\cup
%This set contains $2^{n-2}$ values, where $n= |\mathcal{U}\cup
%joe3: already defined
%  Let $\mathcal{T}_{\mathcal Z} = \{{c}^1,\ldots,
%  {c}^{|\mathcal{T}|}\}$ consist of all 
%settings of the variables in $\mathcal{Z} = \mathcal{U}\cup
Let $\mathcal{Z} = \mathcal{U}\cup
  \mathcal{V} \setminus \{X,Y\}$. 
$\mathcal{T}_{\mathcal Z}$ has $2^{n-2}$ settings, where $n= |\mathcal{U}\cup
  \mathcal{V}|$.
%joe3
%  We refer to the setting of variable $Z$ in the 
  %setting $c^j$ as $c_{Z}^j$. Therefore,
  For a setting $c \in \mathcal{T}_{\mathcal Z}$, let $c_Z$ be the
  setting of the variable $Z$ in $c$.

%joe3
%  $$\PS_M = \frac{\Pr_M( (X=0) \land (Y=0) \land
%  [X\leftarrow 1](Y=1)}{\Pr_M( (X=0) \land (Y=0)) }.$$
  By definition,
\begin{equation}\label{eq:ps}
  \PS_M^{X,Y} = \frac{\Pr_M( (X=0) \land (Y=0) \land
    [X\leftarrow 1](Y=1))}{\Pr_M( (X=0) \land (Y=0)) }.
  \end{equation}

Let the numerator $(X=0) \land (Y=0) \land
%joe3
%[X\leftarrow 1](Y=1)$ be $\psi$. Therefore,
[X\leftarrow 1](Y=1)$ be $\psi$. Then we have
%joe2: again Pr -> \Pr_M, | => \mid
%$$Pr[\psi] = \sum_{{c}^j \in \mathcal{T}} Pr[ (X=0) \land (Y=0) \land
%[X\leftarrow 1](Y=1) \land  (Z = c_Z^j | Pa(Z) = c^j_{Pa(Z)})
%  \forall Z\in \mathcal{C}  ]$$ 
$$\begin{array}{ll}
%joe3: c^j -> c; \mathcal{T} -> \mathcal{T}_{\mathcal{Z}
%  \Pr_M(\psi) = \\ \sum_{c^j \in \mathcal{T}} \Pr_M( (X=0) \land (Y=0) \land
%  [X\leftarrow 1](Y=1) \land  \bigwedge_{Z \in \mathcal{Z}}(Z = c_Z^j
%  \mid Pa(Z) = c^j_{Pa(Z)})).
  \Pr_M(\psi) = \\ \sum_{c \in \mathcal{T}_{\mathcal{Z}}} \Pr_M( (X=0) \land (Y=0) \land
  [X\leftarrow 1](Y=1) \land  \bigwedge_{Z \in \mathcal{Z}}(Z = c_Z
  )).
   \end{array}$$

%joe2
%Let us compute probability of $$\psi_{c_j} = (X=0) \land (Y=0) \land
%[X\leftarrow 1](Y=1) \land (Z = c_Z^j | Pa(Z) = c^j_{Pa(Z)}) \forall Z 
%\in \mathcal{C}  $$  
%joe3*: again, got rid of j's.  Here you mean Pa to be all parents,
%right?  (All the more reason for using the notation Pa^X above
%We next compute the probability of $$\psi_{c_j} = (X=0) \land (Y=0) \land
%[X\leftarrow 1](Y=1) \land \bigwedge_{Z \in \mathcal{Z}}(Z = c_Z^j \mid
%  Pa(Z) = c^j_{Pa(Z)}).$$ 
We next compute the probability of $$\psi_{c} = (X=0) \land (Y=0) \land
[X\leftarrow 1](Y=1) \land \bigwedge_{Z \in \mathcal{Z}}(Z = c_Z).$$ 
 %joe2
% Following Theorem~\ref{ccce}, we know that probability of
%    $\psi_{c^j}$ is equal to 
%joe2*: Again, what's \phi?  Is it the formula true?
    % $$\sum_{\alpha\in \mathcal{F}(\phi) } Pr[\alpha]{1}\{\alpha \implies
%\psi_{c^j}\}$$ 
  From Theorem~\ref{ccce}, it follows that 
%joe3: M is overloaded; also replacing \alpha by \phi_\alpha and c^j
%by c globally
  %  causal models     $M$,
%    $${\Pr}_M(\psi_{c^j}) =
% \sum_{\alpha \implies \psi_{c^j} } {\Pr}_M(\alpha) .$$
in all causal models $M'$ compatible with $M$,
    $${\Pr}_{M'}(\psi_{c}) =
 \sum_{\phi_\alpha \implies \psi_{c} } {\Pr}_{M'}(\phi_\alpha) .$$
Now
%joe3:
%$\alpha \implies \psi_{c^j}$ holds whenever $\alpha$ contains the
$\phi_\alpha \implies \psi_{c}$ holds whenever $\alpha$ contains the
%joe4
%following conjuncts:
following events:
\begin{enumerate}
%joe2: replacing | by \mid globally
  %\item $(X=0  | Pa(X) = c^j_{Pa(X)})  $
\item $(X=0  \mid Pa(X) = c_{Pa(X)})  $
\item $(Y=0   \mid X=0, Pa^X(Y) = c_{Pa^X(Y)})  $
\item $(Y=1   \mid X=1, Pa^X(Y) = c_{Pa^X(Y)})  $
%joe2
%\item $(Z = c_Z^j \mid Pa(Z) = c^j_{Pa(Z)}) \forall Z \in \mathcal{C}$
  \item $(Z = c_Z \mid Pa(Z) = c_{Pa(Z)})$, for all $Z \in \mathcal{Z} $.
\end{enumerate}
%joe2
%
%joe4: 
%therefore, equation above is
%joe4*:
Let $S_c$ consist of all ccces that contain these four events, and
let $\phi_{S_c}$ be the conjunction of the formulas
corresponding to the events in $S$.  Then by Theorem~\ref{ccce}, 
% \sg5:shouldnt this be \bigvee?
%joe10: again, yes
%$$\phi_{S_c} \dimp \bigwedge_{\alpha \in S_c} \phi_\alpha.$$
$$\phi_{S_c} \dimp \bigvee_{\alpha \in S_c} \phi_\alpha.$$
Since the formulas $\phi_\alpha$ for distinct ccces in $S_c$ are
mutually exclusive, we have that 
%joe3
%It follows that
%joe4: unnecessary
%Thus,
% \begin{eqnarray} 
%   %Pr[\psi_{c_j}]  =  \sum_{\alpha\in \mathcal{F}(S) } Pr[\alpha]
%%joe3
%%   {\Pr}_M(\psi_{c}) =  \sum_{\alpha\in \mathcal{F}(S) }
%%   {\Pr}_M(\phi_\alpha), 
%   {\Pr}_{M'}(\psi_{c}) =  \sum_{\alpha\in \mathcal{F}(S) } {\Pr}_{M'}(\phi_\alpha),
% \end{eqnarray}
% 
%joe4
%where $S =\{ (X=0  \mid Pa(X) = c_{Pa(X)}), (Y=0   \mid X=0,
%  Pa(Y) = c_{Pa(Y)}), (Y=1   \mid X=1, Pa(Y) = c_{Pa(Y)}),
%%joe2
%%  \land  (Z = c_Z^j \mid Pa(Z) = c^j_{Pa(Z)}) \forall Z \in  \mathcal{C} $. 
%    \bigwedge_{Z \in \mathcal{Z}} (Z = c_Z \mid Pa(Z) =
%  c_{Pa(Z)})\}$.
%
%joe2*: I didn't follow this, Sainyam.    You seem to be connecting
%conditional probabilities with counterfactuals, and Lemma 5.2 doesn't
%seem to say anything about that.  Also, the  S_\alpha notation
%doesn't seem right here either.
%joe3
%  Using Lemma~\ref{lem:sumccce}, $\Pr_M(\psi_{c}) = \Pr_M(\phi_S)$
%joe4: no longer need lemma
%By Lemma~\ref{lem:sumccce},
$\Pr_{M'}(\psi_{c}) = \Pr_{M'}(\phi_S)$.
  %joe3: this seems unnecssary; it's already been defined.
%  , where
%  %joe2
%  %
%  \begin{eqnarray*}
%\phi_S =  &&[Pa(X) \leftarrow  c_{Pa(X)}](X=0)  \land [X\leftarrow0,
%    %joe3: for consistency, writing Pa(...) rather than Pa[...] globally.
%  %  Pa[Y]\leftarrow c_{Pa(Y)}](Y=0)\\
%    Pa(Y)\leftarrow c_{Pa(Y)}](Y=0)\\ 
% &&  \land [X\leftarrow1, Pa(Y)\leftarrow c_{Pa(Y)}](Y=1)  \land
% \bigwedge_{ Z\in\mathcal{Z}} [Pa(Z) \leftarrow c_{Pa(Z)}](Z = c_Z) 
% \end{eqnarray*}
% 
 Therefore,
%joe5
 % { \scriptsize
 $$\begin{array}{ll}
%joe3: For consistency, writing \Pr(...) instead of \Pr[...] globally
   %   \Pr[\psi] = && \sum_{{c} \in \mathcal{T}} \Pr[\phi_S]\\
%joe3*: \phi_S -> \psi_c; please check!  Also need subscript on Pr
   %   \Pr(\psi) = && \sum_{{c} \in \mathcal{T}_{\mathcal{Z}}} \Pr(\phi_S)\\
      &\Pr_{M'}(\psi) \\
      &=  \sum\limits_{c \in \mathcal{T}_{\mathcal{Z}}} \Pr_{M'}(\phi_c)\\
      & = \sum\limits_{c\in \mathcal{T}_{\mathcal{Z}}} \Pr_{M'}(X=0 \land Y=0
 \land  \bigwedge_{Z \in \mathcal{Z}} (Z = c_Z )  \land
        [X\leftarrow1, Pa^X(Y)\leftarrow c_{Pa^X(Y)}](Y=1) )\\ 
& =\sum\limits_{c_{Pa^X(Y)}\in \mathcal{T}_{Pa^X(Y)}} \\
& \hspace{0.4in}\Pr_{M'}(X=0 \land Y=0
 \land  Pa^X(Y)\leftarrow c_{Pa^X(Y)}  \land  [X\leftarrow1,
%joe5
   %   Pa(Y)\leftarrow c_{Pa(Y)}](Y=1) )
      Pa^X(Y)\leftarrow c_{Pa^X(Y)}](Y=1) ).
 \end{array}$$
% }
%joe3*
 % && \text{Using independence assumptions}\\
 If $M'$ is i-compatible with $M$, then we can further conclude that
%joe5
% { \scriptsize
$$\begin{array}{ll}
 &\sum_{c\in \mathcal{T}_{\mathcal{Z}}} \Pr_{M'}(X=0 \land Y=0 \land  \bigwedge_{Z \in \mathcal{Z}} (Z = c_Z )  \land  [X\leftarrow1, Pa^X(Y)\leftarrow c_{Pa^X(Y)}](Y=1) )\\
   &=\sum\limits_{c_{Pa^X(Y)}\in \mathcal{T}_{Pa^X(Y)}} \Pr_{M'}(X=0 \land Pa^X(Y)=
%joe5
%   c_{Pa(Y)}) \Pr_{M'}( Y=0 \mid X=0 \land Pa(Y)= c_{Pa(Y)})  \Pr_{M'} (Y=1
   c_{Pa^X(Y)})\\
   &  \hspace{1.1in}\Pr_{M'}( Y=0 \mid X=0 \land Pa^X(Y)= c_{Pa^X(Y)}) \\
 &  \hspace{1.1in} \Pr_{M'} (Y=1
   \mid X=1 \land Pa^X(Y)= c_{Pa^X(Y)})\\
&= \sum\limits_{c_{Pa^X(Y)}\in \mathcal{T}_{Pa^X(Y)}} \Pr_{M'}( Y=0 \land X=0 \land Pa^X(Y)= c_{Pa^X(Y)})\\
&\hspace{1.1in} \Pr_{M'} [Y=1
%joe5
%  \mid X=1 \land Pa(Y)= c_{Pa(Y)})
  % \end{array}$$}
    \mid X=1 \land Pa^X(Y)= c_{Pa^X(Y)}).
 \end{array}$$

%joe3
%Substituting the $\Pr(\psi]) into $\PS$, we get
Since $\Pr_{M'}(\psi) = \Pr_M(\psi)$, substituting the expression for
$\Pr_M(\psi)$ into (\ref{eq:ps}), we get
%joe3
%    {\scriptsize
%joe5
   $$\begin{array}{ll}
%$$
%joe5
%\PS_M = \\  \frac{\sum_{c_{Pa(Y)}\in \mathcal{T}_{Pa(Y)}} \Pr( Y=0
\PS_M = \\ \frac{\sum_{c_{Pa^X(Y)}\in \mathcal{T}_{Pa^X(Y)}} \Pr( Y=0
     \land X=0 \land Pa^X(Y)= c_{Pa^X(Y)})   \Pr (Y=1 
   \mid X=1 \land Pa^X(Y)= c_{Pa^X(Y)})}{\Pr(X=0\land Y=0)},
%joe5
\end{array}
   $$
      %joe3
      as desired.
   \end{proof}

%joe4
%We defer the proof of the theorem to Section~\ref{sec:pns} in the appendix.
%joe3*: we should move this proof to the appendix, and make the
%appendi single column.  This will mean that we'll have to put less
%effort into dealing with bad line breaks, and don't have to use a
%smaller font.  Since what follows is the
%proof of (b), perhaps we can say "We prove (b) here; the proof of (a)
%is similar.
%joe4: I commented out the proof from here; we shoudl move it to the
%appendix.  I also made some changes that allow us to avoid  that
%\mathcal{F}(S) notation and the lemma that was cut.
\fullv{
\begin{proof}
%joe2  
  %Let $\mathcal{T} = \{{c}^1,\ldots, {c}^{|\mathcal{T}|}\}$ denote all
%possible settings of variables $\mathcal{C} = \mathcal{U}\cup
%This set contains $2^{n-2}$ values, where $n= |\mathcal{U}\cup
%joe3: already defined
%  Let $\mathcal{T}_{\mathcal Z} = \{{c}^1,\ldots,
%  {c}^{|\mathcal{T}|}\}$ consist of all 
%settings of the variables in $\mathcal{Z} = \mathcal{U}\cup
Let $\mathcal{Z} = \mathcal{U}\cup
  \mathcal{V} \setminus \{X,Y\}$. 
$\mathcal{T}_{\mathcal Z}$ has $2^{n-2}$ settings, where $n= |\mathcal{U}\cup
  \mathcal{V}|$.
%joe3
%  We refer to the setting of variable $Z$ in the 
  %setting $c^j$ as $c_{Z}^j$. Therefore,
  For a setting $c \in \mathcal{T}_{\mathcal Z}$, let $c_Z$ be the
  setting of the variable $Z$ in $c$.

%joe3
%  $$\PS_M = \frac{\Pr_M( (X=0) \land (Y=0) \land
%  [X\leftarrow 1](Y=1)}{\Pr_M( (X=0) \land (Y=0)) }.$$
  By definition,
\begin{equation}\label{eq:ps}
  \PS_M^{X,Y} = \frac{\Pr_M( (X=0) \land (Y=0) \land
    [X\leftarrow 1](Y=1))}{\Pr_M( (X=0) \land (Y=0)) }.
  \end{equation}

Let the numerator $(X=0) \land (Y=0) \land
%joe3
%[X\leftarrow 1](Y=1)$ be $\psi$. Therefore,
[X\leftarrow 1](Y=1)$ be $\psi$. Then we have
%joe2: again Pr -> \Pr_M, | => \mid
%$$Pr[\psi] = \sum_{{c}^j \in \mathcal{T}} Pr[ (X=0) \land (Y=0) \land
%[X\leftarrow 1](Y=1) \land  (Z = c_Z^j | Pa(Z) = c^j_{Pa(Z)})
%  \forall Z\in \mathcal{C}  ]$$ 
$$\begin{array}{ll}
%joe3: c^j -> c; \mathcal{T} -> \mathcal{T}_{\mathcal{Z}
%  \Pr_M(\psi) = \\ \sum_{c^j \in \mathcal{T}} \Pr_M( (X=0) \land (Y=0) \land
%  [X\leftarrow 1](Y=1) \land  \bigwedge_{Z \in \mathcal{Z}}(Z = c_Z^j
%  \mid Pa(Z) = c^j_{Pa(Z)})).
  \Pr_M(\psi) = \\ \sum_{c \in \mathcal{T}_{\mathcal{Z}}} \Pr_M( (X=0) \land (Y=0) \land
  [X\leftarrow 1](Y=1) \land  \bigwedge_{Z \in \mathcal{Z}}(Z = c_Z
  \mid Pa(Z) = c_{Pa(Z)})).
   \end{array}$$

%joe2
%Let us compute probability of $$\psi_{c_j} = (X=0) \land (Y=0) \land
%[X\leftarrow 1](Y=1) \land (Z = c_Z^j | Pa(Z) = c^j_{Pa(Z)}) \forall Z 
%\in \mathcal{C}  $$  
%joe3*: again, got rid of j's.  Here you mean Pa to be all parents,
%right?  (All the more reason for using the notation Pa^X above
%We next compute the probability of $$\psi_{c_j} = (X=0) \land (Y=0) \land
%[X\leftarrow 1](Y=1) \land \bigwedge_{Z \in \mathcal{Z}}(Z = c_Z^j \mid
%  Pa(Z) = c^j_{Pa(Z)}).$$ 
We next compute the probability of $$\psi_{c} = (X=0) \land (Y=0) \land
[X\leftarrow 1](Y=1) \land \bigwedge_{Z \in \mathcal{Z}}(Z = c_Z \mid
  Pa(Z) = c_{Pa(Z)}).$$ 
 %joe2
% Following Theorem~\ref{ccce}, we know that probability of
%    $\psi_{c^j}$ is equal to 
%joe2*: Again, what's \phi?  Is it the formula true?
    % $$\sum_{\alpha\in \mathcal{F}(\phi) } Pr[\alpha]{1}\{\alpha \implies
%\psi_{c^j}\}$$ 
  From Theorem~\ref{ccce}, it follows that 
%joe3: M is overloaded; also replacing \alpha by \phi_\alpha and c^j
%by c globally
  %  causal models     $M$,
%    $${\Pr}_M(\psi_{c^j}) =
% \sum_{\alpha \implies \psi_{c^j} } {\Pr}_M(\alpha) .$$
in all causal models $M'$ compatible with $M$,
    $${\Pr}_{M'}(\psi_{c}) =
 \sum_{\phi_\alpha \implies \psi_{c} } {\Pr}_{M'}(\phi_\alpha) .$$
Now
%joe3:
%$\alpha \implies \psi_{c^j}$ holds whenever $\alpha$ contains the
$\phi_\alpha \implies \psi_{c}$ holds whenever $\alpha$ contains the
%joe4
%following conjuncts:
following events:
\begin{enumerate}
%joe2: replacing | by \mid globally
  %\item $(X=0  | Pa(X) = c^j_{Pa(X)})  $
\item $(X=0  \mid Pa(X) = c_{Pa(X)})  $
\item $(Y=0   \mid X=0, Pa(Y) = c_{Pa(Y)})  $
\item $(Y=1   \mid X=1, Pa(Y) = c_{Pa(Y)})  $
%joe2
%\item $(Z = c_Z^j \mid Pa(Z) = c^j_{Pa(Z)}) \forall Z \in \mathcal{C}$
  \item $(Z = c_Z \mid Pa(Z) = c_{Pa(Z)})$, for all $Z \in \mathcal{Z} $.
\end{enumerate}
%joe2
%
%joe4: 
%therefore, equation above is
n%joe4*:
Let $S_c$ consist of all ccces that contain these four events, and
let $\phi_{S_c}$ be the conjunction of the formulas
corresponding to the events in $S$.  Then by Theorem~\ref{ccce}, 
% \sg5:shouldnt this be \bigvee?
%joe10: Yes!  Corrected
%$$\phi_{S_c} \dimp \bigwedge_{\alpha \in S_c} \phi_\alphfan.$$
$$\phi_{S_c} \dimp \bigvee_{\alpha \in S_c} \phi_\alphfan.$$
Since the formulas $\phi_\alpha$ for distinct ccces in $S_c$ are
mutually exclusive, we have that 
%joe3
%It follows that
%joe4: unnecessary
%Thus,
% \begin{eqnarray} 
%   %Pr[\psi_{c_j}]  =  \sum_{\alpha\in \mathcal{F}(S) } Pr[\alpha]
%%joe3
%%   {\Pr}_M(\psi_{c}) =  \sum_{\alpha\in \mathcal{F}(S) }
%%   {\Pr}_M(\phi_\alpha), 
%   {\Pr}_{M'}(\psi_{c}) =  \sum_{\alpha\in \mathcal{F}(S) } {\Pr}_{M'}(\phi_\alpha),
% \end{eqnarray}
% 
%joe4
%where $S =\{ (X=0  \mid Pa(X) = c_{Pa(X)}), (Y=0   \mid X=0,
%  Pa(Y) = c_{Pa(Y)}), (Y=1   \mid X=1, Pa(Y) = c_{Pa(Y)}),
%%joe2
%%  \land  (Z = c_Z^j \mid Pa(Z) = c^j_{Pa(Z)}) \forall Z \in  \mathcal{C} $. 
%    \bigwedge_{Z \in \mathcal{Z}} (Z = c_Z \mid Pa(Z) =
%  c_{Pa(Z)})\}$.
%
%joe2*: I didn't follow this, Sainyam.    You seem to be connecting
%conditional probabilities with counterfactuals, and Lemma 5.2 doesn't
%seem to say anything about that.  Also, the  S_\alpha notation
%doesn't seem right here either.
%joe3
%  Using Lemma~\ref{lem:sumccce}, $\Pr_M(\psi_{c}) = \Pr_M(\phi_S)$
%joe4: no longer need lemma
%By Lemma~\ref{lem:sumccce},
$\Pr_{M'}(\psi_{c}) = \Pr_{M'}(\phi_S)$.
  %joe3: this seems unnecssary; it's already been defined.
%  , where
%  %joe2
%  %
%  \begin{eqnarray*}
%\phi_S =  &&[Pa(X) \leftarrow  c_{Pa(X)}](X=0)  \land [X\leftarrow0,
%    %joe3: for consistency, writing Pa(...) rather than Pa[...] globally.
%  %  Pa[Y]\leftarrow c_{Pa(Y)}](Y=0)\\
%    Pa(Y)\leftarrow c_{Pa(Y)}](Y=0)\\ 
% &&  \land [X\leftarrow1, Pa(Y)\leftarrow c_{Pa(Y)}](Y=1)  \land
% \bigwedge_{ Z\in\mathcal{Z}} [Pa(Z) \leftarrow c_{Pa(Z)}](Z = c_Z) 
% \end{eqnarray*}
% 
 Therefore,
%joe5
% { \scriptsize
$$\begin{array}{ll}
%joe3: For consistency, writing \Pr(...) instead of \Pr[...] globally
   %   \Pr[\psi] = && \sum_{{c} \in \mathcal{T}} \Pr[\phi_S]\\
%joe3*: \phi_S -> \psi_c; please check!  Also need subscript on Pr
   %   \Pr(\psi) = && \sum_{{c} \in \mathcal{T}_{\mathcal{Z}}} \Pr(\phi_S)\\
      \Pr_{M'}(\psi) &=  \sum_{c \in \mathcal{T}_{\mathcal{Z}}} \Pr_{M'}(\phi_c)\\
      & = \sum_{c\in \mathcal{T}_{\mathcal{Z}}} \Pr_{M'}(X=0 \land Y=0
 \land  \bigwedge_{Z \in \mathcal{Z}} (Z = c_Z )  \land
        [X\leftarrow1, Pa(Y)\leftarrow c_{Pa(Y)}](Y=1) )\\ 
& =\sum_{c_{Pa(Y)}\in \mathcal{T}_{Pa(Y)}} \Pr_{M'}(X=0 \land Y=0
 \land  Pa(Y)\leftarrow c_{Pa(Y)}  \land  [X\leftarrow1,
%joe5
   %   Pa(Y)\leftarrow c_{Pa(Y)}](Y=1) )
      Pa(Y)\leftarrow c_{Pa(Y)}](Y=1) ).
 \end{array}$$
%joe3*
 % && \text{Using independence assumptions}\\
 If $M'$ is i-compatible with $M$, then we can further conclude that
$$\begin{array}{ll}
 &\sum_{c\in \mathcal{T}_{\mathcal{Z}}} \Pr_{M'}(X=0 \land Y=0 \land  \bigwedge_{Z \in \mathcal{Z}} (Z = c_Z )  \land  [X\leftarrow1, Pa(Y)\leftarrow c_{Pa(Y)}](Y=1) )\\
   &=\sum_{c_{Pa(Y)}\in \mathcal{T}_{Pa(Y)}} \Pr_{M'}(X=0 \land Pa(Y)=
   c_{Pa(Y)}) \Pr_{M'}( Y=0 \mid X=0 \land Pa(Y)= c_{Pa(Y)})  \Pr_{M'} (Y=1
   \mid X=1 \land Pa(Y)= c_{Pa(Y)})\\
&= \sum_{c_{Pa(Y)}\in \mathcal{T}_{Pa(Y)}} \Pr_{M'}( Y=0 \land X=0 \land Pa(Y)= c_{Pa(Y)}]  \Pr_{M'} [Y=1
%ejo5
  %  \mid X=1 \land Pa(Y)= c_{Pa(Y)})
    \mid X=1 \land Pa(Y)= c_{Pa(Y)}).
 \end{array}$$

%joe3
%Substituting the $\Pr(\psi]) into $\PS$, we get
Since $\Pr_{M'}(\psi) = \Pr_M(\psi)$, substituting the expression for
$\Pr_M(\psi)$ into (\ref{eq:ps}, we get
%joe3
%    {\scriptsize
   $$\begin{array}{ll}
   \PS_M = \\  \frac{\sum_{c_{Pa(Y)}\in \mathcal{T}_{Pa(Y)}} \Pr( Y=0 \land X=0 \land Pa(Y)= c_{Pa(Y)})   \Pr (Y=1
   \mid X=1 \land Pa(Y)= c_{Pa(Y)})}{\Pr(X=0\land Y=0)},
   \end{array}
   $$}
      %joe3
%joe5
%as desired.
%   \end{proof}
%}
%joe4: \end{fullv}
%joe3: we've already said this.  We can stress it even mro
%This shows that the  probability of sufficiency can be estimated in general
% in terms of conditional probabilities involving $X, Y$ and parents of $Y$.
%joe3: already in theorem statement
%Similarly, we can show that the probability of necessity evaluates to
% {\scriptsize
%  $$ \begin{array}{ll}
%  \PN_M =\\{\sum_{c_{Pa(Y)} \in
%        \mathcal{T}_{Pa(Y)}} {\Pr}_M(  Pa(Y)= c_{Pa(Y)}\mid Y=1 \land
%        X=1 ) {\Pr}_M 
%  (Y=0 \mid X=0 \land Pa(Y)= c_{Pa(Y)})}   \end{array}$$}
%joe3*: moved above, into theorem.  Didn't Pearl already have this or
%something like it?  If so, we should reference him
\commentout{
 To evaluate $\PNS_M$ we use the following lemma.
\begin{lemma}
$\PNS_M = \PS_M \cdot \Pr(X=0 \land Y=0) + \PN_M \cdot \Pr(X=1 \land Y=1)$
\end{lemma}
\begin{proof}
This proof is adapted from~\cite{pearl:2k}.
{\scriptsize
\begin{eqnarray*}
%joe3*: this should go into the appendix
  &&\PNS_M \\
&=& {\Pr}_M[ [X\leftarrow 1] (Y=1) \land [X\leftarrow 0](Y=0) ]\\
&=&{\Pr}_M[ [X\leftarrow 1] (Y=1) \land [X\leftarrow 0](Y=0)\land (X=0)\land(Y=0) ] \\
&& +{\Pr}_M[ [X\leftarrow 1] (Y=1) \land [X\leftarrow 0](Y=0)\land (X=1)\land(Y=1) ] \\
&=& \PS_M \cdot {\Pr}_M[X=0\land Y=0] + \PN_M\cdot {\Pr}_M[X=1\land Y=1]
\end{eqnarray*}
}
\end{proof}
}
%joe3: I already said this earlier.  We can stress it even more (but
%it should be stressed earlier, in any case
\commentout{
According to~\cite{pearl:2k}, accurate estimation of counterfactual scores requires access to
 experimental data. However, our result shows that the score can be estimated under
independence assumptions. This result can have huge implications in settings where
access to experimental data is not possible.
} 
%joe3: typo, I assume
%\noindent \textbf{Extensions.} Note that this result extends to the case
%where $Y$ is any descendant of $Y$ (not necessarily the child of $Y$). 
We can extend Theorem~\ref{thm:pns} to the case
where $Y$ is any descendant of $X$ (not necessarily a child of $X$). 
In this case, the term involving $Pa(Y)$ would change to the set of
%joe3*: isn't this more accurate?
%variables
the ancestors of $Y$ 
at the same level as $X$ in the topological ordering of the 
variables. 
%Furthermore, any general formula involving conditional probabilities or multiple
%iinterventions can also be simplified using the 
%joe4*: I'm not sure what technique you have in mind here, Sainyam.
%Can you say more here?
%same technique. 
%sg1:added more to the next paragraph
We can further extend Theorem~\ref{thm:pns} to arbitrary formulas $\psi$, where
%where the only subformulas $\psi'$ of $\psi$ that involve interventions have
%%joe8
%%the form $[X \leftarrow x]\psi''$
%the form $[X \leftarrow x]\psi''$.
%%, where the only variables that appear in $\psi''$
%%are   descendants of  X.  
%(Note that we are intervening only on a single variable X
%%joe8
%%here, not a set of variables).  If $\psi$ has this form, then
%here, not a set of variables.)  If $\psi$ has this form, then
$\Pr(\psi)$ can 
be determined by calculating the probability of formulas that do not involve
interventions (although they may involve conditional probabilities), and 
%joe8*: this doesn't seem right
%thus can be determined using only observational information,
%which requires singly-exponential many conditional probabilities.
thus can be determined using only observational information.
The key idea of the proof is to convert $\psi$ to a disjunction of
conjunctions, where the disjuncts are mutually exclusive and have the
form 
$\psi_i = \psi_{i0} \land \left(\bigwedge_{j\in \{1,\ldots,r\}} \psi_{ij}\right)$,
  where $\psi_{i0} = \left(\bigwedge_{j\in\{1,\ldots,s\}} (Z_{ij}= z_{ij})\right)$ is a simple
 formula (with no intervention), and
$\psi_{ij}$ for $j>0$ has the form
 $[\vec{X}_j\leftarrow \vec{x}_j] (\bigwedge_{k\in \{1,\ldots,t\}}
%joe8*: Do we need Y_{ijk} to be a child?  In any case, note that the
%theorem statement does not have this requirement, and the proof
%requires "desxcendant".  I replaced "child" by descendant here
 % Y_{ijk}=y_{ijk})$, where $Y_{ijk}$ is a child of
  Y_{ijk}=y_{ijk})$, where $Y_{ijk}$ is a descendant of 
   $\vec{X}_j$ in $M$, so that we can apply the ideas in
the proof of Theorem~\ref{thm:pns} to each disjunct separately. 
 In terms of complexity, we show that $\Pr(\psi)$ can be estimated in
$O(m\cdot 2^{nr^*})$ conditional probability calculations, where $r^*$ is the maximum
number of conjuncts in a disjunction $\psi_i$
 that involve at least one intervention, and $m$ is the number of disjuncts in the DNF.
 Unfortunately, for an arbitrary formula $\psi$, determining $\Pr(\psi)$ may
involve  doubly-exponentially many conditional probabilities.  
%We defer details
%sgn2: added reference
%to Section~\ref{sec:arbitrary} in the appendix.

 \begin{theorem}
 Given a CBN $M = (\S,\P)$ and an arbitrary formila $\psi$,
 then $\Pr(\psi)$ can be
determined by taking the probability of formulas that do not involve
interventions (although they may involve conditional probabilities),
and thus can be determined using only observational information.
 \commentout{
 If $\psi$ is an arbitrary formula such that all subformulas $\psi'$ of $\psi$
that involve interventions have the form $[X \leftarrow x]\psi''$,
% where the only
%variables that appear in $\psi''$ are descendants of X,
then $\Pr(\psi)$ can be
determined by taking the probability of formulas that do not involve
interventions (although they may involve conditional probabilities),
%joe8: If we need space, we can cut this line.
and thus can be determined using only observational information.
}
\label{thm:arbitrary}
 \end{theorem}

\commentout{
%joe3
  %  For a general formula $\psi$ which involves an intervention
%$[X\leftarrow x]$ for some variable $X$  along with a conditioning
%over the same variable $X=x$, 
%joe4*: I think we need to slow down here.  This formula is
  %complicated.  We should  break it down for the reader. How does
  %this result relate to your previous claim about these techniques
  %applying to "any general formula"?  I think this result is trying
  %to show the extent to which the previous result can be extended, so
  %we should say that this is how we extend the previous result, and
  %that we don't think we can extend in much further.  In any case,
%this needs discussion, as I said.    We have plenty of space.
  %joe5: added next sentence.  
  %joe6*: rewrote.
%joe6*: rewrote; in particular added next sentence
%  We can replace interventions by conditional formulas in a slightly
%  more general setting.  To describe it, first observe that
We can further extend Theorem~\ref{thm:pns} to formulas that are the
conjunction of simple formulas and a collection of interventions
(rather than just consisting of a single intervention) with some minor
caveats.  To explain the caveats,  observe that
  if a    formula involves an intervention
%joe5*: There are two issues here, Sainyam.  First of all, there's no
%conditioning in the formula.  Second, you need to make it clearer
%whether it's just having the same variable, or that you're setting
%the variable to the same value (e.g., with X=x' it wouldn't work).
%I rewrote it to what I think you intended
%and conditioning on the same variable $X=x$, 
  $[X\leftarrow x]$ on some variable $X$  such that $X$ is also set to
  $x$ in the formula, 
such as $(X=0 \land [X\leftarrow 0](Y=1))$,
%joe3
%then
the intervention $X\leftarrow x$
%joe5*
%is redundant and can be dropped.
is redundant and can be dropped; for example,
%joe10: made the example slightly more complicated
%$X=0 \land [X\leftarrow 0](Y=1)$ is equivalent to $X=0 \land Y=1$.
$X=0 \land [X\leftarrow 0, Z \leftarrow 1](Y=1)$ is equivalent to
$X=0 \land [Z leftarrow 1](Y=1)$. 
%joe5
%Following the proof of Theorem~\ref{thm:pns},
%joe6: cut this, and replaced it by a single sentence
%As in the proof of Theorem~\ref{thm:pns}, 
%all interventions in 
%%joe5
%%the general formula
%an arbitrary formula are independent when conditioned on the parents
%of the outcome variable.
%Using this property, for each setting of the parents of all outcome
%variables, the joint probability of multiple interventions
%%joe5*
%%can be written as a product of conditional probability for
%%each intervention. Adding this probability for each setting of the
%%parents of  
%can be written as the product of conditional probabilities for
%each intervention, once we drop redundant interventions.
%Adding this probability for each setting of the parents of 
%outcome variables gives the final expression, as shown below.
%joe5*: I still find the above a bit confusing.  We should also say
%whether this is as far as we think the result can be pushed.
%joe3: no point in proliferating names
%We refer to the formula which does not contain any such
% redundant interventions as a reduced formula.
%joe6: added
The next theorem shows that the probability of  a formula that is a
conjunction of a 
simple formula (with no interventions) and a collection of
interventions and has no redundant 
interventions can be computed using conditional probabilities of
simple formulas.  Furthermore, this result can be extended to 
calculate the probability of disjunctions of such
 formulas. Thus, again, we need just observational data for a
 large class of formulas. We defer the proof of the theorem
 to the appendix.}

\commentout{
 \begin{theorem}
%joe3
%   Given
%\begin{enumerate}
   %\item  a causal model $M$
   %\item a reduced formula $\psi\in \Lcaus^+(\S)$, where
%    $\psi = \psi' \land \left(\bigwedge_{i\in \{1,\ldots,k\}} \psi_i\right)$
   If $M$ is a CBN and 
 $\psi = \psi' \land \left(\bigwedge_{i\in \{1,\ldots,k\}} \psi_i\right)$,
%such that $\psi'$ does not contain any intervention (also referred to
%as a simple 
   %formula in Section~\ref{sec:problem})
%      and  the  $k$ distinct interventional expressions
% $\psi_i  = [X_i\leftarrow 1] (Y_i=1)$ where $Y_i$ is a
%   child of $X_i$ in the model $M$, 
%\end{enumerate}
%joe5: If you're going to use the term "redundant intervention", you
%should define it formally above
   %   where $\psi$ contains no redundant intervention,
%      $\psi'$ is a simple formula (with no interventions) 
   where $\psi$ contains no redundant interventions,
      $\psi' = \left(\bigwedge_{j\in\{1,\ldots,t\}} (Z_j = z_j)\right)$ is a simple formula (with no interventions), and
$\psi_i$ has the form
   $[X_i\leftarrow 1] (Y_i=1)$, where $Y_i$ is a child of
   $X_i$ in $M$, 
   %Z_i2 = Yi
   then
   %joe4
%{\scriptsize
 $$\begin{array}{ll}
\Pr_M(\psi) = \\ \sum\limits_{\substack{c_{Pa^{X_i}(Y_i)}\in
    \mathcal{T}_{Pa^{X_i}(Y_i)}\\ \forall i\in \{1,\ldots,k\}}}
%joe4
%&
   {\Pr}_M(\psi'  \land \left(\bigwedge_{i} Pa^{X_i}(Y_i)=
   c_{Pa^{X_i}(Y_i)}\right)) \\
%joe4
   %   &
%\hspace{.2in}
   \times  \prod\limits_{i\in \{1,\ldots,k\}} {\Pr}_M( Y_i=1 \mid X_i=1 \land
%joe4: added period
     %     Pa(Y_i)= c_{Pa(Y_i)}]
          Pa^{X_i}(Y_i)= c_{Pa^{X_i}(Y_i)}).
   \end{array}$$.
%   }
\label{thm:arbitrary}
  \end{theorem}
  }
   %sg2
%Theorem~\ref{thm:arbitrary} shows that the probability of a formula containing
 %conjunction of a simple formula with
 %multiple interventions can be computed using only conditional
 %probabilities of simple formulas. 
%joe6*: I don't see how this relates to the theorem.  We don't have a
%DNF in the theorem.  Now we can try to apply the Theorem to each
%disjunct of the DNF, but why should the disjunct have the form
%required of the theorem?  I cut this, and replaced it by the changes above.
\commentout{
  Since any Boolean combination
 of events can be written as a disjunction of 
 conjunctions $\psi' = \bigvee_{i\in \{1,\ldots, t\}} \psi_i'$, where
 every pair $\psi_i', \psi_j'$ for $i\ne j$ are mutually exclusive,
Theorem~\ref{thm:arbitrary} can be extended
 to any arbitrary formula.
  }
  
  }
%joe6*: Two issues: (1) should we put the proof (or at least a sketch)
%of this result in the appendix; (2) can we say anything (on the order
%of one sentence) on this being as far as we can push the result
%(i.e., in general, we'll need non-observational data for more
%complicated formulas).  

%joe2*: Can we make a more general statement about the simplification
%we might hope to get.  This still seems a little too special case-y.

%Mention that exponential in all intervened attribuets
%Divide formula as a conjunction of different interventions and non-interventional

%\section{Proof of Theorem~\ref{thm:pns}\label{sec:pns}}

\commentout{
\section{Proof of Theorem~\ref{thm:arbitrary}}
\begin{proof}
Let $\mathcal{Z} = \mathcal{U}\cup
  \mathcal{V} \setminus \{Z_{j} : j\in \{1,\ldots, t\}\}$. 
$\mathcal{T}_{\mathcal Z}$ has $2^{|\mathcal{Z}|}$ settings.
  For a setting $c \in \mathcal{T}_{\mathcal Z}$, let $c_Z$ be the
  setting of the variable $Z$ in $c$.
  %joe9
  %
  Then
$$\begin{array}{ll}
  \Pr_M(\psi) = \\ \sum_{c \in \mathcal{T}_{\mathcal{Z}}} \Pr_M\biggl(\bigwedge_{j\in \{1,\ldots,t\}} (Z_j = z_j) \land \left(\bigwedge_{i\in \{1,\ldots,k\}} [X_i\leftarrow 1] (Y_i=1)\right)   \land  \bigwedge_{Z \in \mathcal{Z}}(Z = c_Z
  \mid Pa(Z) = c_{Pa(Z)})\biggr).
   \end{array}$$ 
   
   We next compute the probability of 
%joe9
   %   $$\psi_c = {\Pr}_M\biggl(\bigwedge_{j\in \{1,\ldots,t\}} (Z_j =
      $$\psi_c = {\Pr}_M\biggl(\bigwedge_{j\in \{1,\ldots,t_{ij}\}} (Z_j =
   z_j) \land \bigl(\bigwedge_{i\in \{1,\ldots,k\}} [X_i\leftarrow 1]
   (Y_i=1)\bigr)   \land  \bigwedge_{Z \in \mathcal{Z}}(Z = c_Z 
  \mid Pa(Z) = c_{Pa(Z)})\biggr).$$

     From Theorem~\ref{ccce}, it follows that 
in all causal models $M'$ compatible with $M$,
    $${\Pr}_{M'}(\psi_{c}) =
 \sum_{\phi_\alpha \implies \psi_{c} } {\Pr}_{M'}(\phi_\alpha) .$$
Now
$\phi_\alpha \implies \psi_{c}$ holds whenever $\alpha$ contains the
following events:
\begin{enumerate}
\item $(Z_j=z_j  \mid Pa(Z_j) = c_{Pa(Z_j)})  $, for all $j\in \{1,\ldots,t\}$
\item $(Y_i = 1 \mid X_i=1 \land
          Pa^{X_i}(Y_i)= c_{Pa^{X_i}(Y_i)})$, for all $i\in \{1,\ldots,k\}$
  \item $(Z = c_Z \mid Pa(Z) = c_{Pa(Z)})$, for all $Z \in \mathcal{Z}$.
\end{enumerate}

 Let $S_c$ consist of all ccces that contain these events, and
let $\phi_{S_c}$ be the conjunction of the formulas
corresponding to the events in $S$.  Then by Theorem~\ref{ccce}, 
$$\phi_{S_c} \dimp \bigwedge_{\alpha \in S_c} \phi_\alpha.$$
Since the formulas $\phi_\alpha$ for distinct ccces in $S_c$ are
mutually exclusive, we have that 
$\Pr_{M'}(\psi_{c}) = \Pr_{M'}(\phi_S)$.
 Therefore,
 $$\begin{array}{ll}
      &\Pr_{M'}(\psi) \\
      &=  \sum\limits_{c \in \mathcal{T}_{\mathcal{Z}}} \Pr_{M'}(\phi_c)\\
      & = \sum\limits_{c\in \mathcal{T}_{\mathcal{Z}}} \Pr_{M'}\left(\bigwedge_{j\in\{1,\ldots,t\}} (Z_j=z_j)
 \land  \bigwedge_{Z \in \mathcal{Z}} (Z = c_Z )  \land
      \bigwedge_{i\in \{1,\ldots,k\}}  [X_i\leftarrow1, Pa^{X_i}(Y_i)\leftarrow c_{Pa^{X_i}(Y_i)}](Y_i=1 )\right)\\ 
& =\sum\limits_{\substack{c_{Pa^{X_i}(Y_i)}\in
    \mathcal{T}_{Pa^{X_i}(Y_i)}\\ \forall i\in \{1,\ldots,k\}}} \Pr_{M'} \Biggl(\bigwedge_{j\in\{1,\ldots,t\}}Z_j=z_j
 \land  \bigwedge_{i \in \{1,\ldots,k\}} (Pa^{X_i}(Y_i) = c_{Pa^{X_i}(Y_i)} ) \\
 & \hspace{1.55in} \land
      \bigwedge_{i\in \{1,\ldots,k\}}  [X_i\leftarrow1, Pa^{X_i}(Y_i)\leftarrow c_{Pa^{X_i}(Y_i)}](Y_i=1 )\Biggr).\\
  \end{array}$$
 
 If $M'$ is i-compatible with $M$, then we can further conclude that

$$\begin{array}{ll}
 &\Pr_{M'}(\psi) \\
   & =\sum\limits_{\substack{c_{Pa^{X_i}(Y_i)}\in
    \mathcal{T}_{Pa^{X_i}(Y_i)}\\ \forall i\in \{1,\ldots,k\}}} \Biggl( \Pr_{M'} \biggl(\bigwedge_{j\in\{1,\ldots,t\}}Z_j=z_j 
 \land  \bigwedge_{i \in \{1,\ldots,k\}} (Pa^{X_i}(Y_i) = c_{Pa^{X_i}(Y_i)} )\biggr) \\
 & \hspace{1.2in}  \prod\limits_{i\in \{1,\ldots,k\}} \Pr_{M'}(
       [X_i\leftarrow1, Pa^{X_i}(Y_i)\leftarrow c_{Pa^{X_i}(Y_i)}](Y_i=1 )) \biggr)\\
 \end{array}$$

Since $\Pr_{M'}(\psi) = \Pr_M(\psi)$,  we get the desired result.

\end{proof}
}

%sgn2
%\section{Proof of Theorem~\ref{thm:arbitrary}}
%\section{Proof of Theorem~\ref{thm:arbitrary}\label{sec:arbitrary}}
%joe8: this doesn't quite parse
%In this section, we show that an arbitrary formula  $\psi$
%where the sub-formulas $\psi'$ of $\psi$ that involve interventions have
%In this section, we show that we can calculate the probability of an
%arbitrary formula  $\psi$
%joe10: I don't think we need this
%whose subformulas that involve interventions have
%the form $[X \leftarrow x]\psi''$
%, where the only variables that appear in $\psi''$
%are   descendants of  X.  
%joe8: cut this; we already said it in the main text.
%(Note that we are intervening only on a single variable X
%here, not a set of variables,)
%can be calculated in terms of conditional probabilities which
%in terms of conditional probabilities that
%joe8
%can be estimated from observational data. To show this result,
%can be estimated from observational data.
 To prove this result,
we first convert $\psi$ to an equivalent
formula in a canonical form.
Specifically, it has the form $\psi_1 \lor \cdots \lor
\psi_m$, where the $\psi_i$s are mutually exclusive and 
each $\psi_i$ is a conjunction of the form $\psi_{i0}
%joe9*: the number of conjuncts can depend on i
%\land \cdots \land \psi_{im}$, where $\psi_{i0}$ is a simple formula
\land \cdots \land \psi_{ir_i}$, where $\psi_{i0}$ is a simple formula
%joe9
%  and for $1 \le j \le m$, $\psi_{ij}$ is an intervention formula of the
  and for $1 \le j \le r_i$, $\psi_{ij}$ is a formula of the
%joe9: now the number of conjuncts can depend on i and j
  %  form $[X_{j} \gets x_{j}] (\bigwedge_{k\in \{1,\ldots,t\}}
%joe10
  %  form $[X_{j} \gets x_{j}] (\bigwedge_{k\in \{1,\ldots,t_{ij}\}}
    form $[\vec{X}_{j} \gets \vec{x}_{j}] (\bigwedge_{k\in \{1,\ldots,t_{ij}\}}
  Y_{ijk}=y_{ijk})$, and the interventions are all 
distinct.  This conversion just involves standard propositional
reasoning and two properties which hold under the semantics
described in Section~\ref{sec:semantics}. The
first is that $[{Y} \gets {y}]\phi
\land [{Y} \gets {y}]\phi'$ is equivalent to 
$[{Y} \gets {y}](\phi \land \phi')$.  The second is that
$\neg [{Y} \gets {y}]\phi$ is equivalent to
$[{Y} \gets {y}]\neg \phi$. 

%joe8
%Ignore for now the requirement that the disjuncts be mutually exclusive, 
Ignore for now the requirements that the disjuncts be mutually exclusive, 
that all interventions be distinct, and that there be no leading
%joe9
%negations in intervention formulas.  Using standard propositional
formulas involving interventions.  Using standard propositional
reasoning, we can transform a formula $\phi$ to an equivalent formula
in DNF, where the literals are either simple formulas or
%joe9
%intervention formulas .  Of course, the disjuncts may not be mutually
intervention formulas (i.e., formulas of the form $[X \gets x]\phi$).
Of course, the disjuncts may not be mutually 
exclusive.  Again, using straightforward propositional reasoning, we
can convert the formula to a DNF where the disjuncts are mutually
exclusive.  Rather than writing out the tedious details, we give an
example.  Consider a formula of the form $(\phi_1 \land \phi_2) \lor
(\phi_3 \land \phi_4)$.  This is propositionally equivalent to
$$\begin{array}{l}
(\phi_1 \land \phi_2 \land \phi_3 \land \phi_4) \lor
(\phi_1 \land \phi_2 \land \neg \phi_3 \land \phi_4) \lor
(\phi_1 \land \phi_2 \land \phi_3 \land \neg \phi_4) \lor
(\phi_1 \land \phi_2 \land \neg \phi_3 \land \neg \phi_4)\\ \lor
(\neg \phi_1 \land \phi_2 \land \phi_3 \land \phi_4) \lor
(\phi_1 \land \neg \phi_2 \land \phi_3 \land \phi_4) \lor
%joe8
  %  (\neg \phi_1 \land \neg \phi_2 \land \phi_3 \land \phi_4);
     (\neg \phi_1 \land \neg \phi_2 \land \phi_3 \land \phi_4).
  \end{array}$$
%joe8: it's not clear what the "right form" is.  I think we can safely
%cut this sentence
%Moreover, this formula is in the right form.
We can now  apply the two equivalences
mentioned above to remove leading negations from intervention
formulas and to ensure that, in each disjunct, all interventions are
distinct.  These transformations maintain the fact that the disjuncts
are mutually exclusive.

Since the disjuncts in
$\psi$ are mutually exclusive,
the probability of $\psi$ is the  sum of the probabilities of the
%Joe8
%disjuncts $\Pr(\psi) = \sum_{i\in \{1,\ldots,l\}} \Pr(\psi_i).$
disjuncts; that is, $\Pr(\psi) = \sum_{i\in \{1,\ldots,m\}} \Pr(\psi_i).$
To compute the probability of a disjunct  $\psi_i$,
we first simplify it using the following two observations.
First,  if a    formula involves an intervention
$[X\leftarrow x]$ on some variable $X$  such that $X$ is also set to
  $x$ in the formula, 
%joe10: made the example more complicated, to bring out the point
%such as $(X=0 \land [X\leftarrow 0](Y=1))$,
such as $(X=x \land [X\leftarrow x, Z \leftarrow z](Y=1))$,
the intervention $X\leftarrow x$
is redundant and can be dropped; for example,
%joe10
%$X=0 \land [X\leftarrow 0](Y=1)$ is equivalent to $X=0 \land Y=1$.
$X=0 \land [X\leftarrow 0, Z \leftarrow 1](Y=1)$ is equivalent to $X=0
\land [Z \leftarrow 1](Y=1)$. 
Second, if an intervention formula does not contain a descendant of the
%joe10
%intervened variable such as $\psi = [X\gets x] (\psi_1\land \psi_2)$, where 
%all variables in  $\psi_1$ are non-descendants of $X$, then
intervened variables, such as $\psi = [\vec{X}\gets \vec{x}]
(\psi_1\land \psi_2)$, where  
all variables in  $\psi_1$ are non-descendants of the variables in
$\vec{X}$, then 
%joe10
%the non-descendants $\psi_1$
the variables in $\psi_1$
%joe8
%are not affected by the intervention and can be separated, i.e.
%joe10
%are not affected by the intervention and can be separated out; that is,
are not affected by the intervention, so $\psi_1$ can be pulled out of
the scope of the intervention; that is,
$\psi$ is equivalent to $\psi_1 \land [X\gets x] (\psi_2)$.
%joe8
%Using these observations, we remove all interventions which are redundant
Using these observations, we remove all interventions that are redundant
%joe10
%and separate out all non-descendants
%of the intervened variable from the intervention formula.
and pull formulas involving only non-descendants 
of the intervened variables out of the intervention formula.
%joe10*: I remember worrying about what happened in the case of
%allowing interventions on multiple variables if you had some overlap
%between two interventions, but they weren't identical.  In that case,
%we can't combine interventions.  I don't think it's a problem, but I
%wanted to flag it. 

  After this simplification,  without loss of generality,
  the disjunct $\psi_i$ is a conjunction of formulas 
%joe9
%  $\psi_{i0} \land \left(\bigwedge_{j\in \{1,\ldots,r\}} \psi_{ij}\right)$,
%  where $\psi_{i0} = \left(\bigwedge_{j\in\{1,\ldots,s\}} (Z_{ij}=
  $\psi_{i0} \land \left(\bigwedge_{j\in \{1,\ldots,r_i\}} \psi_{ij}\right)$,
  where $\psi_{i0} = \left(\bigwedge_{j\in\{1,\ldots,s_i\}} (Z_{ij}=
  z_{ij})\right)$ is a simple 
 formula (with no intervention), and
$\psi_{ij}$ for $j>0$ has the form
%joe9
 % $[X_j\leftarrow x_j] (\bigwedge_{k\in \{1,\ldots,t\}}
%joe10
 % $[X_j\leftarrow x_j] (\bigwedge_{k\in \{1,\ldots,t_{ij}\}}
  $[\vec{X}_j\leftarrow \vec{x}_j] (\bigwedge_{k\in \{1,\ldots,t_{ij}\}}
 Y_{ijk}=y_{ijk})$, where $Y_{ijk}$ is a descendant of 
%joe10
 % $X_j$ in $M$. The following theorem
some variable in  $\vec{X}_j$ in $M$. The following theorem
   proves the result for $\psi_i$, which completes the proof.

    \begin{theorem}
      If $M$ is a CBN and
      %joe9 
% $\psi_i = \psi_{i0} \land \left(\bigwedge_{j\in \{1,\ldots,r\}}
 $\psi_i = \psi_{i0} \land \left(\bigwedge_{j\in \{1,\ldots,r_i\}}
      \psi_{ij}\right)$, 
   where $\psi_i$ contains no redundant interventions,
%joe9
   %   $\psi_{i0} = \left(\bigwedge_{j\in\{1,\ldots,s\}} (Z_{ij} =
      $\psi_{i0} = \left(\bigwedge_{j\in\{1,\ldots,s_i\}} (Z_{ij} =
   z_{ij})\right)$ is a simple 
 formula (with no interventions), and
$\psi_{ij}$ for $j>0$ has the form
%joe9*: why is there just a single conjunct here?  I aasumw ir'a  RYPO
% $[X_j\leftarrow x_j] (Y_{ijk}=Y_{ijk})$, where $Y_{ijk}$ is a descendant of 
%joe10
 % $[X_j\leftarrow x_j] (\bigwedge_{k\in\{1,\ldots,
 $[\vec{X}_j\leftarrow \vec{x}
_j] (\bigwedge_{k\in\{1,\ldots,
   t_{ij}\}}Y_{ijk}=y_{ijk})\biggr))$, where $Y_{ijk}$ is a descendant
 of 
%joe10
 % $X_j$ in $M$,
some variable in $\vec{X}_j$ in $M$, 
%joe8
 % then $\Pr(\psi_i)$ can be computed by determining probability
  then $\Pr(\psi_i)$ can be computed by determining the probability
   of formulas that do not involve an intervention.
\label{thm:arbitrarydesc}
  \end{theorem}
    \begin{proof}
    The proof proceeds along lines very similar to the proof of Theorem~\ref{thm:pns}.
Let $\mathcal{Z} = \mathcal{U}\cup
  \mathcal{V} \setminus \{Z_{ij} : j\in \{1,\ldots, s\}\}$. 
$\mathcal{T}_{\mathcal Z}$ has $2^{|\mathcal{Z}|}$ settings.
  For a setting $c \in \mathcal{T}_{\mathcal Z}$, let $c_Z$ be the
  setting of the variable $Z$ in $c$.
  %joe10
  %
  Then
$$\begin{array}{ll}
  \Pr_M(\psi_i) = \\ \sum\limits_{c \in \mathcal{T}_{\mathcal{Z}}}
%joe9
%  \Pr_M\biggl(\bigwedge\limits_{j\in \{1,\ldots,s\}} (Z_{ij} = z_{ij})
  %  \land \left(\bigwedge\limits_{j\in \{1,\ldots,r\}} [X_{j}\leftarrow
    \Pr_M\biggl(\bigwedge\limits_{j\in \{1,\ldots,s_i\}} (Z_{ij} = z_{ij})
%joe10
%    \land \left(\bigwedge\limits_{j\in \{1,\ldots,r_i\}} [X_{j}\leftarrow
    %    x_{j}] (\bigwedge\limits_{k\in\{1,\ldots,
        \land \left(\bigwedge\limits_{j\in \{1,\ldots,r_i\}}
   [\vec{X}_{j}\leftarrow \vec{x}_{j}] (\bigwedge\limits_{k\in\{1,\ldots,
%joe9
%    t\}}Y_{ijk}=y_{ijk})\right)   \land  \bigwedge\limits_{Z \in
    t_{ij}\}}Y_{ijk}=y_{ijk})\right)   \land  \bigwedge\limits_{Z \in
%joe9*: I'm confused here, Sainyam.  Don't you also need to multiply
%this final term by \Pr(Pa(Z) = c_{Pa(Z)}).  If so, this change will
%need to be made several times below.
    \mathcal{Z}}(Z = c_Z)\biggr).
   \end{array}$$ 
   
   We next compute the probability of 
%joe9
%   $$\psi_{ic} = {\Pr}_M\biggl(\bigwedge_{j\in \{1,\ldots,s\}} (Z_{ij}
   %   = z_{ij}) \land \biggl(\bigwedge_{j\in \{1,\ldots,r\}}
      $$\psi_{ic} = {\Pr}_M\biggl(\bigwedge_{j\in \{1,\ldots,s_i\}} (Z_{ij}
   = z_{ij}) \land \biggl(\bigwedge_{j\in \{1,\ldots,r_i\}}
%joe10
   %   [X_{j}\leftarrow x_{j}] (\bigwedge_{k\in\{1,\ldots,
      [\vec{X}_{j}\leftarrow \vec{x}_{j}] (\bigwedge_{k\in\{1,\ldots,
%joe9
     %     t\}}Y_{ijk}=y_{ijk})\biggr)   \land  \bigwedge_{Z \in
          t_{ij}\}}Y_{ijk}=y_{ijk})\biggr)   \land  \bigwedge_{Z \in
     \mathcal{Z}}(Z = c_Z )\biggr).$$

     From Theorem~\ref{ccce}, it follows that 
in all causal models $M'$ compatible with $M$,
    $${\Pr}_{M'}(\psi_{ic}) =
 \sum_{\phi_\alpha \implies \psi_{ic} } {\Pr}_{M'}(\phi_\alpha) .$$
 %joe10*:
 \commentout{
 For now, for ease of exposition, we assume
 that all interventions $[\vec{X}_j \leftarrow \vec{x}_j]$ in $\psi_{ic}$
 involve only a single variable; that is, $\vec{X}_j$ is a singleton.
After we give the proof in this case, we explain how to deal with the
general case.}
 Now
$\phi_\alpha \implies \psi_{ic}$ holds whenever $\alpha$ contains the
following events:
\begin{enumerate}
\item $(Z_{ij}=z_{ij}  \mid Pa(Z_{ij}) = c_{Pa(Z_{ij})})  $, for all
%joe12
  %  $j\in \{1,\ldots,s_i\}$
%  \item $(Z = c_Z \mid Pa(Z) = c_{Pa(Z)})$, for all $Z \in \mathcal{Z}$
    $j\in \{1,\ldots,s_i\}$; 
\item $(Z = c_Z \mid Pa(Z) = c_{Pa(Z)})$, for all $Z \in \mathcal{Z};$
%joe8*: Should this be an item in the list of events that $\alpha$
%contains?  I suspect not; that is, we want to put \end{enumerate}
%here.  Please check
%sg4: It should be because  $\alpha$ must contain $(X = c^j_{X}  | Pa(X) = c^j_{Pa(X) })$
% which requires new notation to be defined.

   %joe9*: rewrote.  I found the current version hard to parse.  This
   %is supposed to be a list of events in \alpha, and it doesn't read
   %that way. Please check.
% \item We now describe the events that capture the effects of an
% intervention $X_j\gets x_j$. 
%joe8: having the set here at the beginning looks strange.  If it's a
%topological ordering, X_j and X_0' should be ordered somehow.
%joe9
   % \item $(X = c^j_{X}  | Pa(X) = c^j_{Pa(X) })$
    \item $(X = c^j_{X}  \mid Pa(X) = c^j_{Pa(X) })$,
  for all $X\in\mathcal{X}_j'$, where 
  $\mathcal{X}_j'$ consists of all descendants of  
%joe12
%  the intervened variables in $\vec{X}_j$ while excluding the variables
  the intervened variables in $\vec{X}_j$ other than the variables
  in $\vec{X}_j$
  and $c^j\in\mathcal{T}_j'$,
   the set of settings of the
variables in $\mathcal{U}\cup \mathcal{V}$, where the following
variables  are
fixed as follows: 
   \begin{enumerate}
%joe12: just adding commas
     %   \item $\vec{X}_j = \vec{x}_j$
        \item $\vec{X}_j = \vec{x}_j$,
%joe9
     %   \item $Y_{ijk} = y_{ijk}$ for all $k\in \{1,\ldots, t\}$
%joe12
          %\item $Y_{ijk} = y_{ijk}$ for all $k\in \{1,\ldots, t_{ij}\}$
          \item $Y_{ijk} = y_{ijk}$ for all $k\in \{1,\ldots, t_{ij}\}$,
%joe12:
%   \item $Z_{ik}=z_{ik}$ for all  $k\in \{1,\ldots,s_i\}$, which are neither
%     in $\vec{X}_j$  nor their descendants $\mathcal{X}_j'$, i.e.,
%     $Z_{ik}\notin(\vec{X}_j \cup \mathcal{X}_j')$
          \item $Z_{ik}=z_{ik}$ for $Z_{ik}\notin(\vec{X}_j \cup 
\mathcal{X}_j')$, $k\in \{1,\ldots,s_i\}$, 
%joe9
     %     \mathcal{X}_j'')$, $k\in \{1,\ldots,s\}$
             \item   $Z = c_Z$ for all $Z\in \mathcal{Z}$ and $Z\notin (\vec{X}_j \cup \mathcal{X}_j')$.
   \end{enumerate}
%joe12*: I don't understand what it means for an event to "hold"; I
%just cut this.
%   These events hold for every value of $j\in \{1,\ldots, r_i\}$.

\commentout{   
   
    \item $(X = c^j_{X}  \mid Pa(X) = c^j_{Pa(X) })$,
  for all $X\in\mathcal{X}_j'$, where 
  %joe9*: I'm really lost here, Sainyam.  I'm not sure what you intended.
%For one thing, I don't know what \mathcal{X}_j' is.  Moreover,
% the union of a singleton and an ordering seems strange.
%Finally, as near as I can tell, you never make use of the details of
%the ordering.  I put down what I thought you meant.  Please check carefully
%Let $\{X_j\}\cup\mathcal{X}_0'\prec \mathcal{X}_1'\prec \ldots \prec
%\mathcal{X}_l'$ 
%  be the topological ordering of all variables $Y_{ijk}$ for all $k$
%joe8
  %  and all variables which are
%joe9*: Sainyam, as I said, I'm quite confused here. If we're trying
%to capture the effects of intervening on X_j, why doesn't it suffice
  %to consider all the descendants of X_j?  That said, what happens if
  %some of the Z_{ij}'s are descendants of X_j  or if Z is?  We should
%add a more explanation as to what's going on after the \end{enumerate}.
%Can you add a short escription of \mathcal{X}_j'' here, and a longer
%one after \end{enumerate}.  I'm hoping that the definition doesn't
%involve the topological sort.  If it does, that too needs more explanation.
  $\mathcal{X}_j'$ consists of descendants of $X_j$,
  and $c^j\in\mathcal{T}_j'$,
  %          and ancestors of any of the variables in $\{Y_{ijk}: k\in
%          %          \{1,\ldots,t\}\}$ .
%           \{1,\ldots,t\}\}$ in the topological sort 
%          
%  Let $\mathcal{X}_j''=  \cup_{i\in\{1,\ldots,l\}} \mathcal{X}_i'$
%  and 
%joe8
  %  $\mathcal{T}_j'$    be the set of all possible settings of
%  variables in $\mathcal{U}\cup \mathcal{V}$, where    only the variables in
%  $\mathcal{X}_j'\setminus \{Y_{ijk} :  k\in\{1,\ldots,t\}\}$ are varied. 
%   The other variables are fixed as follows:
%joe10
  %  is a setting of the
    the set of settings of the
variables in $\mathcal{U}\cup \mathcal{V}$, where the following
%joe9
%except
%  $\mathcal{X}_j'' \setminus \{Y_{ijk} :  k\in\{1,\ldots,t\}\}$ are
variables  are
fixed as follows: 
   \begin{enumerate}
   \item $X_j = x_j$
%joe9
     %   \item $Y_{ijk} = y_{ijk}$ for all $k\in \{1,\ldots, t\}$
        \item $Y_{ijk} = y_{ijk}$ for all $k\in \{1,\ldots, t_{ij}\}$
   \item $Z_{ik}=z_{ik}$ for all  $k\in \{1,\ldots,s_i\}$, which are neither
    the same as $X_j$
   nor the descendants of 
   $X_j$, i.e., $Z_{ik}\notin(\{X_j\} \cup
%joe9
     %     \mathcal{X}_j'')$, $k\in \{1,\ldots,s\}$
          \mathcal{X}_j')$
   \item   $Z = c_Z$ for all $Z\in \mathcal{Z}$ and $Z\notin (\mathcal{X}_j'\cup\{X_j\})$.
   \end{enumerate}
   These events hold for every value of $j\in \{1,\ldots, r_i\}$.
   }
   %joe9
   \end{enumerate}
   %joe8
%   Therefore $|\mathcal{T}'| = 2^{|\mathcal{X}_j'| - t}$. Intuitively,
    %Note  that $|\mathcal{T}'| = 2^{|\mathcal{X}_j''| - t_{ij}}$.
    Intuitively,
   $\mathcal{T}_j'$  captures all possible post-intervention
%joe8
   %   settings of all variables which are the descendants of $X_j$
%   and at the same level or ancestors of $Y_{ijk}$.
   settings of all variables that are descendants of $\vec{X}_j$,
   while fixing $Y_{ijk}$s as $y_{ijk}$. 
%joe10
%   The third set of events $(X = c^j_{X}  | Pa(X) = c^j_{Pa(X) })$
%   for all $X\in\mathcal{X}_j'$
   By fixing the third set of events, $(X = c^j_{X}  | Pa(X) = c^j_{Pa(X) })$
   for all $X\in\mathcal{X}_j'$, we 
   ensure that all events involving descendants of $\vec{X}_j$
   are consistent with respect to one of the post-intervention
   settings $c^j \in \mathcal{T}_j'$. These events represent
   the effects of interventions in
%joe10
%   $X_j\leftarrow x_j$ on its descendants for all values of $j \in
   %   \{1,\ldots, r_i\}$.
      $\vec{X}_j\leftarrow \vec{x}_j$ on its descendants.
      For example, consider a causal graph as shown below
       and
   $\psi_i = [X_1\gets1, X_3\gets 1](Y=1)$. 
            
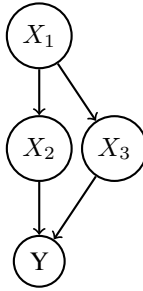
\begin{center}
  \vspace{-4mm}
\begin{tikzpicture}
  \tikzstyle{var} = [circle, draw, thick]
  % Define nodes with the style
  \node[var] (X1) at (0,0) {$X_1$};
  \node[var] (X2) at (0,-1.5) {$X_2$};
    \node[var] (X3) at (1,-1.5) {$X_3$};
    \node[var] (Y) at (0,-3) {Y};

  \draw[->, thick] (X1) -- (X2);
    \draw[->, thick] (X2) -- (Y);
      \draw[->, thick] (X3) -- (Y);
          \draw[->, thick] (X1) -- (X3);
\end{tikzpicture}
\vspace{-5mm}
\end{center}
 In this case, 
   $\vec{X}_j = \{X_1\gets1, X_3\gets 1\} $.
%joe12
% Following the conditions mentioned above, $\alpha$ must contain
By the conditions mentioned above, $\alpha$ must contain
 one of the two
   events $((Y=1 \mid X_2=0, X_3=1) \land (X_2=0 \mid X_1=1))$ 
%joe12
   %   or $((Y=1 \mid X_2=1, X_3=1) \land (X_2=1 \mid X_1=1))$
      or $((Y=1 \mid X_2=1, X_3=1) \land (X_2=1 \mid X_1=1))$,
   because
   $\mathcal{T}_1' =\{\{X_1=1, X_2=0, X_3=1, Y = 1\}, \{X_1=1, X_2=1, X_3=1, Y=1\}\}$.
%joe12
%  This condition ensures that whenever $X_1 = 1$ and $X_3=1$,
   %  $\phi_\alpha$ ensures that
%     $Y=1$. It is easy to see that whenever $\alpha$ does not contain any
  This condition ensures that if $X_1 = 1$ and $X_3=1$, then
   $\phi_\alpha$ implies
  $Y=1$. It is easy to see that if $\alpha$ does not contain either
  of these 
  two events,  then it must contain $((Y=0 \mid X_2=0, X_3=1) \land (X_2=0 \mid X_1=1))$ 
   or $((Y=0 \mid X_2=1, X_3=1) \land (X_2=1 \mid X_1=1))$, in which
%joe12
   %   case  $\phi_\alpha \centernot\implies \psi_i$.
      case  $\phi_\alpha$ does not imply  $\psi_i$.

      \commentout{
    For example, consider a causal graph $X_1\rightarrow X_2\rightarrow Y$ and
   $\psi_i = [X_1\gets1](Y=1)$.  In this case, $j=1$, $\mathcal{X}_1'
%joe10
    %    = \{X_2,Y\}$ and
        = \{X_2,Y\}$ and,  
   following the conditions mentioned above, $\alpha$ must contain
   one of the two
   events $(Y=1 \mid X_2=0 \land X_2=0 \mid X_1=1)$ 
   or $(Y=1 \mid X_2=1 \land X_2=1 \mid X_1=1)$
   because
   $\mathcal{T}_1' =\{\{X_1=1, X_2=0, Y = 1\}, \{X_1=1, X_2=1, Y=1\}\}$.
  This condition ensures that whenever $X_1 = 1$, $\phi_\alpha$ ensures that
  $Y=1$. It is easy to see that whenever $\alpha$ does not contain any of these
  two events,  then it must contain $(Y=0 \mid X_2=0 \land X_2=0 \mid X_1=1)$ 
   or $(Y=0 \mid X_2=1 \land X_2=1 \mid X_1=1)$, in which case 
   $\phi_\alpha \centernot\implies \psi_i$.
   }
   %joe8: removed paragraph break
   %
%It follows that $\alpha$ must contain $(X = c^j_{X}  | Pa(X) = c^j_{Pa(X) })$
%for all $X\in\mathcal{X}_j''$ and $c^j\in\mathcal{T}_j'$.
%joe9          
%\end{enumerate}

 Let $S_c$ consist of all ccces that contain these events, and
let $\phi_{S_c}$ be the conjunction of the formulas
corresponding to the events in $S$. 
%joe10
%This means,
Thus,
$$ \begin{array}{ll}
\phi_{S_c} = & \biggl(\bigwedge\limits_{j'\in \{1,\ldots,s_i\}}
[  Pa(Z_{ij'}) \gets c_{Pa(Z_{ij'})}](Z_{ij'}=z_{ij'})\biggr)  
\land 
\biggl(\bigwedge\limits_{Z\in \mathcal{Z}} [  Pa(Z) \gets c_{Pa(Z)}](Z=c_Z)  \biggr)  \\
& \land  \bigwedge\limits_{j\in\{1,\ldots,r_i\}}\Biggl(\bigvee\limits_{c^j\in \mathcal{T}_j'}
 \biggl( \bigwedge\limits_{X\in \mathcal{X}_j'} 
 [  Pa(X) \gets c^j_{Pa(X)}](X=c^j_X)  \biggr)\Biggr)\\
\hspace{0.24in}=& \biggl(\bigwedge\limits_{j'\in \{1,\ldots,s_i\}} 
[  Pa(Z_{ij'}) \gets c_{Pa(Z_{ij'})}](Z_{ij'}=z_{ij'})\biggr) 
 \land 
\biggl(\bigwedge\limits_{Z\in \mathcal{Z}}
 [  Pa(Z) \gets c_{Pa(Z)}](Z=c_Z)  \biggr)  \\
& \land  \Biggl(\bigvee\limits_{\{c^j\in \mathcal{T}_j'~: ~j\in \{1,\ldots, r_i\}\}} 
\biggl( \bigwedge\limits_{X\in \mathcal{X}_l',~
%joe10
%  l\in\{1,\ldots,r_i\}} [  Pa(X) \gets c^l_{Pa(X)}](X=c^l_X)  \biggr)\Biggr)
  l\in\{1,\ldots,r_i\}} [  Pa(X) \gets c^l_{Pa(X)}](X=c^l_X)  \biggr)\Biggr).
\end{array}
 $$
Then by Theorem~\ref{ccce},
% \sg5:shouldnt this be \bigvee?
%joe10: yet again, yes
%$$\phi_{S_c} \dimp \bigwedge_{\alpha \in S_c} \phi_\alpha.$$
$$\phi_{S_c} \dimp \bigvee_{\alpha \in S_c} \phi_\alpha.$$
Since the formulas $\phi_\alpha$ for distinct ccces in $S_c$ are
mutually exclusive, we have that 
$\Pr_{M'}(\psi_{ic}) = \Pr_{M'}(\phi_S)$.
 Therefore,
 \setlength\parindent{0pt}   
\setlength\arraycolsep{0pt}

 $$\begin{array}{ll}
      &\Pr_{M'}(\psi_i) \\
      &=  \sum\limits_{c \in \mathcal{T}_{\mathcal{Z}}} \Pr_{M'}(\psi_{ic})\\
      & = \sum\limits_{\substack{c\in
%joe9*: I don't think you need to sum over all j here, since you're
%not splitting the conunction.  Please check
%       \mathcal{T}_{\mathcal{Z}},\\ c^j\in \mathcal{T}_j',~ \forall
       %       j\in\{1,\ldots,s\}}} \Pr_{M'}\left(\bigwedge_{j\in
%     \{1,\ldots,s\}} (Z_{ij} = z_{ij}) 
       \mathcal{T}_{\mathcal{Z}},\\ c^j\in \mathcal{T}_j'~:~ j\in \{1,\ldots,r_i\}}}\\     
&  \Pr_{M'}\left(
      \bigwedge\limits_{j'\in  \{1,\ldots,s_i\}} (Z_{ij'} = z_{ij'})
 \land  \bigwedge\limits_{Z \in \mathcal{Z}} (Z = c_Z )  \land
%joe9
% \bigwedge_{j\in \{1,\ldots,r\}, X\in \mathcal{X}_j''} [Pa(X)
 \bigwedge\limits_{ X\in \mathcal{X}_l', ~l\in \{1,\ldots,r_i\}} [Pa(X)
 %joe8
        %        \leftarrow c^j_{Pa(X)}] (X= c^j_{X})\right)\\
                \leftarrow c^l_{Pa(X)}] (X= c^l_{X})\right).
  \end{array}$$
 
%joe12*: I don't see how the law of total probability applies here
% We can further simplify the expression by using the law of total probability.  
 We can further simplify this expression.
Specifically, we can get rid of $[Pa(X) \leftarrow c^j_{Pa(X)}] (X= c^j_{X})$
 %joe12*: I"m confused by the next sentence, Sainyam.  ,
%it talks about descendants of Y_{ijk}, and Y_{ijk} is nowhere in
%sight here.  Did you mean that X is a descendant of Y_{ijk}?  I
%rewrote it below baed on this assumption. Please check.
%for all descendants of $Y_{ijk}$ for all  $k\in \{1,\ldots,t_{ij}\}$
for all descendants $X$ of some $Y_{ijk}$ with  $k\in \{1,\ldots,t_{ij}\}$
and $j\in \{1,\ldots,s_i\}$. 
 We leave the 
details to the reader.
%joe9
 % In the expression $\bigwedge_{j\in \{1,\ldots,r\}, X\in
 The expression above
   may be infeasible for some combinations of settings $c\in \mathcal{T}_\mathcal{Z}$
 and $c^l$ for all $l\in \{1,\ldots,r_i\}$.
For example $[X\gets 1](Y=0) \land [X\gets 1](Y=1)$ has zero probability.
Furthermore, certain formulas in $\bigwedge_{X\in
   \mathcal{X}_l', l\in \{1,\ldots,r_i\}} [Pa(X) \leftarrow c^l_{Pa(X)}] (X= c^l_{X})$
 may be duplicates,
 and some interventions may be redundant.
We need to drop the duplicates and redundant interventions
before further simplifying the expression. 
%joe8*: What formulas are you referring to when you "such formulas".
%Did you mean "we assume that all conjuncts in this formula are distinct"?
%sg4: yes, incorporated.
%joe10
%For simplicity of exposition, we assume that the expression is feasible,
For ease of exposition, we assume that the expression is feasible,
all conjuncts in $\bigwedge_{X\in
   \mathcal{X}_l', l\in \{1,\ldots,r_i\}} [Pa(X) \leftarrow c^l_{Pa(X)}] (X= c^l_{X})$
%joe10
%are distinct
are distinct,
and all interventions are non-redundant.
 
 If $M'$ is i-compatible with $M$, then we can further conclude that
 %joe9: removed break
 %
 
\setlength\parindent{0pt}   
\setlength\arraycolsep{0pt}

$$\begin{array}{ll}
 &\Pr_{M'}(\psi_i) \\
%joe9*: as above
%   &= \sum\limits_{\substack{c\in \mathcal{T}_{\mathcal{Z}},\\ c^j\in
   %       \mathcal{T}_j',~ \forall j\in\{1,\ldots,s\}}}
      &= \sum\limits_{\substack{c\in  \mathcal{T}_{\mathcal{Z}},\\ 
      c^j\in \mathcal{T}_j': j\in \{1,\ldots,r_i\}}}\\
%joe9
   %   \Pr_{M'}\left(\bigwedge_{j\in \{1,\ldots,s\}} (Z_{ij} = z_{ij})
 &     \Pr_{M'}\left(\bigwedge\limits_{j'\in \{1,\ldots,s_i\}} (Z_{ij'} = z_{ij'}) 
 \land  \bigwedge\limits_{Z \in \mathcal{Z}} (Z = c_Z ) \right)\Pr_{M'}\biggl(
%joe9
% \bigwedge_{j\in \{1,\ldots,r\}, X\in \mathcal{X}_j''} [Pa(X)
 \bigwedge\limits_{\substack{X\in \mathcal{X}_l', \\ l\in \{1,\ldots,r_i\}}} [Pa(X)
 \leftarrow c^l_{Pa(X)}] (X =  c^l_{X})\biggr)\\  
%joe9*: yet again
% &= \sum\limits_{\substack{c\in \mathcal{T}_{\mathcal{Z}},\\ c^j\in
%          \mathcal{T}_j',~ \forall j\in\{1,\ldots,s\}}}
 %      \Pr_{M'}\left(\bigwedge_{j\in \{1,\ldots,s\}} (Z_{ij} = z_{ij})
 &= \sum\limits_{\substack{c\in \mathcal{T}_{\mathcal{Z}}, \\
       c^j\in \mathcal{T}_j': j\in \{1,\ldots,r_i\}}}\\
       &    \Pr_{M'}\left(\bigwedge\limits_{j'\in \{1,\ldots,s_i\}} (Z_{ij'} = z_{ij'}) 
  \land  \bigwedge\limits_{Z \in \mathcal{Z}} (Z = c_Z ) \right) \prod\limits_{\substack{ X\in \mathcal{X}_l',\\
  l\in
%joe9
    %    \{1,\ldots,r\}, X\in \mathcal{X}_j''} \Pr_{M'}
        \{1,\ldots,r_i\}}} \Pr_{M'}  
%joe12
  %  (X= c^l_{X} \mid Pa(X) \leftarrow c^l_{Pa(X)} )\\
    (X= c^l_{X} \mid Pa(X) \leftarrow c^l_{Pa(X)} ).
 \end{array}$$

 Since $\Pr_{M'}(\psi_i) = \Pr_M(\psi_i)$,  we get the desired result.
\commentout{
 %joe10*:
 We now briefly comment on what happens if the interventions are not
 on single variables. If $\psi_i$ contains multiple
 interventions, $[\vec{X}_j \leftarrow \vec{x}_j]$,
the same proof extends with a modification in the events that capture post-intervention
settings of the descendants of $X_j$. These events need to capture post-intervention
setting of any of the descendants of $\vec{X}_j$.
Specifically, $\phi_\alpha \implies \psi_{ic}$ whenever $\alpha$ contains the following
 sets of events (the first two sets are the same as above):
\begin{enumerate}
\item $(Z_{ij}=z_{ij}  \mid Pa(Z_{ij}) = c_{Pa(Z_{ij})})  $, for all $j\in \{1,\ldots,s_i\}$
 \item $(Z = c_Z \mid Pa(Z) = c_{Pa(Z)})$, for all $Z \in \mathcal{Z}$
 \item $(X = c^j_{X}  \mid Pa(X) = c^j_{Pa(X) })$,
  for all $X\in\mathcal{X}_j'$, where 
  $\mathcal{X}_j'$ consists of all descendants of  
 the intervened variables in $\vec{X}_j$ while excluding the variables in $\vec{X}_j$
  and $c^j\in\mathcal{T}_j'$,
   the set of settings of the
variables in $\mathcal{U}\cup \mathcal{V}$, where the following
variables  are
fixed as follows: 
   \begin{enumerate}
   \item $\vec{X}_j = \vec{x}_j$
%joe9
     %   \item $Y_{ijk} = y_{ijk}$ for all $k\in \{1,\ldots, t\}$
        \item $Y_{ijk} = y_{ijk}$ for all $k\in \{1,\ldots, t_{ij}\}$
   \item $Z_{ik}=z_{ik}$ for all  $k\in \{1,\ldots,s_i\}$, which are neither
    in $\vec{X}_j$
   nor their descendants
   $\mathcal{X}_j'$, i.e., $Z_{ik}\notin(\vec{X}_j \cup
%joe9
     %     \mathcal{X}_j'')$, $k\in \{1,\ldots,s\}$
          \mathcal{X}_j')$
   \item   $Z = c_Z$ for all $Z\in \mathcal{Z}$ and $Z\notin (\vec{X}_j \cup \mathcal{X}_j')$.
   \end{enumerate}
   These events hold for every value of $j\in \{1,\ldots, r_i\}$.
   \end{enumerate}
      \begin{center}
  \vspace{-4mm}
\begin{tikzpicture}
  \tikzstyle{var} = [circle, draw, thick]
  % Define nodes with the style
  \node[var] (X1) at (0,0) {$X_1$};
  \node[var] (X2) at (0,-1.5) {$X_2$};
    \node[var] (X3) at (1,-1.5) {$X_3$};
    \node[var] (Y) at (0,-3) {Y};

  \draw[->, thick] (X1) -- (X2);
    \draw[->, thick] (X2) -- (Y);
      \draw[->, thick] (X3) -- (Y);
          \draw[->, thick] (X1) -- (X3);
\end{tikzpicture}
\vspace{-5mm}
\end{center}
     For example, consider a causal graph as shown above and
   $\psi_i = [X_1\gets1, X_3\gets 1](Y=1)$.  In this case, 
   $\vec{X}_j = \{X_1\gets1, X_3\gets 1\} $.
   Following the conditions mentioned above, $\alpha$ must contain
   one of the two
   events $((Y=1 \mid X_2=0, X_3=1) \land (X_2=0 \mid X_1=1))$ 
   or $((Y=1 \mid X_2=1, X_3=1) \land (X_2=1 \mid X_1=1))$
   because
   $\mathcal{T}_1' =\{\{X_1=1, X_2=0, X_3=1, Y = 1\}, \{X_1=1, X_2=1, X_3=1, Y=1\}\}$.
  This condition ensures that whenever $X_1 = 1$ and $X_3=1$,
  $\phi_\alpha$ ensures that
  $Y=1$. It is easy to see that whenever $\alpha$ does not contain any of these
  two events,  then it must contain $((Y=0 \mid X_2=0, X_3=1) \land (X_2=0 \mid X_1=1))$ 
   or $((Y=0 \mid X_2=1, X_3=1) \land (X_2=1 \mid X_1=1))$, in which case 
   $\phi_\alpha \centernot\implies \psi_i$.
   %joe11: removed parasgraph break
   %
   }
 \end{proof}

In terms of complexity, each intervention $\vec{X}_j\leftarrow \vec{x}_j$ requires at most $2^n$ different
%joe10
%settings in the set $\mathcal{T}_j'$. Therefore, the above expression
settings in the set $\mathcal{T}_j'$. Therefore, the expression above
for $\Pr_M(\psi_i)$ 
has  $O(2^{n(r_i+1) })$ 
setting combinations in the summation and $O(n r_i+1)$ conditional probability calculations for each
such setting. This shows that an arbitrary formula $\psi$ can be evaluated in terms of 
$O(m (nr^*+1) 2^{n(r^*+1)})$ conditional probability calculations, where $r^*$ is the maximum
number of conjuncts in a disjunction $\psi_i$
 that involve at least one intervention, and $m$ is the number of disjuncts in the DNF.
%joe10*: say something about what happens if the interventions are not
%on s

%joe3
%\section{Efficient computation}
%joe4: added label

%joe14: You should add your funding too, if you have any yet.
\paragraph{Acknowledgments:} Halpern's work was supported in part by
AFOSR grant FA23862114029, MURI grant W911NF-19-1-0217, ARO grant
W911NF-22-1-0061, and NSF grant FMitF-2319186.

\newpage
% References
%joe3
%\bibliographystyle{plainnat}
\bibliographystyle{chicago}
\bibliography{z,joe} 
\newpage

\appendix
\onecolumn

\end{document}